\documentclass[letterpaper]{article}
\usepackage{uai2020}
\usepackage[margin=1in]{geometry}

\usepackage{times}

\title{Generalized Policy Elimination: an efficient algorithm for Nonparametric Contextual Bandits}

\usepackage{mathtools}
\mathtoolsset{showonlyrefs,showmanualtags}
\usepackage{amssymb,amsmath,amsthm}
\usepackage{color}
\usepackage{algorithm}
\usepackage{dsfont}
\usepackage{algpseudocode}
\usepackage{mathtools} 
\usepackage{fullpage}
\usepackage{natbib}
\usepackage{bm}
\newtheorem{theorem}{Theorem}
\newtheorem{theorem*}{Theorem}
\newtheorem{corollary}{Corollary}
\newtheorem{assumption}{Assumption}

\newtheorem{definition}{Definition}
\newtheorem{lemma}{Lemma}
\newtheorem{proposition}{Proposition}

\newcommand{\Ind}{\textbf{1}}
\newcommand{\Ref}{\mathrm{ref}} 
\newcommand{\gref}{g_{\Ref}} 
\newcommand{\hinge}{\mathrm{hinge}} 
\newcommand{\Id}{\mathrm{Id}} 

\newcommand{\argmax}{\mathop{\arg\max}}
\newcommand{\argmin}{\mathop{\arg\min}}
\newcommand{\cadlag}{c\`adl\`ag}
\DeclareMathOperator{\Unif}{Unif}
\DeclareMathOperator{\Var}{Var}

\author{Aur\'{e}lien F. Bibaut, Antoine Chambaz, Mark J. van der Laan} 

\allowdisplaybreaks

\begin{document}

\maketitle

\begin{abstract}
  We  propose   the  Generalized   Policy  Elimination  (GPE)   algorithm,  an
  oracle-efficient  contextual bandit  (CB) algorithm  inspired by  the Policy
  Elimination  algorithm  of  \cite{dudik2011}.   We prove  the  first  regret
  optimality guarantee theorem for  an oracle-efficient CB algorithm competing
  against a nonparametric class  with infinite VC-dimension.  Specifically, we
  show  that GPE  is regret-optimal  (up  to logarithmic  factors) for  policy
  classes with integrable entropy.

  For classes  with larger entropy, we  show that the core  techniques used to
  analyze GPE  can be  used to design  an $\varepsilon$-greedy  algorithm with
  regret bound  matching that of the  best algorithms to date.   We illustrate
  the  applicability of  our algorithms  and theorems  with examples  of large
  nonparametric policy  classes, for  which the relevant  optimization oracles
  can be efficiently implemented.
\end{abstract}

\section{Introduction}

In  the  contextual  bandit  (CB)  feedback  model,  an  agent  (the  learner)
sequentially observes a vector of  covariates (the context), chooses an action
among finitely many options, then receives  a reward associated to the context
and  the chosen  action. A  CB algorithm  is a  procedure carried  out by  the
learner, whose goal is to maximize the reward collected over time.
Known  as policies,  functions that  map  any context  to  an action  or to  a
distribution over actions play a key role in the CB literature. In particular,
the performance of a CB algorithm is typically measured by the gap between the
collected reward  and the reward that  would have been collected  had the best
policy in  a certain class  $\Pi$ been exploited.   This gap is  the so-called
\textit{regret against  policy class  $\Pi$}.  The class  $\Pi$ is  called the
\textit{comparison class}.

The  CB framework  applies naturally  to settings  such as  online recommender
systems, mobile health and clinical trials, to name a few. 
Although the regret is 
defined relative to a given policy class, the goal in most settings
is arguably 
to maximize the (expected cumulative) reward in an absolute sense.  It is thus
desirable to  compete against  large nonparametric  policy classes,  which are
more likely to contain a policy close to the best measurable policy.

The complexity  of a nonparametric class  of functions can be  measured by its
covering        numbers.          The        $\epsilon$-covering        number
$N(\epsilon,\mathcal{F}, L_r(P))$  of a class  $\mathcal{F}$ is the  number of
balls of radius $\epsilon  > 0$ in $L_r(P)$ norm ($r \geq  1$) needed to cover
$\mathcal{F}$.     The    $\epsilon$-covering    entropy   is    defined    as
$\log N(\epsilon, \mathcal{F}, L_r(P))$. Upper  bounds on the covering entropy
are  well   known  for   many  classes  of   functions.   For   instance,  the
$\epsilon$-covering  entropy   of  a   $p$-dimensional  parametric   class  is
$O(p   \log  (1/\epsilon))$   for  all   $r   \geq  1$.    In  contrast,   the
$\epsilon$-covering           entropy          of           the          class
$\{f:  [0,1]^d  \rightarrow  \mathbb{R}:  \forall x,  y,  |f^{(\lfloor  \alpha
  \rfloor)}(x)  - f^{(\lfloor  \alpha \rfloor)}(y)|  \leq M  \|x-y\|^{\alpha -
  \lfloor \alpha \rfloor} \}$\footnote{$\lfloor  \alpha\rfloor$ is the integer
  part; 
  $f^{(m)}$ is the $m$-th derivative.
}of   $d$-variate   H\"older   functions  is   $O(\epsilon^{-d/\alpha})$   for
$r = \infty$ (hence all  $ r \geq 1$) \citep[Theorem 2.7.1]{vdV-Wellner-1996}.
Another  popular  measure  of   complexity  is  the  Vapnik-Chervonenkis  (VC)
dimension.  Since the  $\epsilon$-covering entropy of a class  of VC dimension
$V$  is  $O(r  V  \log  (1/\epsilon))$  for  all  $r  \geq  1$  \citep[Theorem
2.6.7]{vdV-Wellner-1996}, the complexity  of a class with  finite VC dimension
is essentially the same as that of a parametric class.

We will consider  classes $\Pi$ of policies with  either a \textit{polynomial}
or     a     \textit{logarithmic}      covering     entropy,     for     which
$\log N(\epsilon, \Pi, L_{r}(P))$ is either $O(\epsilon^{-p})$ for some $p >0$
or $O(\log(1/\epsilon))$. The former are much bigger than the latter.

Efficient CB algorithms competing against classes of functions with polynomial
covering       entropy       have        been       proposed       \citep[e.g.
by][]{cesa-bianchi17a,foster-krishnamurthy2018}.  However, these algorithm are
not     regret-optimal    in     a     minimax     sense.     In     parallel,
\citet{dudik2011,agarwalb14}  have  proposed  efficient algorithms  which  are
regret-optimal for finite policy classes, or for policy classes with finite VC
dimension.  Thus there seems to be a  gap: as of today, no efficient algorithm
has been  proven to be  regret-optimal for comparison classes  with polynomial
entropy (or with infinite VC dimension).  In this article, we partially bridge
this gap.  We  provide the first efficient algorithm to  be regret-optimal (up
to some  logarithmic factors) for  comparison classes with  integrable entropy
(that   is,    $\log   N(\epsilon,\Pi,L_{r}(P))   =    O(\epsilon^{-p})$   for
$p\in  (0,1)$).   Our   main  algorithm,  that  we   name  Generalized  Policy
Elimination (GPE) algorithm, is derived  from the Policy Elimination algorithm
of \cite{dudik2011}.


\begin{figure}\label{fig:regret_exponent}
\begin{center}
\includegraphics[scale=0.40]{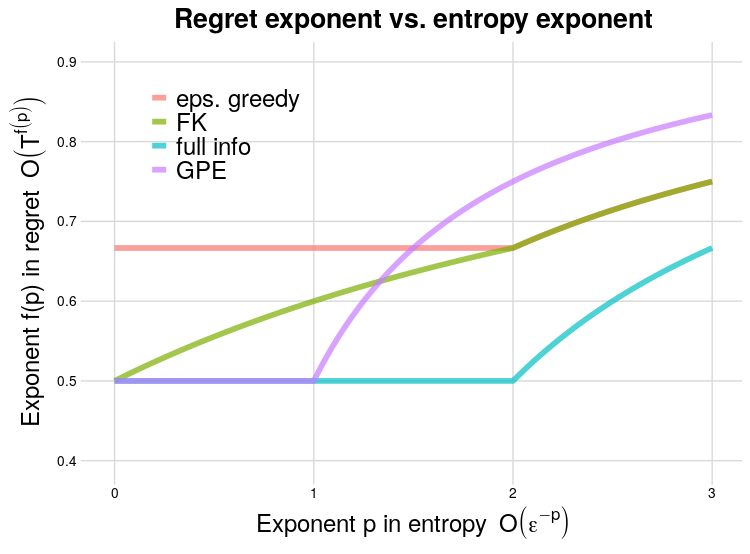}
\caption{Exponent  in regret  upper bound  (up  to logarithmic  factors) as  a
  function  of  the   exponent  in  the  (supremum   norm)  covering  entropy.
  \textit{FK}      is      the       theoretical      upper      bound      of
  \cite{foster-krishnamurthy2018}. \textit{Full info} is the bound achieved by
  Empirical Risk Minimizers under full information feedback.}
\end{center}
\end{figure}

\subsection{Previous work}

Many  contributions have  been made  to the  area of  nonparametric contextual
bandits. Among others,  one way to classify them is  according to whether they
rely on  some version  of the exponential  weights algorithm,  on optimization
oracles, or on a discretization of the covariates space.

\paragraph{Exponential  weights-based  algorithms.}  The  exponential  weights
algorithm has  a long history in  adversarial online learning, dating  back to
the seminal articles of  \cite{Vovk90} and \cite{littlestone-warmuth1994}. The
Exp3  algorithm of  \cite{auer-cesa-bianchi-freund-schapire2001} is  the first
instance  of  exponential  weigthts  for the  adversarial  multi-armed  bandit
problem. The  Exp4 algorithm of  \cite{auer2002} extends it to  the contextual
bandit setting. Infinite policy classes can be handled by running a version of
the Exp4 algorithm  on an $\varepsilon$-cover of the policy  class.  While the
Exp4  algorithm enjoys  optimal (in  a  minimax sense)  regret guarantees,  it
requires maintaining a set  of weights over all elements of  the cover, and is
thus  intractable  for  most  nonparametric classes,  because  their  covering
numbers typically grow  exponentially in $1/\epsilon$.  \cite{cesa-bianchi17a}
proposed  the first  cover-based efficient  online learning  algorithm.  Their
algorithm relies on  a hierarchical cover obtained by  the celebrated chaining
device of  \cite{Dudley1967}.  It achieves  the minimax regret under  the full
information feedback model  but not under the bandit  feedback model, although
it yields  rate improvements  over past works  for large  nonparametric policy
classes. \cite{cesa-bianchi17a}'s regret  bounds are expressed in  terms of an
entropy integral. An alternative  approach to nonparametric adversarial online
learning   is  that   of  \cite{chatterji19a},   who  proposed   an  efficient
exponential-weights algorithm  for a  reproducing kernel  Hilbert-space (RKHS)
comparison class.  They  characterized the regret in terms  of the eigen-decay
of the  kernel.  They obtained  optimal regret  if the kernel  has exponential
eigen-decay.

\paragraph{Oracle efficient algorithms.}  The  first oracle-based CB algorithm
is  the  epoch-greedy  algorithm of  \cite{langford-zhang2008}.   Epoch-greedy
allows to turn  any supervised learning algorithm into a  CB algorithm, making
it practical and  efficient (in terms of  the number of calls  to a supervised
classification   subroutine).   Its   regret   can  be   characterized  in   a
straighforward manner as a function of the sample complexity of the supervised
learning   algorithm,   but   is  suboptimal.    \cite{dudik2011}   introduced
RandomizedUCB,    the   first    regret-optimal   efficient    CB   algorithm.
\cite{agarwalb14}  improved on  their work  by  requiring fewer  calls to  the
oracle. \citep{foster18a} pointed out  that the aforementioned algorithms rely
on  cost-sensitive classification  oracles, which  are in  general intractable
(even  though  for  some  relatively natural  classes  there  exist  efficient
algorithms).     \cite{foster18a}    proposed    regret-optimal,    regression
oracles-based algorithms, motivated by the fact that regression oracles can in
general be implement efficiently. Another  way to make tractable these oracles
is, in  the case  of cost-sensitive classification  oracles, to  use surrogate
losses, as studied by  \cite{foster-krishnamurthy2018}. They gave regret upper
bounds (see  Figure \ref{fig:regret_exponent}) and a  nonconstructive proof of
the existence of an algorithm that achieves them.  They also proposed an epoch
greedy-style algorithm  that achieves the  best regret guarantees to  date for
entropy $\log  N(\epsilon, \Pi)$ of  order $\epsilon^{-p}$  for some $p  > 2$.
The caveat of the surrogate loss-based  approach is that guarantees are either
in terms  of so-called  \textit{margin-based regret}, or  can be  expressed in
terms of the  usual regret, but under the  so-called realizability assumption.
We refer the interested  reader to~\cite{foster-krishnamurthy2018} for further
details.

\paragraph{Covariate space  discretization-based algorithms.}  A third  way to
design nonparametric CB algorithms consists  in discretizing the context space
into bins and running multi-armed bandit algorithms in each bin. This approach
was pioneered by \cite{rigollet-zeevi2010} and extended by \cite{perchet2013}.
They take  a relatively  different perspective  from the  previously mentioned
works,  in the  sense that  the  comparison class  is defined  in an  implicit
fashion:  they assume  that the  expected reward  of each  action is  a smooth
(H\"older)  function of  the  context,  and they  compete  against the  policy
defined  by the  argmax  over actions  of the  expected  reward. Their  regret
guarantees are optimal in a minimax sense.

\subsection{Our contributions}

\paragraph{Primary  contribution.}    In  this   article,  we   introduce  the
Generalized Policy Elimination algorithm,  derived from the Policy Elimination
algorithm of \cite{dudik2011}.  GPE is an oracle-efficient algorithm, of which
the regret can be bounded in terms  of the metric entropy of the policy class.
In particular we show that if the  entropy is integrable, then GPE has optimal
regret, up to  logarithmic factors.  The key  enabler of our results  is a new
maximal      inequality      for       martingale      processes      (Theorem
\ref{thm:max_ineq_IS_weighted_mart_process}             in            appendix
\ref{section:max_ineqs}),  inspired   by  \citep{vandeGeer2000,vanHandel2011}.
Although our  regret upper  bounds for  GPE are no  longer optimal  for policy
classes with non-integrable entropy, we show that  we can use the same type of
martingale process techniques to design an $\varepsilon$-greedy type algorithm
that matches the current best upper bounds.

\paragraph{Comparison  to previous  work.}   Earlier  works on  regret-optimal
oracle-efficient                      algorithms                     \cite[for
instance]{dudik2011,agarwalb14,foster18a}  have  in  common  that  the  regret
analysis holds  for a  finite number  of policies or  for policy  classes with
finite VC dimension. GPE is the first oracle-efficient algorithm for which are
proven  regret  optimality guarantees  against  a  truly nonparametric  policy
classes (that is, larger than VC).

\paragraph{Secondary contributions.} 
In addition to the nonparametric  extension of policy elimination and analysis
of $\varepsilon$-greedy in terms of (bracketing) entropy, we introduce several
ideas that,  to the best  of our  knowledge, have not  appeared so far  in the
literature. In  particular, we  demonstrate the possibility  of doing  what we
call  \textit{direct  policy optimization},  that  is  of directly  finding  a
maximizer  $\widehat{\pi}$ of  $\pi  \mapsto \widehat{\mathcal{V}}(\pi)$  over
$\Pi$     where    $\widehat{\mathcal{V}}(\pi)$     estimates    the     value
$\mathcal{V}(\pi)$ of  policy $\pi$. As  far as we  know, no example  has been
given yet  of a  nonparametric class  $\Pi$ for  which $\widehat{\pi}$  can be
efficiently  computed, although  some articles  postulate the  availability of
$\widehat{\pi}$ \citep{Luedtke2019,AtheyWager2017}.  Here,  we exhibit several
rich  classes  for  which  direct   policy  optimization  can  be  efficiently
implemented.  Another secondary contribution is the first formal regret bounds
for the  $\varepsilon$-greedy algorithm, which  follows from the same  type of
arguments as in the analysis of GPE.  We were relatively surprised to see that
unlike the epoch-greedy algorithm,  the $\varepsilon$-greedy algorithm has not
been formally analyzed yet, to the best  of our knowledge.  This may be due to
the fact  that doing  so requires  martingale process  theory, which  has only
recently started to receive attention in the CB literature.

\subsection{Setting}

For each $m \geq 1$, denote $[m]\doteq \{1,\ldots, m\}$.

At     time     $t    \geq     1$,     the     learner    observes     context
$W_t  \in  \mathcal{W}\doteq  [0,1]^d$,  chooses  an  action  $A_t  \in  [K]$,
$K \geq  2$, and receives  the outcome/reward  $Y_t \in \{0,1\}$.   We suppose
that the  contexts are i.i.d.   and the rewards are  conditionally independent
given actions and  contexts, with fixed conditional  distributions across time
points.   We  denote  $O_t$  the  triple   $(W_t,  A_t,  Y_t)$,  and  $P$  the
distribution\footnote{$P$ is  partly a  fact of  nature, through  the marginal
  distribution of  context and the  conditional distributions of  reward given
  context and  action, and  the result  of the  learner's decisions.}   of the
infinite   sequence   $O_1,O_2,\ldots,   O_{t},  \ldots{}$.    Moreover,   let
$O^{\Ref} \doteq  (W^{\Ref}, A^{\Ref},  Y^{\Ref})$ be  a random  variable such
that   $W^{\Ref}    \sim   W_{1}$,   $A^{\Ref}|W^{\Ref}    \sim   \Unif([K])$,
$Y^{\Ref}|A^{\Ref}, W^{\Ref} \sim Y_{1}|A_{1}, W_{1}$.  We denote
$F_{t}$ the filtration induced by $O_1,\ldots,O_t$.

Generically   denoted   $f$  or   $\pi$,   a   policy   is  a   mapping   from
$\mathcal{W}   \times   [K]$   to    $\mathbb{R}_+$   such   that,   for   all
$w \in \mathcal{W}$, $\sum_{a\in[K]} f(a,w) =  1$.  Thus, a policy can be viewed
as mapping a  context to a distribution  over actions.  We say  the learner is
carrying  out   policy  $\pi$  at   time  $t$  if,   for  all  $a   \in  [K]$,
$w  \in  \mathcal{W}$,  $P[A_t=a|W_t=w]  =  \pi(a,w)$.   Owing  to  statistics
terminology, we  also call \textit{design} the  policy carried out at  a given
time point. The value $\mathcal{V}(\pi)$ of $\pi$ writes as
\begin{equation}
  \mathcal{V}(\pi)\doteq               E_{P}              \left[\sum_{a\in[K]}
    E_{P}[Y|A=a,W]\pi(a|W)\right].  
\end{equation} 

For any two policies $f$ and $g$, we denote
\begin{align}
  \label{eq:V(g,f)}
  V(g,f)\doteq E_P \left[\sum_{a\in[K]}\frac{f(a|W)}{g(a|W)} \right].
\end{align}
We call $V(g,f)$  the importance sampling (IS)  ratio of $f$ and  $g$.  The IS
ratio drives  the variance of IS  estimators of $\mathcal{V}(f)$ had  the data
been collected under policy $g$.

\section{Generalized Policy Elimination}

Introduced by \cite{dudik2011}, the policy elimination algorithm relies on the
following key fact. Let $\gref$ be  the uniform distribution over actions used
as a reference design/policy:
\begin{align}
  \forall (a,w) \in [K] \times \mathcal{W},\ \gref(a,w)\doteq K^{-1}.
\end{align}

\begin{proposition}\label{proposition:existence_exploration_policy}        Let
  $\delta >  0$. For  all compact  and convex  set $\mathcal{F}$  of policies,
  there exists a policy $g \in \mathcal{F}$ such that
\begin{align}
  \sup_{f  \in   \mathcal{F}}  V(\delta  \gref   +  (1-\delta)g,  f)   \leq  2
  K. \label{eq:existence_bound_IS_ratio} 
\end{align}
\end{proposition}
We  refer  to  their  article  for   a  proof  of  this  result.   Proposition
\ref{proposition:existence_exploration_policy}  has  an important  consequence
for exploration.  Suppose that at time $t$ we have a set of candidate policies
$\mathcal{F}_t$,    and    that     the    designs    $g_1,...,g_t$    satisfy
\eqref{eq:existence_bound_IS_ratio}   with  $\mathcal{F}_t$   substituted  for
$\mathcal{F}$.   We can  then estimate  the value  of candidate  policies with
error uniformly  small over  $\mathcal{F}_t$.  This in  turn has  an important
implication for  exploitation: we can  eliminate from $\mathcal{F}_t$  all the
policies  that have  value below  some well-chosen  threshold, yielding  a new
policy  set  $\mathcal{F}_{t+1}$,  and  choose  the  next  exploration  policy
$g_{t+1}$    in   $\mathcal{F}_{t+1}$.     This    reasoning   suggested    to
\cite{dudik2011} their policy elimination algorithm: (1) initialize the set of
candidate  policies to  the entire  policy  class, (2)  choose an  exploration
policy  that ensures  small value  estimation error  uniformly over  candidate
policies, (3) eliminate low value policies,  (4) repeat steps (2) and (3).  We
present formally our version of the policy algorithm as algorithm~\ref{alg:PE}
below.

In this  section, we  show that  under an  entropy condition,  and if  we have
access to  a certain optimization oracle,  our GPE algorithm is  efficient and
beats  existing  regret upper  bounds  in  some nonparametric  settings.   Our
contribution here is chiefly to extend the regret analysis of \cite{dudik2011}
to classes of functions characterized by their metric entropy in $L_\infty(P)$
norm.  This requires  us to prove a new chaining-based  maximal inequality for
martingale                          processes                         (Theorem
\ref{thm:max_ineq_under_param_dependent_IS_bound}          in         appendix
\ref{section:max_ineqs}). On  the computational side, our  algorithm relies on
having access to slightly more powerful oracles than that of \cite{dudik2011}.
We             present              them             in             subsection
\ref{subsection:efficient_algo_for_policy_search}  and  give several  examples
where these oracles can be implemented efficiently.

We  now   formally  state  our   GPE  algorithm.   Consider  a   policy  class
$\mathcal{F}$.          For         any        policy         $f$,         any
$o=(w,a,y) \in \mathcal{W}\times  [K] \times \{0,1\}$, define  the policy loss
and its IS-weighted counterpart
\begin{align}
  \ell(f)(o) 
  &\doteq  f(a,w) (1-y),\\ 
  \ell_\tau(f)(o) 
  &\doteq  \frac{\gref(a,w)}{g_\tau(a,w)}f(a,w)(1-y), 
\end{align}
the                             corresponding                             risk
$R(f)\doteq  E[\ell(f)(O^{\Ref})] =  E_{P}[\ell_{\tau}(f)(O_{\tau})]$ and  its
empirical                                                          counterpart
$\widehat{R}_t(f)\doteq t^{-1} \sum_{\tau=1}^t \ell_\tau(f)(O_\tau)$.

\begin{algorithm}[H]
  \begin{algorithmic}
    \caption{Generalized Policy Elimination}
    \label{alg:PE}
    \State  {\bfseries  Inputs:}  policy  class  $\mathcal{F}$,  $\epsilon>0$,
    sequences $(\delta_t)_{t \geq 1}$,  $(x_t)_{t \geq 1}$.  \State Initialize
    $\mathcal{F}_1$   as  $\mathcal{F}$.    \For{$t  \geq   1$}  \State   Find
    $\widetilde{g}_t \in \mathcal{F}_t$ such that, for all $f \in \mathcal{F}_t$,
    \begin{align}
      & \frac{1}{t-1} \sum_{\tau = 1}^{t-1} \frac{f(a|W_\tau)}{ (\delta_t \gref + (1 - \delta_t) \widetilde{g}_t)(a|W_\tau)} \leq 2K. \label{eq:exploration_policy_search_step}
    \end{align}
    \State Define $g_t = \delta_t \gref + (1 - \delta_t) \widetilde{g}_t$.  
    \State  Observe  context  $W_t$, sample  action  $A_t\sim~g_t(\cdot|W_t)$,
    collect reward $Y_t$.  \State Define $\mathcal{F}_{t+1}$ as 
    \begin{align}
        \left\lbrace f  \in \mathcal{F}_t : \widehat{R}_t(f)  \leq \min_{f \in
      \mathcal{F}_t}             \widehat{R}_t(f)            +             x_t
      \right\rbrace. \label{eq:policy_elim_step} 
    \end{align}
    \EndFor
  \end{algorithmic}
\end{algorithm}

\subsection{Regret analysis}\label{subsection:PE_regret_analysis}

Our regret analysis relies on the following assumption.
\begin{assumption}[Entropy condition]
  \label{assumption:entropy}
  There  exist  $c  >  0$,  $p  >  0$ such  that,  for  all  $\epsilon  >  0$,
  $\log N(\epsilon, \mathcal{F}, L_\infty(P)) \leq c \epsilon^{-p}.$
\end{assumption}

Defining          $\mathcal{F}_{t+1}         \subset          \mathcal{F}_{t}$
as~\eqref{eq:policy_elim_step},  the  policy  elimination  step,  consists  in
removing from $\mathcal{F}_t$ all the policies that are known to be suboptimal
with high probability.   The threshold $x_t$ thus plays the  role of the width
of  a uniform-over-$\mathcal{F}_t$  confidence interval.   Set $\epsilon  > 0$
arbitrarily.     We    will    show    that   the    following    choice    of
$(\delta_\tau)_{\tau\geq  1}$ and  $(x_\tau)_{\tau\geq  1}$  ensures that  the
confidence intervals  hold with  probability $1-6\epsilon$, uniformly  both in
time  and over  the  successive $\mathcal{F}_\tau$'s:  for  all $\tau\geq  1$,
$\delta_\tau\doteq  \tau^{-(1/2 \wedge 1/(2p))}$ and
\begin{small}
  \begin{align}
    & x_{\tau} \doteq  x_\tau(\epsilon) \doteq  \sqrt{v_\tau(\epsilon)} \bigg\{
      \frac{c_1}{\tau^{\frac{1}{2} \wedge
      \frac{1}{2p}}} + \frac{c_2 + c_5 \sqrt{v_\tau(\epsilon)}}{\sqrt{\tau}}  \\
    & \times \sqrt{\log \left( \frac{\tau (\tau + 1)}{\epsilon} \right) } + \frac{1}{\tau \delta_\tau} \left( c_3 + c_7 \log \left(
      \frac{\tau(\tau+1)}{\epsilon} \right)\right)\bigg\}
  \end{align}
\end{small}
\!\!\!--- defined in appendix~\ref{sec:regret:analysis}, $v_\tau(\epsilon)$ is
a          high          probability          upper          bound          on
$\sup_{f \in \mathcal{F}_\tau} \Var_P(\ell_\tau(f)(O_\tau)|F_{\tau-1})$. It is
constructed as  follows.  It  can be  shown that  the conditional  variance of
$\ell_\tau(f)(O_{\tau})$ given $F_{\tau-1}$ is driven by the expected IS ratio
$E_P[\sum_{a\in[K]}     f(a,W)     /    g_\tau     (a,W)|F_{\tau-1}]$.      Step
\ref{eq:exploration_policy_search_step} ensures  that the empirical  mean over
past observations  of the  IS ratio  is no greater  than $2K$,  uniformly over
$\mathcal{F}_\tau$.   The gap  $(v_\tau(\epsilon) -  2K)$  is a  bound on  the
supremum  over  $\mathcal{F}_{\tau}$ of  the  deviation  between empirical  IS
ratios and the true IS ratios.

We   now  state   our  regret   theorem  for   algorithm  \ref{alg:PE}.    Let
$f^*  \doteq   \argmin_{f  \in   \mathcal{F}}$  be   the  optimal   policy  in
$\mathcal{F}$.

\begin{theorem}[High probability regret bound for policy elimination]\label{thm:regret_PE}
  Consider         algorithm        \ref{alg:PE}.          Suppose        that
  Assumption~\ref{assumption:entropy} is met.  Then, with probability at least
  $1- 7 \epsilon$, for all $t \geq 1$,
\begin{align}
  \sum_{\tau = 1}^t 
  & \left(\mathcal{V}(f^*) - Y_\tau\right) \\
  & \leq \sqrt{t \log \left(\frac{1}{\epsilon}\right)} + 2 \sum_{\tau=1}^t
    x_\tau(\epsilon) + \sum_{\tau=1}^t \delta_\tau \\ 
  &= \begin{cases}
    O\left(\sqrt{t} \left(\log(\frac{t}{\epsilon})\right)^{3/2}\right) & \text{ if } p \in (0,1) \\
    O\left(t^{\frac{p-1/2}{p}}
      \left(\log(\frac{t}{\epsilon})\right)^{3/2}\right) & \text{ if } p > 1
  \end{cases}.
\end{align}
\end{theorem}

The     proof     of     Theorem     \ref{thm:regret_PE},     presented     in
appendix~\ref{sec:regret:analysis}, hinges on the three following facts.
\begin{enumerate}
\item  Controlling  the  supremum  w.r.t.  $f\in  \mathcal{F}_{\tau}$  of  the
  empirical       estimate       of        the       IS       ratio       (see
  \eqref{eq:exploration_policy_search_step} in  the first step of  the loop in
  algorithm \ref{alg:PE})  allows to  control the supremum  w.r.t. $f$  of the
  true IS ratio $V(g_{\tau},f)$.
\item With the  specification of $(x_t)_{t\geq 1}$  and $(\delta_t)_{t\geq 1}$
  sketched  above   we  can   guarantee  that,   with  probability   at  least
  $1   -   3\epsilon$,  $f^*   \in   \mathcal{F}_t   \subset  \ldots   \subset
  \mathcal{F}_{1}$. 
\item If $f^{*} \in \mathcal{F}_t$ then we can prove that, with probability at
  least $1 - 5 \epsilon$, for all $\tau \in [t]$,
\begin{align}
  R(\widetilde{g}_\tau) - R(f^*) \leq 2 x_\tau(\epsilon).
\end{align}
This  in turn  yields a  high probability  bound on  the cumulative  regret of
algorithm \ref{alg:PE}.
\end{enumerate}

\subsection{An efficient algorithm for the exploration policy search step}
\label{subsection:efficient_algo_for_policy_search}

We  show  that  the  exploration  policy  search  step  can  be  performed  in
$O(\text{poly}(t))$ calls  to two optimization  oracles that we  define below.
The   explicit  algorithm   and  proof   of   the  claim   are  presented   in
appendix~\ref{sec:efficient:algorithm:for:policy:search:in:GPE}.

\begin{definition}[Linearly Constrained Least-Squares Oracle]
  \label{def:LCLSO}
  We call Linearly Constrained Least-Squares Oracle (LCLSO) over $\mathcal{F}$
  a   routine   that,   for   any   $t   \geq   1$,   $q   \geq   1$,   vector
  $w \in \mathbb{R}^{Kt}$, sequence  of vectors $W_1,...,W_t \in \mathcal{W}$,
  set of vectors $u_1,...,u_q \in \mathbb{R}^{Kt}$, and scalars $b_1,...,b_q$,
  returns, if there exists one, a solution to
  \begin{align}
    &\min_{f  \in  \mathcal{F}}  \sum_{\substack{a  \in  [K]\\  \tau  \in  [t]}}
    (w(a,\tau) - f(a, W_\tau))^2 \text{ subject to } \\ 
    &\forall m \in [q], \sum_{\substack{a \in  [K]\\ \tau \in [t]}} u_m(a, \tau)
    f(a, W_\tau) \leq b_\tau. 
  \end{align}
\end{definition}

\begin{definition}[Linearly Constrained Cost-Sensitive Classification Oracle]
  \label{def:LCCSCO}
  We call  Linearly Constrained Cost-Sensitive Classification  Oracle (LCCSCO)
  over $\mathcal{F}$ a  routine that, for any  $t \geq 1$, $q  \geq 1$, vector
  $C \in  (\mathbb{R}_+)^{Kt}$, set of vectors  $W_1,...,W_t \in \mathcal{W}$,
  set  of  vectors  $u_1,...,u_q  \in \mathbb{R}^{Kt}$,  and  set  of  scalars
  $b_1,...,b_q \in \mathbb{R}$ returns, if there exists one, a solution to
\begin{align}
  &\min_{f \in \mathcal{F}} \sum_{\substack{a \in [K]\\ \tau \in [t]}} C(a,\tau) f(a, W_\tau) \text{ subject to } \\
  &\forall m \in [q],  \sum_{\substack{a \in [K]\\ \tau \in [t]}} u_m(a, \tau) f(a, W_\tau) \leq b_\tau.
\end{align}
\end{definition}

The following theorem is our main  result on the computational tractability of
the policy search step.

\begin{theorem}[Computational       cost      of       exploration      policy
  search]
  \label{thm:runtime_policy_search} 
  For every $t \geq 1$, exploration policy search at time $t$ can be performed
  in $O((Kt)^2 \log t)$ calls to both LCLSO and LCCSCO.
\end{theorem}

The proof of Theorem~\ref{thm:runtime_policy_search}  builds upon the analysis
of \cite{dudik2011}.  Like  them, we use the famed ellipsoid  algorithm as the
core component.  The general idea is as follows.  We show that the exploration
policy  search step  \eqref{eq:exploration_policy_search_step}  boils down  to
finding a point  $w \in \mathbb{R}^{Kt}$ that belongs to  a certain convex set
$\mathcal{U}$,  and to  identifying a  $\widetilde{g}_t \in  \mathcal{F}_{t}$ such
that $\sum_{a,\tau}  (f(a, W_\tau) -  w(a,\tau))^2 \leq \Delta$ for  a certain
$\Delta              >               0$.               In              section
\ref{subsubection:policy_search_as_convex_feasiblity},       we       identify
$\mathcal{U}$           and           $\Delta$.           In           section
\ref{subsubsection:solving_feasibility_with_ellipsoid_alg}, we demonstrate how
to find a point in $\mathcal{U}$ with the ellipsoid algorithm.

\section{Finite sample guarantees for $\varepsilon$-greedy}
\label{section:epsilon_greedy}

In  this  section,  we  give  regret   guarantees  for  two  variants  of  the
$\varepsilon$-greedy algorithm competing against  a policy class characterized
by bracketing  entropy, denoted  thereon $\log N_{[\,]}$,  and defined  in the
appendix\footnote{It is  known that $\log N(\epsilon,  \mathcal{F}, L_{r}(P))$
  is smaller  than $\log  N_{[\,]}(2\epsilon, \mathcal{F}, L_{r}(P))$  for all
  $\epsilon > 0$.}.  Corresponding to two choices of an input argument $\phi$,
the two variants of algorithm~\ref{alg:thealg} differ in whether they optimize
w.r.t.  the policy either an estimate of its value or an estimate of its hinge
loss-based risk.

We  formalize  this  as  follows.   We consider  a  class  $\mathcal{F}_0$  of
real-valued  functions  over $\mathcal{W}$  and  derive  from it  two  classes
$\mathcal{F}^{\Id}$ and $\mathcal{F}^{\hinge}$ defined as
\begin{align}
  \mathcal{F}^{\Id} \doteq \big\{ 
  & (a,w) \mapsto f_a(w):  f_1, \ldots, f_K \in \mathcal{F}_0,\\
  & \forall w \in \mathcal{W}, (f_1(w),...,f_K(w)) \in \Delta(K) \big\}, \label{eq:F_Id_def}
\end{align}
where $\Delta(K)$ is the $K$-dimensional probability simplex, and
\begin{align}
  \mathcal{F}^{\hinge} \doteq \big\{
  & (a,w) \mapsto f_a(w) : f_1, \ldots f_K \in \mathcal{F}_0,\\
  &\forall w \in \mathcal{W}, \textstyle{\sum_{a\in[K]}} f_a(w) = 0 \big\}.
    \label{eq:F_hinge_def}
\end{align}

Let  $\phi^{\Id}$ be  the identity  mapping and  $\phi^{\hinge}$ be  the hinge
mapping  $x  \mapsto \max(0,  1  +  x)$,  both over  $\mathbb{R}$.   Following
exisiting   terminology   \cite[for  instance]{foster-krishnamurthy2018},   an
element of $\mathcal{F}$  is called a regressor. Each regressor  $f$ is mapped
to   a    policy   $\pi$    through   a   \textit{policy    mapping},   either
$\widetilde{\pi}^{\Id}$      if     $f      \in     \mathcal{F}^{\Id}$      or
$\widetilde{\pi}^{\hinge}$  if $f  \in  \mathcal{F}^{\hinge}$  where, for  all
$(a,w) \in [K] \times \mathcal{W}$,
\begin{align}
  \widetilde{\pi}^{\Id}(f)(a,w) 
  & = f(a,w), \\
  \widetilde{\pi}^{\hinge}(f)(a,w) 
  & = \Ind\{a = \argmax_{a' \in [K]} f(a',w)\}.  
\end{align}
  
For  $\phi$   set  either   to  $\phi^{\Id}$   or  $\phi^{\hinge}$,   for  any
$f:[K]\times      \mathcal{W}       \to      \mathbb{R}$,       for      every
$o  =  (w,  a, y)  \in  \mathcal{W}  \times  [K]  \times \{0,  1\}$  and  each
$\tau \geq 1$, define
\begin{align}
  \ell^\phi(f)(o) 
  &\doteq  \phi(f(a,w)) (1-y),\\ 
  \ell^\phi_\tau(f) 
  &\doteq \frac{\gref(a,w)}{g_\tau(a,w)} \phi(f(a,w))(1-y),
\end{align}
the                          corresponding                         $\phi$-risk
$R^\phi(f)        \doteq        E[\ell^\phi(f)(O^{\Ref})]       =        E_{P}
[\ell^{\phi}_{\tau}(f)(O_{\tau})]$     and    its     empirical    counterpart
$\widehat{R}_t (f)  \doteq t^{-1}  \sum_{\tau=1}^t \ell_\tau^\phi(f)(O_\tau)$.
Finally,   the   \textit{risk}   of   any   policy   $\pi$   is   defined   as
$R(\pi)    \doteq    R^{\phi}(\pi)$    with    $\phi=\phi^{\Id}$    and    the
\textit{hinge-risk} of  any regressor $f \in  \mathcal{F}^{\hinge}$ is defined
as $R^{\hinge}(f) \doteq R^{\phi} (f)$ with $\phi=\phi^{\hinge}$.

We can now present the $\varepsilon$-greedy algorithm.

\begin{algorithm}[H]
   \caption{$\varepsilon$-greedy.}
   \label{alg:thealg}
\begin{algorithmic}
  \State   \textbf{Input:}   convex    surrogate   $\phi$,   regressor   class
  $\mathcal{F}$,      policy      mapping     $\widetilde{\pi}$,      sequence
  $(\delta_t)_{t  \geq 1}$.   \State Initialize  $\widehat{\pi}_0$ as  $\gref$
  \For{$t  \geq  1$} \State  Define  policy  as  mixture between  $\gref$  and
  $\widehat{\pi}_{t-1}$:
   \begin{align}
     g_t = \delta_t \gref + (1 - \delta_t) \widehat{\pi}_{t-1}
   \end{align}
   \State Observe  context $W_t$, sample  action $A_t \sim g_t(\cdot  | W_t)$,
   collect reward $Y_t$.  \State Compute optimal empirical regressor
   \begin{align}
     \label{eq:fthat}
     \widehat{f}_t = \argmin_{f \in \mathcal{F}} \frac{1}{t}
     \sum_{\tau=1}^t \ell_\tau^\phi(f)(O_\tau). 
   \end{align}
   \State          Compute          optimal          policy          estimator
   $\widehat{\pi}_t = \widetilde{\pi}(\widehat{f}_{t})$.
    \EndFor
\end{algorithmic}
\end{algorithm}

We  consider  two  instantiations  of  the  algorithm:  one  corresponding  to
$(\phi^{\Id},    \mathcal{F}^{\Id},    \widetilde{\pi}^{\Id})$   and    called
\textit{direct    policy   optimization},    the   other    corresponding   to
$(\phi^{\hinge},  \mathcal{F}^{\hinge}, \widetilde{\pi}^{\hinge})$  and called
\textit{hinge-risk optimization}.

\paragraph{Regret decomposition.}   Denote $\pi_\Pi^*$  the optimal  policy in
$\Pi\doteq \widetilde{\pi}(\mathcal{F})$ and  $\pi^{*}$ any\footnote{There may
  exist more than one.} optimal measurable policy.  The key idea in the regret
analysis  of the  $\varepsilon$-greedy algorithm  is the  following elementary
decomposition    (details    in    appendix~\ref{sec:regret:analysis:greedy}):
$Y_t - R(\pi^*) = $
\begin{multline}
  \label{eq:regret_decomp_eps_greedy}
  \underbrace{Y_t - E_{P}[Y_t | F_{t-1}]}_{\text{reward noise}} +
  \underbrace{\delta_t (R(\gref) - R(\pi^*))}_{\text{exploration cost}} \\
  +     (1-      \delta_t)     \underbrace{(      R(\widehat{\pi}_{t-1})     -
    R(\pi^*))}_{\text{exploitation cost}}. 
\end{multline}

\paragraph{Control of the exploitation cost.}

In  the  direct  policy  optimization  case, we  can  give  exploitation  cost
guarantees  under   no  assumption   other  than   an  entropy   condition  on
$\mathcal{F}$.   In the  hinge-risk  optimization case,  we  need a  so-called
realizability                        assumption.                        Denote
$\mathbb{R}^K_{=0} \doteq  \{x \in \mathbb{R}^{K}  : \sum_{a \in [K]}  x_{a} =
0\}$.

\begin{assumption}[Hinge-realizability]
  \label{assumption:realiza}
  Let
  \begin{equation}
    f^*   \doteq    \argmin_{f   :   [K]   \times    \mathcal{W}   \rightarrow
      \mathbb{R}^K_{=0}} R^\hinge(f)
  \end{equation} 
  be the minimizer  over all measurable regressors of the  hinge-risk.  We say
  that    a    regressor    class   $\mathcal{F}^{\hinge}$    satisfies    the
  hinge-realizability      assumption      for     the      hinge-risk      if
  $f^* \in \mathcal{F}^{\hinge}$.
\end{assumption}
Imported     from     the     theory    of     classification     calibration,
Assumption~\ref{assumption:realiza} allows  us to bound  the risk of  a policy
$R(\widetilde{\pi}^{\hinge}(f))$   in  terms   of   the   hinge-risk  of   the
regressor~$f$.  The proof relies on the following result:

\begin{lemma}[Hinge-calibration]
  \label{lemma:hinge_calibration}
  Consider a regressor class $\mathcal{F}^{\hinge}$. Let
  \begin{equation}
    \pi^* \in \argmin_{\pi : [K] \times \mathcal{W} \rightarrow \Delta(K)} R(\pi)
  \end{equation}
  be an optimal measurable policy. It holds that
  $R(\pi^*)     =    R(\widetilde{\pi}^{\hinge}(f^*))$     and,    for     all
  $f \in \mathcal{F}^{\hinge}$,
  \begin{align}
    R(\widetilde{\pi}^{\hinge}(f)) - R(\pi^*) \leq R^{\hinge}(f) - R^{\hinge}(f^*). 
  \end{align}
\end{lemma}
We  refer  the  reader to  \cite{bartlett-jordan-mcauliffe2006,pires2016}  for
proofs,   respectively   when   $K=2$   and    when   $K   \geq   2$.    Under
Assumption~\ref{assumption:realiza},       Lemma~\ref{lemma:hinge_calibration}
teaches us  that we  can bound the  exploitation cost in  terms of  the excess
hinge-risk
$R^{\hinge}(f)  -   \min_{f'  \in  \mathcal{F}^{\hinge}}   R^{\hinge}(f')$,  a
quantity that we can bound by  standard arguments from the theory of empirical
risk minimization.   The fondamental building  block of our  exploitation cost
analysis  is therefore  the following  finite sample  deviation bound  for the
empirical $\phi$-risk minimizer.

\begin{theorem}[$\phi$-risk    exponential    deviation    bound    for    the
  $\varepsilon$-greedy algorithm]\label{thm:phi-risk_dev_bound_eps_greedy}
  Let $\phi$ and $\mathcal{F}$  be either $\phi^{\Id}$ and $\mathcal{F}^{\Id}$
  or    $\phi^{\hinge}$    and     $\mathcal{F}^{\hinge}$.     Suppose    that
  $g_1,\ldots,  g_t$   is  a   sequence  of  policies   such  that,   for  all
  $\tau  \in [t]$,  $g_\tau$ is  $F_{\tau-1}$-measurable.  Suppose  that there
  exist $B,\delta > 0$ such that
  \begin{gather}
    \sup_{f_1, f_2 \in \mathcal{F}} \sup_{a\in[K], w \in \mathcal{W}}
    |\phi(f_1(a,w)) - \phi(f_2(a,w))| \leq B,\\
    \min_{\tau \in [t]} g(A_\tau, W_\tau) \geq \delta \text{ a.s.}
  \end{gather}
  Define $f_{\mathcal{F}}^{*} \doteq \argmin_{f \in \mathcal{F}} R^{\phi}(f)$,
  the  $\mathcal{F}$-specific optimal  regressor of  the $\phi$-risk,  and let
  $\widehat{f}_{t}$ be  the empirical $\phi$-risk  minimizer \eqref{eq:fthat}.
  Then, for all $x > 0$ and $\alpha \in (0, B)$,
  \begin{multline}
    P \bigg[ R^\phi(\widehat{f}_t) - R^\phi(f^*_\mathcal{F}) \geq H_t\left(\alpha, \delta, B^{2} K/\delta, B \right) \\
    + 160 B \sqrt{K x/\delta t} + 3B/\delta t x \bigg] \leq 2 e^{-x},
  \end{multline}
  with $H_t(\alpha, \delta, v, B) \doteq \alpha + 160 \sqrt{v/t}$
  \begin{equation}
    \times\int_{\alpha/2}^B  \sqrt{\log (1  + N_{[\,]}(\epsilon,  \mathcal{F},
      L_2(P))} d\epsilon + \frac{3 B}{\delta t} \log 2.
  \end{equation}
\end{theorem}
As a direct corollary, we can express rates of convergence for the $\phi$-risk
in terms of the bracketing entropy rate.
\begin{corollary}\label{corollary:rate_exploitation_eps_greedy}
  Suppose                                                                 that
  $\log( 1 + N_{[\,]}(\epsilon,  \mathcal{F}, L_2(P))) = O(\epsilon^{-p})$ for
  some $p \in (0, 1)$. Then
  \begin{equation}
    R^\phi(\widehat{f}_t)  -  R^\phi(f_{\mathcal{F}}^*)   =  O_P  \left(
      (\delta t)^{- \left( \frac{1}{2} \wedge \frac{1}{p} \right) } \right).  
  \end{equation}
\end{corollary}

\paragraph{Control of the regret.} 

The                  cumulative                  reward                  noise
$\sum_{\tau=1}^t (Y_\tau  - E_{P}[Y_\tau | F_{\tau  - 1}])$ can be  bounded by
the Azuma-Hoeffding inequality.   From \eqref{eq:regret_decomp_eps_greedy} and
Corollary~\ref{corollary:rate_exploitation_eps_greedy},   $\delta_t$  controls
the  trade  off between  the  exploration  and  exploitation costs.   We  must
therefore choose a  $\delta_{t}$ that minimizes the total of  these two which,
from              the               above,              scales              as
$O( \delta_t  + ( t  \delta_t)^{- (  \frac{1}{2} \wedge \frac{1}{p})  })$. The
optimal choice is $\delta_t  \propto t^{-(\frac{1}{3} \wedge \frac{1}{p+1})}$.
The following  theorem formalizes the  regret guarantees  under the form  of a
high-probability bound.

\begin{theorem}[High  probability  regret   bound  for  $\varepsilon$-greedy.]
  Suppose that  the bracketing  entropy of  the regressor  class $\mathcal{F}$
  satisfies
  $\log (1  + N_{[\,]}(\epsilon, \mathcal{F}, L_2(P))  =O( \epsilon^{-p})$ for
  some $p  > 0$.  Set  $\delta_t = t^{-(\frac{1}{3} \vee  \frac{p}{p+1})}$ for
  all $t\geq 1$.  Suppose that
  \begin{itemize}
  \item   either  $\phi   =  \phi^{\Id}$,   $\mathcal{F}$  is   of  the   form
    $\mathcal{F}^{\Id}$, $\widetilde{\pi} = \widetilde{\pi}^{\Id}$,
  \item   or  $\phi   =   \phi^{\hinge}$,  $\mathcal{F}$   is   of  the   form
    $\mathcal{F}^{\hinge}$, $\widetilde{\pi}  = \widetilde{\pi}^{\hinge}$, and
    $\mathcal{F}$ satisfies Assumption~\ref{assumption:realiza}.
  \end{itemize}
  Then, with probability $1- \epsilon$,
  \begin{multline}
    \sum_{\tau=1}^t (\mathcal{V}(\pi^*) - Y_\tau ) \leq \sqrt{t \log
      (2/\epsilon)} \\
    + t^{\frac{p}{p+1}} \sqrt{\log(2 t(t+1)/\epsilon)}.
\end{multline}
\end{theorem}

\section{Examples of policy classes}

\subsection{A nonparametric additive model}

We say that $a(\epsilon) = \widetilde{O}(b(\epsilon))$ if there exists $c > 0$
such that  $a(\epsilon) = O(b(\epsilon) \log^{c}(1/\epsilon))$.   We present a
policy class  that has entropy $\widetilde{O}(\epsilon^{-1})$,  and over which
the  two optimization  oracles  presented  in Definitions~\ref{def:LCLSO}  and
\ref{def:LCCSCO} reduce  to linear  programs.  Let $\mathbb{D}([0,1])$  be the
set of \cadlag{} functions and let  the variation norm $\|\cdot\|_v$ be given,
for all $h \in \mathbb{D}([0,1])$, by
\begin{align}
  \|h\|_v   \doteq   \sup_{m   \geq   2}  \sup_{x_{1},   \ldots,   x_{m}}
  \sum_{i=1}^{m-1} |h(x_{i+1}) - h(x_i)|
\end{align}
where the right-hand  side supremum is over the subdivisions  of $[0,1]$, that
is                                                                        over
$\{(x_{1}, \ldots, x_{m}) : 0 \leq x_{1} \leq \ldots \leq x_{m} \leq 1\}$. Set
$C,M > 0$ then introduce
\begin{align}
  \mathcal{H} \doteq  \left\lbrace h  \in \mathbb{D}([0,1])  : \|h\|_v  \leq M
  \right\rbrace 
\end{align}
and  the additive  nonparametric additive  model  derived from  it by  setting
$\mathcal{F}_0 \doteq$
\begin{equation}
  \big\{ (a,w) 
  \mapsto \sum_{l=1}^d \alpha_{a,l} h_l(w_l)  : |\alpha_{a,l}| \leq C, h_{a,l}
  \in \mathcal{H} \big\}.
\end{equation}
Let  $\mathcal{F}   =  \mathcal{F}^{\Id}$  derived  from   $\mathcal{F}_0$  as
in~\eqref{eq:F_Id_def}.

The following lemma formally bounds the entropy of the policy class.

\begin{lemma}\label{lemma:entropy_additive_model} 
  There    exists    $\epsilon_0   \in    (0,1)$    such    that,   for    all
  $\epsilon \in (0,\epsilon_0)$,
  \begin{align}
    \log N_{[\,]}(\epsilon, \mathcal{F}, \|\cdot\|_\infty) \leq & K \log N_{[\,]}(\epsilon, \mathcal{F}_0, \|\cdot\|_\infty) \\
    \leq & K c_0 \epsilon^{-1} \log (1 / \epsilon).
  \end{align}
  for some $c_0 > 0$ depending on $(C,d,M)$.
\end{lemma}

We  now state  a result  that shows  that LCLSO  and LCCSCO  reduce to  linear
programs over $\mathcal{F}$. We first need to state a definition.
\begin{definition}[Grid induced by a set of points]
  Consider $d$ subdivisions of $[0,1]$ of the form
  \begin{align}
    0 =& w_{1,1} \leq w_{1,2} \leq \ldots \leq w_{1,q_1} = 1,\\
       &\vdots \\
    0 =& w_{d,1} \leq w_{1,2} \leq \ldots \leq w_{d,q_d} = 1.
\end{align}
The \emph{rectangular  grid induced by these  $d$ subdivisions} is the  set of
points              $(w_{1,i_1},w_{2,i_2},\ldots,w_{i,i_d})$              with
$i_1 \in  [q_1],...,\ i_d \in [q_d]$.   We call a \emph{rectangular  grid} any
rectangular grid induced by some set of $d$ subdivisions of $[0,1]$.

Consider a set of points $w_1,\ldots,w_n  \in [0,1]^d$. A minimal grid induced
by $w_1,\ldots w_n$ is any rectangular grid that contains $w_1,\ldots w_n$ and
that is of minimal cardinality.  We  denote by $G(w_1, \ldots, w_n)$ a minimal
rectangular grid induced by $w_1,\ldots w_n$ chosen arbitrarily.
\end{definition}

\begin{lemma}\label{lemma:representation_thm_additive_model}
  Let  $w_{0}  =  \boldsymbol{0},  w_1,\ldots,  w_t  \in  [0,1]^d$.   For  all
  $l                     \in                     [d]$,                     let
  $\widetilde{\mathcal{H}}_{l,t}                                        \doteq
  \widetilde{\mathcal{H}}_{l,t}(w_{0,l},\ldots,w_{t,l}) \doteq$
\begin{equation}
  \big\{  x   \mapsto  \sum_{\tau=0}^t  \beta_\tau  1\{x   \geq  w_{\tau,l}\}:
  \beta_{\tau} \in \mathbb{R}, \sum_{\tau=0}^t |\beta_\tau| \leq M\big\}
\end{equation}
and $\widetilde{\mathcal{F}}_{0,t} \doteq$
\begin{equation}
  \big\{(a,w)  \mapsto  \sum_{l=1}^d \alpha_{a,l}  \widetilde{h}_{a,l}(w_l)  :
  |\alpha_{a,l}| \leq B, \widetilde{h}_{a,l} \in \mathcal{H}_{l,t}\big\}.
\end{equation}
Let   $(u_{a,\tau})_{  a   \in   [K],  \tau   \in  [t]}$   be   a  vector   in
$\mathbb{R}^{Kt}$.   Let  $\widetilde{f}^*$ be  a  solution  to the  following
optimization problem $(\mathcal{P}_2)$:
\begin{align}
  \max_{\widetilde{f} \in \widetilde{\mathcal{F}}_{0,t} } 
  & \sum_{a\in[K]} \sum_{\tau=1}^t u_{a,\tau} \widetilde{f}(a, W_\tau)\\
  \text{ s.t. } 
  & \forall a \in [K],\ \forall w \in \mathcal{G}(w_0,\ldots,w_t),\  \widetilde{f}(a,w) \geq 0, \label{eq:P2_pos_constrt}\\
  &  \forall w \in \mathcal{G}(w_0,\ldots,w_t),\  \sum_{a\in[K]} \widetilde{f}(a,w) = 1. \label{eq:P2_sums_to1_constr}
\end{align}
Then,  $\widetilde{f}$ is  a solution  to the  following optimization  problem
$(\mathcal{P}_1)$:
\begin{align}
  \max_{f \in \mathcal{F}_0} 
  & \sum_{a\in[K]} \sum_{\tau=1}^t u_{a,\tau} f(a,W_\tau) \\
  \text{ s.t. } 
  &\forall a \in [K], \forall w \in [0,1]^d, f(a,w) \geq 0, \label{eq:P1_pos_constrt}\\
  &\forall w \in [0,1]^d, \ \sum_{a\in[K]} f(a,w) = 1. \label{eq:P1_sums_to1_constr}
\end{align}
\end{lemma}

\subsection{C\`adl\`ag policies with bounded sectional variation norm} 

The class of $d$-variate \cadlag{}  functions with bounded sectional variation
norm  is a  nonparametric function  class with  bracketing entropy  bounded by
$O(\epsilon^{-1}  \log(1/\epsilon)^{2(d-1)})$,   over  which   empirical  risk
minimization takes  the form of  a LASSO  problem.  It has  received attention
recently   in  the   nonparametric   statistics  literature   \citep{vdL_2015,
  Fang-Guntuboyina-Sen-2019, bibaut-vdL-2019}.  Empirical risk minimizers over
this class of  functions have been termed Highly Adaptive  Lasso estimators by
\cite{vdL_2015}.  The experimental study  of \cite{benkeser2016} suggests that
Highly Adaptive  Lasso estimators are competitive  against supervised learning
algorithms such as Gradient Boosting Machines and Random Forests.

\paragraph{Sectional      variation      norm.}       For      a      function
$f:[0,1]^d \rightarrow  \mathbb{R}$, and a  non-empty subset $s$ of  $[d]$, we
call  the $s$-section  of  $f$ and  denote  $f_s$ the  restriction  of $f$  to
$\{x \in [0,1]^d:  \forall i \in s,  x_i = 0\}$. The  sectional variation norm
(svn) is defined based on the notion of Vitali variation.  Defining the notion
of  Vitali  variation  in  full  generality  requires  introducing  additional
concepts.     We     thus     relegate      the     full     definition     to
appendix~\ref{sec:cadlag:svn},  and  present it  in  a  particular case.   The
Vitali  variation   of  an  $m$-times  continuously   differentiable  function
$g:[0,1]^m \rightarrow \mathbb{R}$ is defined as
\begin{equation}
  V^{(m)}(g)   \doteq   \int_{[0,1]^m}   \bigg|  \frac{\partial^m   g}{\partial   x_1
    \ldots \partial x_m} \bigg|.
\end{equation}
For  arbitrary  real-valued  \cadlag{}   functions  $g$  on  $[0,1]^{m}$  (non
necessarily  $m$  times  continuously differentiable),  the  Vitali  variation
$V^{(m)}(g)$  is  defined  in  appendix~\ref{sec:cadlag:svn}.  The  svn  of  a
function $f:[0,1]^d \rightarrow \mathbb{R}$ is defined as
\begin{equation}
  \|f\|_v \doteq  |f(0)|  + \sum_{\emptyset \neq s \subset [d]} V^{(|s|)}(f_s),
\end{equation}
that is the sum of its absolute value  at the origin and the sum of the Vitali
variation  of  its  sections.   Let  $\mathbb{D}([0,1]^d)$  be  the  class  of
\cadlag{} functions with domain $[0,1]^d$ and, for some $M > 0$, let
\begin{align}
  \mathcal{F}_0\doteq \left\lbrace  f \in  \mathbb{D}([0,1]^{d}) : \|f\|_v  \leq M
  \right\rbrace \label{eq:def_HAL_class} 
\end{align}
be the class of \cadlag{} functions with svn smaller than $M$.

\paragraph{Entropy bound.}  

The following result is taken from \citep{bibaut-vdL-2019}.
\begin{lemma}
  Consider $\mathcal{F}_0$ defined in  \eqref{eq:def_HAL_class}.  Let $P$ be a
  probability      distribution       over      $[0,1]^d$       such      that
  $\|\cdot\|_{P,2}  \leq  c_0 \|\cdot\|_{\mu,  2}$,  with  $\mu$ the  Lebesgue
  measure and $c_{0} > 0$.  Then there exist $c_{1} > 0, \epsilon_0 \in (0,1)$
  such that, for  all $\epsilon \in (0,\epsilon_0)$ and  all distributions $P$
  over $[0,1]^d$,
  \begin{align}
    \log  N_{[\,]}(\epsilon, \mathcal{F}_0,  L_2(P))  \leq c_1  M\epsilon^{-1}
    \log (M/\epsilon)^{2d-1}. 
  \end{align}
\end{lemma}

\paragraph{Representation of  ERM.}  We show that  empirical risk minimization
(ERM) reduces to  linear programming in both our direct  policy and hinge-risk
optimization settings.

\begin{lemma}[Representation  of the  ERM  in the  direct policy  optimization
  setting]\label{lemma:HAL_ERM_direct_policy_optim}
  Consider   a   class   of   policies   of   the   form   $\mathcal{F}^{\Id}$
  \eqref{eq:F_Id_def}  derived from  $\mathcal{F}_0$ \eqref{eq:def_HAL_class}.
  Let  $\phi  = \phi^{\hinge}$.   Suppose  we  have observed  $(W_1,A_1,Y_1)$,
  \ldots, $(W_t,A_t,Y_t)$ and let $\widetilde{W}_1,\ldots, \widetilde{W}_m$ be
  the elements of $G(W_1,\ldots,W_t)$.

Let $(\beta^a_j)_{a\in[K], j\in[m]}$ be a solution to
\begin{equation}\label{eq:EMR_HAL_dpo}
  \begin{split}
    \min_{\beta \in \mathbb{R}^{Km}} & \sum_{\tau=1}^t \sum_{a\in[K]} \bigg\{ \frac{1\{A_\tau = a\}}{g_\tau(A_\tau,W_\tau)}(1-Y_\tau)\\
    & \hspace*{3cm}\times \sum_{j=1}^m \beta^a_j 1\{W_\tau \geq \widetilde{W}_j \} \bigg\}\\
    \text{ s.t. } & \forall l \in [m],\  \sum_{a\in[K]} \sum_{j=1}^m \beta^a_j 1\{\widetilde{W}_l \geq \widetilde{W}_j\} = 1,\\
    & \forall l \in [m], \forall a \in [K],\  \sum_{j=1}^m \beta^a_j 1\{\widetilde{W}_l \geq \widetilde{W}_j \} \geq 0, \\
    & \forall a \in [K],\ \sum_{j=1}^m | \beta^a_j| \leq M.\
  \end{split}
\end{equation}
Then $f : (a,w) \mapsto \sum_{j=1}^m \beta^a_j 1\{w \geq \widetilde{W}_j\}$ is
a                                  solution                                 to
$\min_{f \in \mathcal{F}^{\Id}} \sum_{\tau=1}^t \ell^\phi_\tau(f)(O_\tau).$
\end{lemma}
We   present    a   similar   result    for   the   hinge-risk    setting   in
appendix~\ref{sec:cadlag:svn}. It  is relatively easy  to prove with  the same
techniques  that  ERM  over  $\mathcal{F}^{\hinge}$  also  reduces  to  linear
programming when $\mathcal{F}_0$ is an RKHS.

\section{Conclusion}

We present  the first  efficient CB algorithm  that is  regret-optimal against
policy  classes with  \textit{polynomial}  entropy.  We  acknowledge that  our
algorithm might not be  practical. It inherits some of the  caveats of PE: (1)
the probability of  the regret bound is a pre-specified  parameter, (2) if the
algorithm eliminates the best policy, it never recovers.

We  conjecture  that  regret  optimality  could be  proven  for  classes  with
non-integrable entropy. The role of  integrability is purely technical and due
to our proof techniques.

\nocite{massart2007}
\bibliography{biblio}

\appendix
\onecolumn

\section{Notation}\label{section:notation}

Set  arbitrarily  $n  \geq  1$  and  let  $\phi$  be  either  $\phi^{\Id}$  or
$\phi^{\hinge}$.    We   denote   by    $P_n$   the   empirical   distribution
$n^{-1}     \sum_{i=1}^n    \text{Dirac}(O_i)$.      For    all     measurable
$f: [K] \times  \mathcal{W} \to \mathbb{R}$, we  let $\ell_{1:n}^{\phi}(f)$ be
the                vector-valued                random                function
$(\ell_{1}^{\phi}(f),         \ldots,        \ell_{n}^{\phi}(f))$         over
$[K] \times  \mathcal{W}$.  In order  to alleviate notation, we  introduce the
following empirical process theory-inspired notation.  For any \textit{fixed},
measurable function $f: [K] \times \mathcal{W} \to \mathbb{R}$,
\begin{align}
  P \ell_{1:n}(f) 
  & \doteq \frac{1}{n} \sum_{i=1}^{n} E_{P}
    \left[\ell_{i}^{\phi}(f)(O_{i})       |       F_{i-1}\right],\\
  P_{n} \ell_{1:n}(f) 
  &              \doteq             \frac{1}{n}              \sum_{i=1}^{n}
    \ell_{i}^{\phi}(f)(O_{i}),\\ 
  (P-P_{n}) \ell_{1:n}(f) 
  & \doteq \frac{1}{n} \sum_{i=1}^{n} \left(E_{P}
    \left[\ell_{i}^{\phi}(f)(O_{i})       |       F_{i-1}\right]       -
    \ell_{i}^{\phi}(f)(O_{i})\right). 
\end{align}
For          a           \textit{random}          measurable          function
$f:     [K]     \times     \mathcal{W}     \to     \mathbb{R}$,     we     let
$P      \ell_{1:n}(f)      \doteq     P      \ell_{1:n}(f')|_{f'=f}$,      and
$P_{n} \ell_{1:n}(f')|_{f'=f}$, $(P-P_{n}) \ell_{1:n}(f')|_{f'=f}$. 

\section{Maximal inequalities}\label{section:max_ineqs}

\subsection{The basic maximal inequality for IS-weighted martingale processes}

\begin{definition}[Bracketing entropy, \cite{vdV-Wellner-1996}]
  Given two  functions $l,u:  \mathcal{X}\rightarrow \mathbb{R}$,  the bracket
  $[l,u]$ is the  set of all functions  $f:\mathcal{X} \rightarrow \mathbb{R}$
  such that,  for all $x  \in \mathcal{X}$, $l(x)  \leq f(x) \leq  u(x)$.  The
  bracketing number $N_{[\,]}(\epsilon, \mathcal{F}, L_r(P))$ is the number of
  brackets  $[l,u]$   such  that  $\|l-u\|_{P,r}  \leq   \epsilon$  needed  to
  cover~$\mathcal{F}$.
\end{definition}
The following proposition  is a well-known result  relating bracketing numbers
and covering numbers~\cite[for instance]{vdV-Wellner-1996}.
\begin{proposition}
  For   any  probability   distribution   $P$,  for   all   $\epsilon  >   0$,
  $N(\epsilon,  \mathcal{F}, L_r(P))  \leq  N_{[\,]}(2 \epsilon,  \mathcal{F},
  L_r(P))$                                                                 and
  $N(\epsilon,  \mathcal{F},   \|\cdot\|_\infty)  \leq   N_{[\,]}(2  \epsilon,
  \mathcal{F},\|\cdot\|_\infty)$.
\end{proposition}
In the  statement of Theorem~\ref{thm:regret_PE}, the  high-probability regret
bount for  GPE, we used  the covering numbers  in uniform norm.   The previous
lemma allows us  to carry out the  analysis in terms of  bracketing numbers in
uniform norm.

\begin{theorem}[Maximal inequality for IS-weighted martingale processes]\label{thm:max_ineq_IS_weighted_mart_process}
  Consider  the setting  of Section~\ref{section:epsilon_greedy}  in the  main
  text.     Specifically,    suppose    that    for   all    $i    \geq    1$,
  $A_i |W_i \sim  g_i(\cdot | A_i)$ where $g_i$  is $F_{i-1}$-measurable.  Let
  $n \geq 1$, and $f_0 \in \mathcal{F}$. Suppose that
  \begin{itemize}
  \item  there  exists  $\delta  >0$  such   that,  for  every  $i  \in  [n]$,
    $g_i(a,w) \geq \delta$;
  \item there exists $B > 0$ such that
    $\sup_{f  \in  \mathcal{F}}  \sup_{a,w   \in  [K]  \times  \mathcal{W}}  |
    \phi(f(a,w)) - \phi(f_0(a,w))| \leq B$;
  \item there exists $v > 0$ such that
    $\sup_{f \in \mathcal{F}} \bar{V}_n(\phi(f) - \phi(f_0)) \leq v$,
    where,      for     any      pair      $(f,      g)$     of      functions
    $[K] \times \mathcal{W} \rightarrow \mathbb{R}_+$,
    $\bar{V}_n(g) \doteq n^{-1} \sum_{i=1}^t V(g_i, f)$
    (the  definition of  $V(g,f)$ is  given in  \eqref{eq:V(g,f)} in  the main
    text).
  \end{itemize}
  Then, for all $\alpha \in [0,B]$,
  \begin{align}
    P \left[  \sup_{f \in \mathcal{F}} M_n(f)  \geq H_n(\alpha, \delta, v,  B) +
    160 \sqrt{\frac{v x}{n}} + 3 \frac{B x}{\delta n} \right] \leq 2 e^{-x}, 
  \end{align}
  where
  \begin{align}
    M_n(f)  \doteq  \frac{1}{n}  \sum_{i=1}^n E\left[  \ell_{i}^{\phi}(O_i)  -
    \ell_{i}^{\phi}(f_0)(O_i)|           F_{i-1}           \right]           -
    \left(\ell_{i}^{\phi}(O_i)     -     \ell_{i}^{\phi}(f_0)(O_i)
    \right), \label{eq:def_M_t_of_f_eps_basic_IS} 
  \end{align}
  and
  \begin{align}
    H_n(\alpha,   \delta,  v,   B)  \doteq   \alpha  +   160  \sqrt{\frac{v}{n}}
    \int_{\alpha/2}^B  \sqrt{\log  (1  +  N_{[\,]}(\epsilon,  \phi(\mathcal{F}),
    L_2(P)))} d \epsilon + 3 \frac{B}{\delta n} \log 2. 
  \end{align}
\end{theorem}

\begin{proof}[Proof of theorem \ref{thm:max_ineq_IS_weighted_mart_process}]
  The    proof   follows    closely   the    proof   of    \cite[Theorem   A.4
  in][]{vanHandel2011}.

  \paragraph{From a conditional expectation bound to a deviation bound.}
  Let $x > 0$ and let $A$ be the event
  \begin{align}
    A \doteq \left\lbrace \sup_{f \in \mathcal{F}} M_{n}(f) \geq \psi(x) \right\rbrace,
  \end{align}
  with
  $\psi(x) \doteq H_n(\alpha, \delta, v, B)  + \sqrt{v x/n} + B x/(\delta n)$.
  Observe that, for any $x > 0$,
  \begin{align}
    \psi(x) \leq E^A_P \left[ \sup_{f \in \mathcal{F}} \Ind\{ \bar{V}_n(f) \leq v \} (P - P_n) \ell_{1:n} (f) \right].
  \end{align}
  Therefore, to prove the claim, it suffices to prove that
  \begin{align}
    E^A_P \left[ \sup_{f \in \mathcal{F}} \Ind  \{ \bar{V}_n(f) \leq v \} (P -
    P_n) \ell_{1:n}  (f) \right] \leq  \psi \left( \log \left(1  + \frac{1}{P[A]}
    \right) \right), 
  \end{align}
  as this would imply 
  \begin{align}
    \Psi(x) \leq  \psi \left( \log \left(1 + \frac{1}{P[A]} \right) \right) \leq \psi \left( \log \left(\frac{2}{P[A]} \right) \right),
  \end{align}
  which, as $\psi$ is increasing, implies $P[A] \leq 2 e^{-x}$, which is the wished claim.
  
\paragraph{Setting up the notation.} 
In this proof, we will denote
\begin{align}
H \doteq \left\lbrace \phi(f) - \phi(f_0) : f \in \mathcal{F} \right\rbrace.
\end{align}
Observe that by assumption $\mathcal{H}$ has diameter in $\|\cdot\|_\infty$ norm (and thus in $L_2(P)$ norm) smaller than $B$.
For all $j \geq 0$, let $\epsilon_j = B 2^{-j}$, and let 
\begin{align}
\mathcal{B}_j \doteq \{ (\underline{h}^{j,\rho}, \overline{h}^{j,\rho}) : \rho =1,\ldots, N_j \}
\end{align}
be an $\epsilon_j$-bracketing of $\mathcal{H}$ in $L_2(P)$ norm.
Further suppose that $\mathcal{B}_j$ is a minimal bracketing, that is that $N_j = N_{[\,]}(\epsilon_j, \mathcal{H}, L_2(P))$. For all $j, h$, let $\rho(j,h)$ be the index of a bracket in $\mathcal{B}_j$ that contains $h$, that is $\rho(j,h)$ is such that
\begin{align}
\underline{h}^{j, \rho(j,f)} \leq h \leq \overline{h}^{j, \rho(j,f)}.
\end{align}
For all $h \in \mathcal{H}$, $j \geq 0$, $i \in [n]$ let 
\begin{align}
\lambda^{j,h} \doteq \underline{h}^{j, \rho(j,h)},
\end{align}
and
\begin{align}
\Delta^{j,h}_i \doteq (h - \lambda^{j,h})(A_i,W_i).
\end{align}

\paragraph{Adaptive chaining.} The core idea of the proof is a so-called adaptive chaining device: for any $h$, and any $i \in [n]$, we write
\begin{align}
h(A_i,W_i) =& h(A_i,W_i) - \lambda^{\tau_i^h, h}(A_i,W_i) \vee \lambda^{\tau_i^h - 1, h}(A_i,W_i) \\
&+ \lambda^{\tau_i^h, h}(A_i,W_i) \vee \lambda^{\tau_i^h - 1, h}(A_i,W_i) - \lambda^{\tau_i^h - 1, h}(A_i,W_i) \\
&+ \sum_{j=1}^{\tau_i^h - 1} \lambda^{j,h}(A_i,W_i) \vee \lambda^{j-1,h}(A_i,W_i) \\
&+ \lambda^{0,h}(A_i,W_i), \label{eq:pf_chaining_policy_elim_chaining_eq1}
\end{align}
for some $\tau_i^h \geq 0$ that plays the role of the depth of the chain. We choose the depth $\tau_i^h$  so as to control the supremum norm of the links of the chain. Specifically, we let 
\begin{align}
\tau_i^h \doteq \min \left\lbrace j \geq 0: \Delta^{j,h}_i > a_j \right\rbrace \wedge J,
\end{align}
for some $J \geq 1$, and a decreasing positive sequence $a_j$, which we will explicitly specify later in the proof. The chaining decomposition in \ref{eq:pf_chaining_policy_elim_chaining_eq1} can be rewritten as follows:
\begin{align}
h(A_i,W_i) =& \lambda^{0,h}(A_i,W_i) \\
&+ \sum_{j=0}^J \left\lbrace h(A_i,W_i) - \lambda^{j,h} \vee \lambda^{j-1,h}(A_i,W_i) \right\rbrace \Ind\{\tau_i^h = j\} \\
&+ \sum_{j=1}^J \big\{ \left(\lambda^{j,h}(A_i,W_i) \vee \lambda^{j-1,h}(A_i,W_i) - \lambda^{j-1,h}(A_i,W_i)\right) \Ind\{\tau_i^h = j)\} \\
& \qquad \qquad + \left(\lambda^{j,h}(A_i,W_i) - \lambda^{j-1,h}(A_i,W_i) \right) \Ind\{\tau_i^h > j\} \big\}
\end{align}
Denote $a_i^h \doteq \lambda^{0,h}(A_i,W_i)$, 
\begin{align}
b_i^{j,h} \doteq \left\lbrace h(A_i,W_i) - \lambda^{j,h} \vee \lambda^{j-1,h}(A_i,W_i) \right\rbrace \Ind\{\tau_i^h = j\},
\end{align}
and 
\begin{align}
c^{j,h}_i \doteq & \left(\lambda^{j,h}(A_i,W_i) \vee \lambda^{j-1,h}(A_i,W_i) - \lambda^{j-1,h}(A_i,W_i)\right) \Ind\{\tau_i^h = j)\} \\
&+ \left(\lambda^{j,h}(A_i,W_i) - \lambda^{j-1,h}(A_i,W_i) \right) \Ind\{\tau_i^h > j\}.
\end{align}
Overloading the notation, we will denote, for every $i \in [n]$ and function $h : [K] \times \mathcal{W} \rightarrow \mathbb{R}$,
\begin{align}
\ell_{i}(h) \doteq \frac{h(A_i, W_i)(1-Y_i)}{g_i(A_i,W_i)}.
\end{align}
From the linearity of $ \ell_{1},\ldots, \ell_{n}$, we have that
\begin{align}
(P - P_n) \ell_{1:n}(h) = A_n^h + \sum_{j=0}^J B_n^{j,h} + \sum_{j=1}^J C_n^{j,h},
\end{align}
with 
\begin{align}
A_n^h \doteq & \frac{1}{n} \sum_{i=1}^n E[\ell_{i}(a_i^h) | F_{i-1} ] - \ell_{i}(a_i^h), \\
B_n^{j,h} \doteq & \frac{1}{n} \sum_{i=1}^n  E[\ell_{i}(b_i^{j,h})| F_{i-1}] - \ell_{i}(b_i^{j,h}), \\
C_n^{j,h} \doteq & \frac{1}{n} \sum_{i=1}^n E[\ell_{i}(c_i^{j,h}| F_{i-1} ] - \ell_{i}(c_i^{j,h}).
\end{align}
The terms $A_n^h$, $B_n^{j,h}$ and $C_n^{j,h}$ can be intepreted as follows. For any given $h$ and chain corresponding to $h$:
\begin{itemize}
\item $A_n^h$ represents the root, at the coarsest level, of the chain,
\item if the chain goes deeper than depth $j$, $C_n^{j,h}$ is the link of the chain between  depths $j-1$ and $j$,
\item if the chain stops at depth $j$, $B_n^{j,h}$ is the tip of the chain.
\end{itemize}
We control each term separately.

\paragraph{Control of the roots.} Observe that, for all $i \in [n]$,
\begin{align}
  E_P[ \ell_{i}(a_i^h)^2 | F_{i-1}] =& E_P \left[ \frac{\lambda^{0,h}(A_i,W_i)^2}{g_i(A_i|W_i)^2} (1-Y_i)^2 | F_{i-1}\right] \\
  \leq & E_P \left[ \sum_{a\in[K]} \frac{\left(\lambda^{0,h}(a,W_i) - h(a,W_i) + h(a,W_i)\right)^2}{g_i(a,W_i)} \bigg| F_{i-1} \right] \\
  \leq & 2 \delta^{-1} \left\lbrace E_P \left[ \sum_{a\in[K]} \left(\lambda^{0,h}(a,W_i) - h(a,W_i)\right)^2 \bigg| F_{i-1} \right] + E_P \left[ \sum_{a\in[K]} h(a,W_i)^2 \bigg| F_{i-1} \right] \right\rbrace \\
  \leq & 4 K \delta^{-1} \epsilon_0^2.
\end{align}
In the second line we have used that $(1-Y_i) \in [0,1]$. As $\|\lambda^{0,h}\|_\infty \leq B$ and $\inf_{a,w} g_i(a,w) \geq \delta$, we have that
\begin{align}
|\ell_{i}(a_i^h)| \leq B \delta^{-1}.
\end{align}
Therefore, from lemma \ref{lemma:bernstein-max_ineq_finite_sets},
\begin{align}
E^A_P \left[ \sup_{h \in \mathcal{H}} A_n^h \right] \leq 8 \epsilon_0 \sqrt{\frac{K}{\delta} \log\left(1 + \frac{N_0}{P[A]}\right)} + \frac{8}{3} \frac{B}{ \delta n} \log \left(1 + \frac{N_0}{P[A]} \right).
\end{align}

\paragraph{Control of the tips.} 

As $\lambda_i^{j,h}$ is a lower bracket, $\ell_{i}(b_i^{j,h}) \leq 0$ and thus
\begin{align}
E_P\left[ \ell_{i}(b_i^{j,h}) | F_{i-1} \right] - \ell_{i}(b_i^{j,h}) \leq & E_P \left[ \ell_{i}(b_i^{j,h}) | F_{i-1} \right] \\
=& E_P \left[ \frac{(h(A_i,W_i) - \lambda^{j,h}(A_i,W_i) \vee \lambda^{j,h}(A_i,W_i)}{g_i(A_i,W_i)}(1-Y_i) \Ind \{\tau_i^h = j\} \bigg| F_{i-1} \right] \\
\leq & E_P \left[ \frac{h(A_i,W_i) - \lambda^{j,h}(A_i,W_i) \vee \lambda^{j-1, h}(A_i,W_i)}{g_i(A_i,W_i)} \Ind \{ \tau_i^h =j\} \bigg| F_{i-1} \right] \\
\leq & E_P \left[\frac{\Delta_i^{j,h} \Ind \{\tau_i^h = j \}}{g_i(A_i,W_i)}  \bigg| F_{i-1} \right].
\end{align}
We treat separately the case $j < J$ and the case $j = J$. We first start with the case $j < J$. If $\tau_i^h = j$, we must then have $\Delta_i^{j,h} > a_j$, which implies that
\begin{align}
E \left[ \ell_{i}(b_i^{j,h}) \big| F_{i-1} \right] - \ell_{i}(b_i^{j,f}) =& E_P \left[\frac{\Delta_i^{j,h} \Ind \{\tau_i^h = j \}}{g_i(A_i,W_i)}  \bigg| F_{i-1} \right]\\
\leq & \frac{1}{a_j} E \left[ \frac{(\Delta_i^{j,h})^2}{g_i(A_i,W_i)} \bigg| F_{i-1} \right] \\
\leq & \frac{1}{a_j} E \left[ \sum_{a\in[K]} \left(h(a,W_i)-\lambda^{j,h}(a,W_i) \right)^2 \big| F_{i-1} \right] \\
\leq & \frac{K \epsilon_j^2}{a_j}.
\end{align}
Therefore, for $j < J$, 
\begin{align}
E^A_P\left[\sup_{h \in \mathcal{F}}  B_n^{j,f}\right] \leq \frac{K \epsilon_j^2}{a_j}.
\end{align}
Now consider the case $j=J$. We have that
\begin{align}
  B_n^{J,h}  \leq & \frac{1}{n} \sum_{i=1}^n E_P \left[ \Delta_i^{J,h} \big| F_{i-1} \right] \\
  =& \sum_{i=1}^n E_P \left[ \frac{h(A_i,W_i) - \lambda^{J,h}(A_i,W_i)}{g_i(A_i,W_i)} \bigg| F_{i-1} \right] \\
  \leq & \frac{1}{n} \sum_{i=1}^n E_P \left[ \sum_{a\in[K]}  h(a,W_i) - \lambda^{J,h}(a,W_i) \bigg| F_{i-1} \right] \\
  \leq & \frac{1}{n} \sqrt{n} \left( \sum_{i=1}^n E_P \left[ \left( \sum_{a\in[K]} h(a,W_i) - \lambda^{J,h}(a,W_i) \right)^2 \bigg| F_{i-1} \right] \right)^{1/2} \\
  \leq& \left( \frac{K}{n} \sum_{i=1}^n E_P \left[\sum_{a\in[K]} (h(a,W_i) - \lambda^{J,h}(a,W_i) )^2 | F_{i-1} \right] \right)^{1/2}\\
  \leq & K \epsilon_J.
\end{align}
Therefore,
\begin{align}
E^A_P \left[ \sup_{h \in \mathcal{H}} B_n^{J,h} \right] \leq K	\epsilon_J.
\end{align}

\paragraph{Control of the links.} Observe that $\lambda^{j,} - \lambda^{j-1,h} = \lambda^{j,h} - h + h - \lambda^{j-1,h}$. Using that $\lambda^{j,h} \leq h$ and $\lambda^{j-1,h} \leq h$ the definitions of $\Delta^{j,h}_i$ and $\Delta^{j-1,h}$ yield
\begin{align}
&-\Delta^{j,h}_i \leq (\lambda^{j,h} - h)(A_i,W_i) \Ind \{ \tau_i^h > j \} \leq 0, \\
\text{and } & 0  \leq (h- \lambda^{j-1,h})(A_i,W_i) \Ind \{ \tau_i^h \geq j \} \leq \Delta_i^{j-1,h}.
\end{align}
Therefore, recalling the definition of $c^{j,h}_i$, we have that
\begin{align}
-\Delta_i^{j,h} \Ind \{ \tau_i^h >j \} \leq c_i^{j,h} \leq \Delta_i^{j-1,h} \Ind \{\tau_i^h \geq j \}.
\end{align}
Applying $\ell_{i}$ to $c_i^{j,h}$ amounts to multiplying it with a non-negative random variable. Therefore,
\begin{align}
-\ell_{i}(\Delta_i^{j,h} \Ind \{\tau_i^h > j \} \leq \ell_{i}(c_i^{j,h} \leq) \leq \ell_{i}(\Delta_i^{j,h} \Ind \{\tau_i^{j,h} \geq j \}),
\end{align}
and then
\begin{align}
|\ell_{i}(c_i^{j,f})| \leq \Delta_i^{j,f} \Ind \{ \tau_i^f > j \} \vee \Delta_i^{j-1,f} \Ind\{\tau_i^f \geq j \}.
\end{align}
From the definition of $\tau_i^{j,f}$ and the fact that $(1-Y_i) \in [0,1]$, we have that
\begin{align}
|\ell_{i}(c_i^{j,f})| \leq a_j \vee a_{j-1}.
\end{align}
Besides,
\begin{align}
E_P \left[ \ell_{i}(c_i^{j,f})^2 | F_{i-1} \right] \leq 2 \left\lbrace E_P \left[\ell_{i}(\Delta_i^{j,f})^2 | F_{i-1} \right] + E_P \left[\ell_{i}(\Delta_i^{j-1,f})^2 | F_{i-1} \right] \right\rbrace.
\end{align}
We have that, for all $j$,
\begin{align}
E_P \left[ \ell_{i}(\Delta_i^{j,h})^2 \bigg| F_{i-1} \right] =& E_P \left[ \frac{(f(A_i,W_i) - \lambda^{j,h}(A_i,W_i))^2(1-Y_i)^2}{g_i^2(A_i,W_i)} \bigg| F_{i-1} \right] \\
\leq & E_P \left[ \sum_{a\in[K]} \frac{(f(a,W_i) - \lambda^{j,h}(a,W_i))^2}{g_i(a,W_i)} \bigg| F_{i-1} \right] \\
\leq & \delta^{-1} K \epsilon_j^2.
\end{align}
Therefore, for all $i$, $j$, 
\begin{align}
E_P \left[ (\ell_{i}(c_i^{j,h}))^2 | F_{i-1} \right] \leq \delta^{-1} K (\epsilon_{j-1}^2 + \epsilon_j^2).
\end{align}
Observe that $C_n^{j,h}$ depends on $h$ only through $\rho(0,h)$,\ldots,$\rho(j,h)$. Therefore, as $h$ varies over $\mathcal{H}$, $C_n^{j,h}$ varies over a collection of at most 
\begin{align}
\bar{N}_j \doteq \prod_{k=0}^j N_k
\end{align}
random variables. Therefore, from lemma \ref{lemma:bernstein-max_ineq_finite_sets},
\begin{align}
E^A_P\left[\sup_{h \in \mathcal{H}} C_n^{j,h}\right] \leq 4 \sqrt{\frac{2 K (\epsilon_j^2 + \epsilon_{j-1}^2)}{\delta n} \log \left(1 + \frac{\bar{N}_j}{P[A]} \right) } +  \frac{8}{3} \frac{a_j \vee a_{j-1}}{ \delta n} \log \left(1 + \frac{\bar{N}_j}{P[A]} \right)
\end{align}

\paragraph{End of the proof.} Collecting the bounds on $E^A_P[\sup_{h \in \mathcal{H}} B_n^{j,h}]$, $E^A_P[\sup_{h \in \mathcal{H}} B_n^{j,h}]$ and $E^A_P[\sup_{h \in \mathcal{H}} C_n^{j,h}]$ yields 
\begin{align}
E^A_P \left[\sup_{h\in\mathcal{H}}  (P-P_n) \ell_{1:n}(h) \right] \leq & K \epsilon_J +\sum_{j=0}^{J-1} \frac{K \epsilon_j^2}{a_j}\\
& + 8 \sqrt{\frac{K}{\delta n} \log\left(1 + \frac{N_0}{P[A]}\right)} + \frac{8}{3}  \frac{B}{\delta n} \log \left(1 + \frac{N_0}{P[A]} \right)  \\
&+ \sum_{j=1}^J 8 \sqrt{\frac{K}{\delta n} \log \left( 1 + \frac{\bar{N}_j}{P[A]} \right) } \\
&+ \sum_{j=1}^J \frac{8}{3} \frac{a_{j-1}}{ \delta n} \log \left(1 + \frac{\bar{N}_j}{P[A]} \right).
\end{align}
Set $$a_j = \epsilon_j \sqrt{\frac{\delta n}{K \log (1+\bar{N}_j / P[A])}}.$$
Replacing $a_j$ in the previous display yields
\begin{align}
E^A_P \left[\sup_{f\in\mathcal{F}} (P-P_n) \ell_{1:n}(f) \right] \leq & K \epsilon_J + \frac{8}{3} \frac{B}{\delta n} \log \left( 1 + \frac{N_0}{P[A]} \right) \\
& + 20 \sum_{j=0}^{J-1} \epsilon_j \sqrt{\frac{K}{ \delta n} \log \left( 1 + \frac{\bar{N}_{j+1}}{P[A]} \right) }.
\end{align}
Since $(1 + \bar{N}_j / P[A]) \leq (1 + 1/P[A]) \prod_{k=0}^j (1 + N_k)$, we have
\begin{align}
\sum_{j=1}^J \epsilon_{j-1} \sqrt{\log \left( 1  + \frac{\bar{N}_j}{P[A]}  \right) } \leq & 2 \sum_{j=0}^J \epsilon_j \sqrt{ \log \left(1 + \frac{1}{P[A]} \right) + \sum_{k=0}^j \log (1 + N_k) } \\
\leq & 2\left( \sum_{j=0}^J \epsilon_j \right) \sqrt{\log \left( 1 + \frac{1}{P[A]} \right)} + 2\sum_{j=0}^J \epsilon_j \sum_{k=0}^j \sqrt{\log (1 + N_k)}
\end{align}
We first look at the second term. We have that
\begin{align}
\sum_{j=0}^J \epsilon_j \sum_{k=0}^j \sqrt{\log (1 + N_k)} =& \sum_{k=0}^J \sqrt{\log (1 + N_k)} \sum_{j=k}^J 2^{-j} \\
\leq & 2 \sum_{k=j}^J 2^{-k} \sqrt{\log (1 + N_k)} \\
=& 4 \sum_{k=0}^J (\epsilon_k - \epsilon_{k+1} ) \sqrt{\log (1 + N_k)} \\
\leq & 4 \int_{\alpha/2}^B \sqrt{\log (1 + N_{[\,]}(\epsilon, \mathcal{F}, L_2(P)))} d\epsilon.
\end{align}
Therefore, observing that $\sum_{j=0}^J \epsilon_j \leq 2$, and gathering the previous bounds yields that
\begin{align}
&E^A_P \left[\sup_{h \in \mathcal{H}}  (P-P_n) \ell_{1:n}(h) \right]\\
 \leq & K \epsilon_J + 160 \sqrt{\frac{K}{\delta n}} \int_{\alpha / 2}^B \sqrt{\log (1 + N_{[\,]}(u^2, \mathcal{F}, L_\infty(P))} du \\
 &+ \frac{8}{3} \frac{B}{\delta n} \log\left(1 + N_{[\,]}(1, \mathcal{F}, L_\infty(P))\right) \\
 &+ 160 \sqrt{\frac{K }{ \delta n}} \sqrt{ \log \left(1  + \frac{1}{P[A]} \right)} + \frac{8}{3} \frac{B}{\delta n} \log \left(1 + \frac{1}{P[A]} \right) \\
 \leq & H_n(v, \delta, \alpha) + 160 \sqrt{\frac{v}{n}} \sqrt{ \log \left(1  + \frac{1}{P[A]} \right)} + 3\frac{B}{ \delta n} \log \left(1 + \frac{1}{P[A]} \right),
\end{align} 
with 
\begin{align}
H_n(v, \delta, \alpha) \doteq K \alpha + 160 \sqrt{\frac{K}{\delta n}} \int_{\alpha / 2}^B \sqrt{\log (1 + N_{[\,]}(u^2, \mathcal{F}, L_\infty(P)))} du +  3\frac{B}{\delta n} \log \left(1 + N_{[\,]}(1, \mathcal{F}, L_\infty(P))\right).
\end{align}
\end{proof}

\subsection{Maximal inequality for policy elimination}

\begin{theorem}[Maximal inequality under parameter-dependent IS ratio bound]\label{thm:max_ineq_under_param_dependent_IS_bound}
  Let       $\mathcal{F}$       be       a      class       of       functions
  $\mathcal{A}  \times \mathcal{W}  \rightarrow [0,1]$.   Suppose that  we are
  under the contextual bandit setting described earlier, and that $g_i$ is the
  $F_{i-1}$-measurable design at time point $i$.  Let, for any $i \geq 1$, any
  $f \in \mathcal{F}$,
\begin{align}
V_i(f) \doteq E_P\left[ \sum_{a\in[K]} \frac{f(a|W)}{g(a|W)} \right].
\end{align}
For any $n \geq 1$, $f \in \mathcal{F}$, denote
\begin{align}
\bar{V}_n(f) \doteq \frac{1}{n} \sum_{i=1}^n V_i(f).
\end{align}
Let $l$ be the direct policy optimization loss, and for all $i$, let $\ell_{i}$ be its importance-sampling weighted counterpart for time point $i$, that is, for all $f \in \mathcal{F}$, $o = (w, a, y) \in \mathcal{O}$,
\begin{align}
l(f)(o) \doteq & \sum_{a\in[K]} C(a,W) f(a,W)\\
\text{ and } \ell_{i}(f)(o) \doteq & \frac{f(a|w)}{g_i(a|w)}(1-y).
\end{align}
Suppose   that   there   exists   $\delta    >   0$   such,   that   for   all
$a,w\in    \mathcal{A}    \times    \mathcal{W}$     and    $i    \in    [n]$,
$g_i(a|w) \geq \delta$.

Then, for all $x > 0$, $v > 0$, $\epsilon \in [0,1]$
\begin{align}
P \left[ \sup_{f \in \mathcal{F}} \Ind\{\bar{V}_n(f) \leq v \} (P - P_n) \ell_{1:n}(f) \geq  H_n(v, \delta, \epsilon) + 37 \sqrt{\frac{v x}{n}} + 3 \frac{x}{\delta n} \right] \leq 2 e^{-x},
\end{align}
with 
\begin{align}
H_n(v, \delta, \epsilon) \doteq \sqrt{v \epsilon} + 127 \sqrt{\frac{v}{n}} \int_{\sqrt{\epsilon / 2}}^1 \sqrt{\log (1 + N_{[\,]}(u^2, \mathcal{F}, L_\infty(P)))} du +  \frac{3}{\delta n} \log \left(1 + N_{[\,]}(1, \mathcal{F}, L_\infty(P))\right).
\end{align}
\end{theorem}

The proof of the preceding theorem relies on the following lemma, which is a direct corollary of corollary A.8 in \cite{vanHandel2011}.

\begin{lemma}[Bernstein-like maximal inequality for finite sets]\label{lemma:bernstein-max_ineq_finite_sets}
  Let, for any $i \in [n], j \in [N]$, $X_{i,j}$ be an $F_i$-measurable random
  variable,  and  let,  for  any  $j  \in  [N]$,  $M_{t}^j  \doteq  \sum_{i=1}^n
  X_{i,j}$. Let for all $j \in [N]$,
\begin{align}
\sigma_{n,j}^2 \doteq \frac{1}{n} \sum_{i=1}^n E_P[X_{i,j}^2 | F_{i-1} ].
\end{align}
Suppose that for all $i \in [n]$, $j \in [N]$, $|X_{i,j}| \leq b$ a.s. for some $b \geq 0$. Then, for any event $A \in \mathcal{F}$,
\begin{align}
E^A \left[ \max_{j \in [N]} \Ind\{\sigma_{n,j}^2 \leq \sigma^2\} M_{t}^j \right] \leq 4 \sigma \sqrt{\log \left( 1 + \frac{N}{P[A]} \right)} + \frac{8}{3} b \log \left(1 + \frac{N}{P[A]} \right).
\end{align}
\end{lemma}

\begin{proof}[Proof of lemma \ref{lemma:bernstein-max_ineq_finite_sets}]
Observe that
\begin{align}
\frac{2 b^2}{n} \sum_{i=1}^n E \left[ \phi \left( \frac{X_i}{b} \right) \bigg| F_{i-1} \right] \leq& \frac{2 b^2}{n} \sum_{i=1}^n \sum_{k \geq 2} \frac{b^{k-2}}{b^k k!} E[X_i^2 | F_{i-1}] \\
\leq & \frac{2}{n} \sum_{k \geq 2} \frac{1}{k!} \sum_{i = 1}^n E[X_i^2 | F_{i-1} ] \\
\leq & 2 \phi(1) \sigma_{n,j}^2 \\
\leq & 2 \sigma_{n,j}^2 .
\end{align}
The conclusion follows from corollary A.8 in \cite{vanHandel2011}.
\end{proof}

\begin{proof}[Proof of theorem \ref{thm:max_ineq_under_param_dependent_IS_bound}]
The proof follows closely the proof of theorem A.4 in \cite{vanHandel2011}
\paragraph{From a conditional expectation bound to a deviation bound.}
Let $x > 0$ and let $A$ be the event
\begin{align}
A \doteq \left\lbrace \sup_{f \in \mathcal{F}} \Ind\{ \bar{V}_n(f_ \leq v \} (P - P_n) \ell_{1:n}(f) \geq \psi(x) \right\rbrace,
\end{align}
with 
\begin{align}
\psi(x) =  H_n(v, \delta, \epsilon) + 37 \sqrt{\frac{v x}{n}} + 3 \frac{x}{\delta n} 
\end{align}
Observe that for any $x > 0$, 
\begin{align}
\psi(x) \leq E^A_P \left[ \sup_{f \in \mathcal{F}} \Ind \{ \bar{V}_n(f) \leq v \} (P - P_n) \ell_{1:n} (f) \right].
\end{align}
Therefore, to prove the claim, it suffices to prove that 
\begin{align}
 E^A_P \left[ \sup_{f \in \mathcal{F}} \Ind \{ \bar{V}_n(f) \leq v \} (P - P_n) \ell_{1:n} (f) \right] \leq \psi \left( \log \left(1 + \frac{1}{P[A]} \right) \right),
\end{align}
as this would imply 
\begin{align}
\Psi(x) \leq  \psi \left( \log \left(1 + \frac{1}{P[A]} \right) \right) \leq \psi \left( \log \left(\frac{2}{P[A]} \right) \right),
\end{align}
which, as $\psi$ is increasing, implies $P[A] \leq 2 e^{-x}$, which is the wished claim.

\paragraph{Setting up the notation.} For all $j \geq 0$, let $\epsilon_j = 2^{-j}$, and let 
\begin{align}
\mathcal{B}_j \doteq \{ (\underline{f}^{j,\rho}, \overline{f}^{j,\rho}) : \rho =1,\ldots, N_j \}
\end{align}
be an $\epsilon_j$-bracketing of $\mathcal{F}$ in $L_\infty(P)$ norm.
Further suppose that $\mathcal{B}_j$ is a minimal bracketing, that is that $N_j = N_{[\,]}(\epsilon_j, \mathcal{F}, L_\infty(P))$. For all $j, f$, let $\rho(j,f)$ be the index of a bracket of $\mathcal{B}_j$ that contains $f$, that is $\rho(j,f)$ is such that
\begin{align}
\underline{f}^{j, \rho(j,f)} \leq f \leq \overline{f}^{j, \rho(j,f)}.
\end{align}
For all $f \in \mathcal{F}$, $j \geq 0$, $i \in [n]$ let 
\begin{align}
\lambda^{j,f} \doteq \underline{f}^{j, \rho(j,f)},
\end{align}
and
\begin{align}
\Delta^{j,f}_i \doteq (f - \lambda^{j,f})(A_i,W_i).
\end{align}

\paragraph{Adaptive chaining.} The core idea of the proof is a so-called adaptive chaining device: for any $f$, and any $i \in [n]$, we write
\begin{align}
f(A_i,W_i) =& f(A_i,W_i) - \lambda^{\tau_i^f, f}(A_i,W_i) \vee \lambda^{\tau_i^f - 1, f}(A_i,W_i) \\
&+ \lambda^{\tau_i^f, f}(A_i,W_i) \vee \lambda^{\tau_i^f - 1, f}(A_i,W_i) - \lambda^{\tau_i^f - 1, f}(A_i,W_i) \\
&+ \sum_{j=1}^{\tau_i^f - 1} \lambda^{j,f}(A_i,W_i) \vee \lambda^{j-1,f}(A_i,W_i) \\
&+ \lambda^{0,f}(A_i,W_i), \label{eq:pf_chaining_policy_elim_chaining_eq1}
\end{align}
for some $\tau_i^f \geq 0$ that plays the role of the depth of the chain. We choose the depth $\tau_i^f$  so as to control the supremum norm of the links of the chain. Specifically, we let 
\begin{align}
\tau_i^f \doteq \min \left\lbrace j \geq 0: \frac{\Delta^{j,f}_i}{g_i(A_i,W_i)} > a_j \right\rbrace \wedge J,
\end{align}
for some $J \geq 1$, and a decreasing positive sequence $a_j$, which we will explicitly specify later in the proof. The chaining decomposition in \ref{eq:pf_chaining_policy_elim_chaining_eq1} can be rewritten as follows:
\begin{align}
f(A_i,W_i) =& \lambda^{0,f}(A_i,W_i) \\
&+ \sum_{j=0}^J \left\lbrace f(A_i,W_i) - \lambda^{j,f} \vee \lambda^{j-1,f}(A_i,W_i) \right\rbrace \Ind\{\tau_i^f = j\} \\
&+ \sum_{j=1}^J \big\{ \left(\lambda^{j,f}(A_i,W_i) \vee \lambda^{j-1,f}(A_i,W_i) - \lambda^{j-1,f}(A_i,W_i)\right) \Ind\{\tau_i^f = j)\} \\
& \qquad \qquad + \left(\lambda^{j,f}(A_i,W_i) - \lambda^{j-1,f}(A_i,W_i) \right) \Ind\{\tau_i^f > j\} \big\}
\end{align}
Denote $a_i^f \doteq \lambda^{0,f}(A_i,W_i)$, 
\begin{align}
b_i^{j,f} \doteq \left\lbrace f(A_i,W_i) - \lambda^{j,f} \vee \lambda^{j-1,f}(A_i,W_i) \right\rbrace \Ind\{\tau_i^f = j\},
\end{align}
and 
\begin{align}
c^{j,f}_i \doteq & \left(\lambda^{j,f}(A_i,W_i) \vee \lambda^{j-1,f}(A_i,W_i) - \lambda^{j-1,f}(A_i,W_i)\right) \Ind\{\tau_i^f = j)\} \\
&+ \left(\lambda^{j,f}(A_i,W_i) - \lambda^{j-1,f}(A_i,W_i) \right) \Ind\{\tau_i^f > j\}.
\end{align}
From the linearity of $ \ell_{1},\ldots, \ell_{n}$, we have that
\begin{align}
(P - P_n) \ell_{1:n}(f) = A_n^f + \sum_{j=0}^J B_n^{j,f} + \sum_{j=1}^J C_n^{j,f},
\end{align}
with 
\begin{align}
A_n^f \doteq & \frac{1}{n} \sum_{i=1}^n E[\ell_{i}(a_i^f) | F_{i-1} ] - \ell_{i}(a_i^f), \\
B_n^{j,f} \doteq & \frac{1}{n} \sum_{i=1}^n  E[\ell_{i}(b_i^{j,f})| F_{i-1}] - \ell_{i}(b_i^{j,f}), \\
C_n^{j,f} \doteq & \frac{1}{n} \sum_{i=1}^n E[\ell_{i}(c_i^{j,f}| F_{i-1} ] - \ell_{i}(c_i^{j,f}).
\end{align}
The terms $A_n^f$, $B_n^{j,f}$ and $C_n^{j,f}$ can be intepreted as follows. For any given $f$ and chain corresponding to $f$:
\begin{itemize}
\item $A_n^f$ represents the root, at the coarsest level, of the chain,
\item if the chain goes deeper than depth $j$, $C_n^{j,f}$ is the link of the chain between  depths $j-1$ and $j$,
\item if the chain stops at depth $j$, $B_n^{j,f}$ is the tip of the chain.
\end{itemize}
We control each term separately.

\paragraph{Control of the roots.} Observe that, for all $i \in [n]$, $|\ell_{i}(a_i^f)| \leq \delta^{-1}$ a.s., and that 
\begin{align}
E_P[ \ell_{i}(a_i^f)^2 | F_{i-1}] =& E_P \left[ \frac{\lambda^{0,f}(A_i,W_i)^2}{g_i(A_i|W_i)^2} (1-Y_i)^2 | F_{i-1}\right] \\
\leq & E_P \left[ \sum_{a\in[K]} \frac{f(a,W_i)}{g_i(A_i|W_i)}  | F_{i-1} \right] \\
=& V_i(f).
\end{align}
In the second line we have used that, $\lambda^{0,f}(A_i,W_i) \leq f(A_i,W_i)$, that $(1-Y_i) \in [0,1]$, and that $f(a,W_i) \in [0,1]$.
Therefore, from lemma \ref{lemma:bernstein-max_ineq_finite_sets},
\begin{align}
E^A_P \left[\sup_{f \in \mathcal{F}} \Ind\{\bar{V}_n(f) \leq v\} A_n^f \right] \leq 4 \sqrt{\frac{v}{n} \log\left(1 + \frac{N_0}{P[A]}\right)} + \frac{8}{3 \delta n} \log \left(1 + \frac{N_0}{P[A]} \right).
\end{align}

\paragraph{Control of the tips.} As $\ell_{i}(b_i^{j,f}) \leq 0$, we have that
\begin{align}
E_P[\ell_{i}(b_i^{j,f}) | F_{i-1}] - \ell_{i}(b_i^{j,f}) \leq & E_P [ \ell_{i}(b_i^{j,f}) | F_{i-1}] \\
=& E_P \left[ \frac{f(A_i, W_i) - \lambda^{j,f}(A_i,W_i) \vee \lambda^{j-1,f}(A_i,W_i)}{g_i(A_i,W_i)}(1-Y_i) \Ind\{\tau_i^f = j\} \bigg| F_{i-1} \right] \\
\leq & E_P \left[ \frac{\Delta_i^{j,f}}{g_i(A_i,W_i)} \Ind\{\tau_i^f = j\} \bigg| F_{i-1} \right]
\end{align}
We treat separately the case $j < J$ and the case $j = J$. We first start with the case $j < J$. If $\tau_i^f = j$, we must then have $\Delta_i^{j,f}/g_i(A_i,W_i) > a_j$, which implies that
\begin{align}
E \left[ \ell_{i}(b_i^{j,f}) \big| F_{i-1} \right] - \ell_{i}(b_i^{j,f}) \leq & \frac{1}{a_j} E \left[ \frac{(\Delta_i^{j,f})^2}{g_i^2(A_i,W_i)} \bigg| F_{i-1} \right] \\
\leq & \frac{1}{a_j} E \left[ \sum_{a\in[K]} \frac{f(a,W_i)}{g_i(a,W_i)} (f(a,W_i)-\lambda^{j,f}(a,W_i) ) \bigg| F_{i-1} \right] \\
\leq & \frac{1}{a_j} V_i(f) \epsilon_j.
\end{align}
The second line above follows from the fact that $0 \leq f - \lambda^{j,f} \leq f$ since $0 \leq \lambda^{j,f} \leq f$. The third line above follows from the fact that $0 \leq (f -\lambda^{j,f})(a,W_i) \leq \|f - \lambda^{j,f} \|_{\infty} \leq \epsilon_j$. Therefore, for $j < J$, 
\begin{align}
E^A_P\left[\sup_{f \in \mathcal{F}} \Ind\{\bar{V}_n(f) \leq v \} B_n^{j,f}\right] \leq \frac{1}{a_j} v \epsilon_j.
\end{align}
Now consider the case $j=J$. We have that
\begin{align}
B_n^{J,f} \leq & \frac{1}{n} \sum_{i=1}^n E_P\left[ \frac{(f - \lambda^{J,f})(A_i,W_i)}{g_i(A_i,W_i)} \bigg| F_{i-1} \right] \\
\leq & \frac{1}{n} \sqrt{n} \left(\sum_{i=1}^n E_P\left[ \frac{(f - \lambda^{J,f} )^2(A_i,W_i)}{g_i^2(A_i,W_i)} \bigg| F_{i-1} \right] \right)^{1/2} \\
\leq & \left( \frac{1}{n} \sum_{i=1}^n E_P \left[ \sum_{a\in[K]} \frac{f(a,W_i)}{g_i(a|W_i)} (f(a,W_i) - \lambda(a,W_i)) | F_{i-1} \right] \right)^{1/2} \\
\leq & \left( \frac{1}{n} \sum_{i=1}^n V_i(f) \epsilon_J\right)^{1/2} \\
\leq & \sqrt{v \epsilon_J}.
\end{align}
The second line follows from Cauchy-Schwartz and Jensen. The third line uses the same arguments as in the case $j < J$ treated before.
Therefore,
\begin{align}
E^A_P \left[ \sup_{f \in \mathcal{F}} \Ind\{ \bar{V}_n(f) \leq v \} B_n^{J,f} \right] \leq \sqrt{v \epsilon_J}.
\end{align}

\paragraph{Control of the links.} Observe that $\lambda^{j,f} - \lambda^{j-1,f} = \lambda^{j,f} - f + f - \lambda^{j-1,f}$. Using that $\lambda^{j,f} \leq f$ and $\lambda^{j-1,f} \leq f$ the definitions of $\Delta^{j,f}_i$ and $\Delta^{j-1,f}$ yields
\begin{align}
&-\Delta^{j,f}_i \leq (\lambda^{j,f} - f)(A_i,W_i) \Ind \{ \tau_i^f > j \} \leq 0, \\
\text{and } & 0  \leq (f- \lambda^{j-1,f})(A_i,W_i) \Ind \{ \tau_i^f \geq j \} \leq \Delta_i^{j-1,f}.
\end{align}
Therefore, recalling the definition of $c^{j,f}_i$, we have that
\begin{align}
-\Delta_i^{j,f} \Ind \{ \tau_i^f >j \} \leq c_i^{j,f} \leq \Delta_i^{j-1,f} \Ind \{\tau_i^f \geq j \}.
\end{align}
Applying $\ell_{i}$ to $c_i^{j,f}$ amounts to multiplying it with a non-negative random variable. Therefore,
\begin{align}
-\ell_{i}(\Delta_i^{j,f} \Ind \{\tau_i^f > j \} \leq \ell_{i}(c_i^{j,f} \leq) \leq \ell_{i}(\Delta_i^{j,f} \Ind \{\tau_i^{j,f} \geq j \}),
\end{align}
and then
\begin{align}
|\ell_{i}(c_i^{j,f})| \leq \Delta_i^{j,f} \Ind \{ \tau_i^f > j \} \vee \Delta_i^{j-1,f} \Ind\{\tau_i^f \geq j \}.
\end{align}
From the definition of $\tau_i^{j,f}$ and the fact that $(1-Y_i) \in [0,1]$, we have that
\begin{align}
|\ell_{i}(c_i^{j,f})| \leq a_j \vee a_{j-1}.
\end{align}
Besides,
\begin{align}
E_P \left[ \ell_{i}(c_i^{j,f})^2 | F_{i-1} \right] \leq 2 \left\lbrace E_P \left[\ell_{i}(\Delta_i^{j,f})^2 | F_{i-1} \right] + E_P \left[\ell_{i}(\Delta_i^{j-1,f})^2 | F_{i-1} \right] \right\rbrace.
\end{align}
We have that, for all $j$,
\begin{align}
E_P \left[ (\ell_{i}(\Delta_i^{j,f}))^2 | F_{i-1} \right] =& E_P \left[ \frac{(f(A_i,W_i) - \lambda^{j,f}(A_i,W_i))^2}{g_i(A_i|W_i)^2} (1-Y_i)^2 \bigg| F_{i-1} \right] \\
\leq & E_P \left[ \sum_{a\in[K]} \frac{f(a, W_i)}{g_i(a|W_i)} (f(a,W_i) -\lambda^{j,f}(a,W_i)) \bigg| F_{i-1} \right] \\
\leq & V_i(f) \epsilon_j.
\end{align}
Therefore, for all $i$, $j$, 
\begin{align}
E_P \left[ (\ell_{i}(c_i^{j,f}))^2 | F_{i-1} \right] \leq V_i(f) (\epsilon_j + \epsilon_{j-1}).
\end{align}
Observe that $C_n^{j,f}$ depends on $f$ only through $\rho(0,f)$,\ldots,$\rho(j,f)$. Therefore, as $f$ varies over $\mathcal{F}$, $C_n^{j,f}$ varies over a collection of at most 
\begin{align}
\bar{N}_j \doteq \prod_{k=0}^j N_k
\end{align}
random variables. Therefore, from lemma \ref{lemma:bernstein-max_ineq_finite_sets},
\begin{align}
E^A_P\left[\sup_{f \in \mathcal{F}} \Ind\{\bar{V}_n(f) \leq v \} C_n^{j,f}\right] \leq 4 \sqrt{\frac{v (\epsilon_j + \epsilon_{j-1})}{n} \log \left(1 + \frac{\bar{N}_j}{P[A]} \right) } +  \frac{8}{3} \frac{a_j \vee a_{j-1}}{n} \log \left(1 + \frac{\bar{N}_j}{P[A]} \right)
\end{align}

\paragraph{End of the proof.} Collecting the bounds on $E^A_P[\sup_{f \in \mathcal{F}} \Ind\{\bar{V}_n(f) \leq v \} B_n^{j,f}]$, $E^A_P[\sup_{f \in \mathcal{F}} \Ind\{\bar{V}_n(f) \leq v \} B_n^{j,f}]$ and $E^A_P[\sup_{f \in \mathcal{F}} \Ind\{\bar{V}_n(f) \leq v \} C_n^{j,f}]$ yields 
\begin{align}
E^A_P \left[\sup_{f\in\mathcal{F}} \Ind \{ \bar{V}_n(f) \leq v \} (P-P_n) \ell_{1:n}(f) \right] \leq & \sqrt{v \epsilon_J} +\sum_{j=0}^{J-1} \frac{v \epsilon^j}{a_j}\\
& + 4 \sqrt{\frac{v}{n} \log\left(1 + \frac{N_0}{P[A]}\right)} + \frac{8}{3 \delta n} \log \left(1 + \frac{N_0}{P[A]} \right)  \\
&+ \sum_{j=1}^J 4 \sqrt{\frac{2 \epsilon_{j-1} v}{n} \log \left( 1 + \frac{\bar{N}_j}{P[A]} \right) } \\
&+ \sum_{j=1}^J \frac{8}{3} \frac{a_{j-1}}{n} \log \left(1 + \frac{\bar{N}_j}{P[A]} \right).
\end{align}
Set $$a_j = \frac{3}{8} \sqrt{\frac{n v \epsilon_j}{\log (1 + \bar{N}_{j+1} / P[A]) } }.$$
Replacing $a_j$ in the previous display yields
\begin{align}
E^A_P \left[\sup_{f\in\mathcal{F}} \Ind \{ \bar{V}_n(f) \leq v \} (P-P_n) \ell_{1:n}(f) \right] \leq & \sqrt{v \epsilon_J} + \frac{8}{3 \delta n} \log \left( 1 + \frac{N_0}{P[A]} \right) \\
& + \sum_{j=0}^J (8 + 2 \sqrt{2}) \sqrt{\frac{v \epsilon_j}{n} \log \left( 1 + \frac{\bar{N}_j}{P[A]} \right) }.
\end{align}
Since $(1 + \bar{N}_j / P[A]) \leq (1 + 1/P[A]) \prod_{k=0}^j (1 + N_k)$, we have
\begin{align}
\sum_{j=0}^J \sqrt{\epsilon_j \log \left( 1  + \frac{\bar{N}_j}{P[A]}  \right) } \leq & \sum_{j=0}^J \sqrt{\epsilon_j} \sqrt{ \log \left(1 + \frac{1}{P[A]} \right) + \sum_{k=0}^j \log (1 + N_k) } \\
\leq & \left( \sum_{j=0}^J \sqrt{\epsilon_j} \right) \sqrt{\log \left( 1 + \frac{1}{P[A]} \right)} + \sum_{j=0}^J \sqrt{\epsilon_j} \sum_{k=0}^j \sqrt{\log (1 + N_k)}
\end{align}
We first look at the second term. We have that
\begin{align}
\sum_{j=0}^J \sqrt{\epsilon_j} \sum_{k=0}^j \sqrt{\log (1 + N_k)} =& \sum_{k=0}^J \sqrt{\log (1 + N_k)} \sum_{j=k}^J (\sqrt{2})^{-j} \\
\leq & \frac{\sqrt{2}}{\sqrt{2} - 1} \sum_{k=0}^J (\sqrt{2})^{-k} \sqrt{\log (1 + N_k)} \\
=& \left(\frac{\sqrt{2}}{\sqrt{2} - 1}\right)^2 \sum_{k=0}^J (\epsilon_k - \epsilon_{k+1} ) \sqrt{\log (1 + N_k)}.
\end{align}
Letting $u_k = \sqrt{\epsilon_k}$, we have that $N_k = N_{[\,]}(u^2_k, \mathcal{F}, L_\infty(P))$ and thus
\begin{align}
\sum_{j=0}^J \sqrt{\epsilon_j} \sum_{k=0}^j \sqrt{\log (1 + N_k)} \leq & \int_{u_{J+1}}^{u_1} \sqrt{\log(1+N_{[\,]}(u^2, \mathcal{F}, L_\infty(P)))} du. 
\end{align}
Therefore, observing that $\sum_{j=0}^J \sqrt{\epsilon_j} \leq \sqrt{2} / (\sqrt{2} - 1)$, and gathering the previous bounds yields that
\begin{align}
&E^A_P \left[\sup_{f\in\mathcal{F}} \Ind \{ \bar{V}_n(f) \leq v \} (P-P_n) \ell_{1:n}(f) \right]\\
 \leq & \sqrt{v \epsilon_J} + (8 + 2 \sqrt{2}) \left( \frac{\sqrt{2}}{\sqrt{2}-1} \right)^2 \sqrt{\frac{v}{n}} \int_{\sqrt{\epsilon_J / 2}}^1 \sqrt{\log (1 + N_{[\,]}(u^2, \mathcal{F}, L_\infty(P))} du \\
 &+ \frac{8}{3} \frac{1}{\delta n} \log\left(1 + N_{[\,]}(1, \mathcal{F}, L_\infty(P))\right) \\
 &+ \frac{\sqrt{2}}{\sqrt{2} - 1} (8 + 2 \sqrt{2}) \sqrt{\frac{v}{n}} \sqrt{ \log \left(1  + \frac{1}{P[A]} \right)} + \frac{8}{3 \delta n} \log \left(1 + \frac{1}{P[A]} \right) \\
 \leq & H_n(v, \delta, \epsilon_J) + 37 \sqrt{\frac{v}{n}} \sqrt{ \log \left(1  + \frac{1}{P[A]} \right)} + \frac{3}{ \delta n} \log \left(1 + \frac{1}{P[A]} \right),
\end{align} 
with 
\begin{align}
H_n(v, \delta, \epsilon) \doteq \sqrt{v \epsilon} + 127 \sqrt{\frac{v}{n}} \int_{\sqrt{\epsilon / 2}}^1 \sqrt{\log (1 + N_{[\,]}(u^2, \mathcal{F}, L_\infty(P)))} du +  \frac{3}{\delta n} \log \left(1 + N_{[\,]}(1, \mathcal{F}, L_\infty(P))\right).
\end{align}
\end{proof}

\section{Regret analysis of the policy evaluation algorithm}
\label{sec:regret:analysis}

\subsection{Definition of $v_\tau$ and constants in the definition of $x_\tau$}

For all $\delta > 0$, $v > 0$, $p > 0$, $\tau \geq 1$, let
\begin{align}
a_\tau(\epsilon, \delta, v, p) \doteq \sqrt{v} \left\lbrace \frac{c_1(c,p)}{\tau^{\frac{1}{2} \wedge \frac{1}{2p}}} + \frac{c_2}{\sqrt{\tau}} \sqrt{\log \left( \frac{\tau (\tau + 1)}{\epsilon} \right)} + \frac{1}{\delta \tau} \left(c_3 + c_4 \log \left( \frac{\tau (\tau + 1)}{\epsilon} \right) \right) \right\rbrace,
\end{align}
with 
\begin{align}
c_1(c,p) \doteq 
\begin{cases} 
\frac{127 \sqrt{c}}{1-p} &\text{ if } p \in (0,1) \\
1 + \frac{127 \sqrt{c} 2^{\frac{p-1}{2}}}{p-1} & \text{ if } p > 1,
\end{cases}
\end{align}
$c_2 = 37$, $c_3 = 3 \log 2$, and $c_4 = 3$.
For all $ \delta > 0$, $v > 0$, $\tau \geq 1$, let 
\begin{align}
b_\tau(\epsilon, \delta, v) \doteq c_5 \sqrt{\frac{v}{\tau} \log \left( \frac{\tau ( \tau + 1) }{\epsilon} \right) } + \frac{c_6}{\delta \tau} \log \left( \frac{\tau ( \tau + 1) }{\epsilon} \right),
\end{align}
with $c_5 = c_6 = 2$.
For all $\delta > 0$, $v > 0$, $\tau \geq  1$, $p > 0$, let
\begin{align}
x_\tau(\epsilon, \delta, v, p) \doteq 2 (a_\tau(\epsilon, \delta, v, p) + b_\tau(\epsilon, \delta, v) ).
\end{align}
For all $\delta > 0$, $\tau \geq 1$, let
\begin{align}
v_\tau(\epsilon, \delta) \doteq 2K + \delta^{-1} \left\lbrace \frac{c'_1(c,p)}{\tau^{\frac{1}{2} \wedge \frac{1}{p}}} + \frac{32}{\sqrt{\tau}} \sqrt{\log \left( \frac{\tau ( \tau + 1) }{\epsilon} \right)} + \frac{16 \log 2}{\tau} + \frac{16}{\tau} \log \left( \frac{\tau ( \tau + 1) }{\epsilon} \right) \right\rbrace,
\end{align}
with 
\begin{align}
c'_1(c,p) \doteq 
\begin{cases}
\frac{64 \sqrt{c}}{1-p/2} & \text{ if } p \in (0,2), \\
1 + \frac{64 \times 2^{p/2-1} \sqrt{c}}{p/2-1} & \text{ if } p > 1.
\end{cases}
\end{align}
The quantity $v_\tau$ from the main text is defined as $v_\tau \doteq v_\tau(\epsilon, \delta_\tau)$.

We can now give the explicit definitions of the sequences $(\delta_t)$ and $(x_t)$. For all $\tau \geq 1$, let 
\begin{align}
\delta_\tau \doteq  \tau^{-\left(\frac{1}{2} \wedge \frac{1}{2p}\right)} \qquad
\text{and } \qquad x_\tau \doteq  x_\tau(\epsilon, \delta_\tau, v_\tau(\epsilon, \delta_\tau), p).
\end{align}

The constant $c_7$ in the main text is defined as $c_7 \doteq c_4  + c_6$.

\subsection{Proofs}

\begin{lemma}[Bound in the max IS ratio in terms of max empirical IS ratio]\label{lemma:PE_control_variance}.
Consider a class of policies $\mathcal{F}$ as in the current section. Suppose that $g: \mathcal{A} \times \mathcal{W} \rightarrow [0,1]$ is such that $g$ is uniformly lower bounded by some $\delta > 0$, that is, for all $a, w \in \mathcal{A} \times \mathcal{W}, g(a,w) \geq \delta$.

Suppose that assumption $\textbf{A1}$ holds. Then, for all $\epsilon > 0$,
\begin{align}
P \left[ \sup_{f \in \mathcal{F}} (P - P_n) \left\lbrace \sum_{a\in[K]} \frac{f(a|W)}{g(a|W)} \right\rbrace \geq v_n(\epsilon, \delta) - 2K \right] \leq 2 \frac{\epsilon}{n(n+1)}.
\end{align}
\end{lemma}

The proof of lemma \ref{lemma:PE_control_variance} relies on the following result, which is a slighlty modified version of corollary 6.9 in \cite{massart2007}. The only differences are that
\begin{itemize}
\item we state it with lower bound of the entropy integral $\alpha/2>0$, instead of $0$, which makes appear an approximation error term $\alpha$,
\item we state it for i.i.d. random variables instead of independent random variables, we set to 1 the value of $\epsilon$ in the original statement of the theorem.
\end{itemize} 

\begin{proposition}\label{prop:max_ineq_empirical_process_Massart}
Let $\mathcal{F}$ be a class of functions $f:\mathcal{X}\rightarrow \mathbb{R}$. Let $X_1,\ldots,X_n$ be i.i.d. random variables with domain $\mathcal{X}$ and common marginal distribution $P$. Suppose that there exists $\sigma$ and $b$ such that, for all $f \in \mathcal{F}$, for any $k \geq 2$,
\begin{align}
E_P[|f(X)|^k] \leq \frac{k!}{2} \sigma^2 b^{k-2}.
\end{align}
Assume that for all $\epsilon > 0$, there exists a set of brackets $\mathcal{B}(\epsilon, b)$ covering $\mathcal{F}$ such that, for all bracket $[l,u]$ in $\mathcal{B}(\epsilon, \delta)$, 
\begin{align}
E[((u-l)(X))^k] \leq \frac{k!}{2} \epsilon^2 b^{k-2}.
\end{align}
We call such a $\mathcal{B}(\epsilon, \delta)$ an $(\epsilon, b)$ bracketing of $\mathcal{F}$, and we denote $\mathcal{N}_{[\,]}(\epsilon, b\, \mathcal{F})$ the minimal cardinality of such an $\mathcal{B}(\epsilon, b)$. 

Then, for all $\alpha \in (0,\sigma)$, and for all $x > 0$,
\begin{align}
P \left[\sup_{\in \mathcal{F}} (P-P_n) f \geq H_n(\alpha, \sigma, b) +  10 \sigma\sqrt{\frac{x}{n}} + 2 b x \right] \leq e^{-x},
\end{align}
where 
\begin{align}
H_n(\alpha, \sigma, b) \doteq \alpha + \frac{27}{\sqrt{n}} \int_{\alpha / 2}^\sigma \sqrt{\log \mathcal{N}_{[\,]}(\epsilon, b, \mathcal{F})} d\epsilon + \frac{2(\sigma + b)}{n} \log \mathcal{N}_{[\,]}(\sigma, b, \mathcal{F}).
\end{align}
\end{proposition}

\begin{proof}[Proof of proposition \ref{prop:max_ineq_empirical_process_Massart}]
It suffices to choose $J$ in the proof of corollary 6.9 in \cite{massart2007} such that $\alpha / 2 \leq \epsilon_J < \alpha$, and not let it go to $\infty$ at the end of the proof.
\end{proof}

\begin{proof}[Proof of lemma \ref{lemma:PE_control_variance}]
Let 
\begin{align}
\mathcal{H} \doteq \left\lbrace h: w \mapsto \sum_{a\in[K]} \frac{f(a,w)}{g(a,w)} : f \in \mathcal{F} \right\rbrace.
\end{align}
Observe that, for all $h \in \mathcal{H}$, $h(W)| \leq \delta^{-1}$, as $g \geq \delta$, and thus $E_P [ h^2(W) ] \leq \delta^{-2}$. Observe that an $\epsilon$-bracketing of $\mathcal{F}$ in $L_2(P)$ induces a $(\sqrt{K} \epsilon \delta^{-1},b)$ bracketing of $\mathcal{H}$ in the sense of proposition \ref{prop:max_ineq_empirical_process_Massart}. Therefore, from proposition \ref{prop:max_ineq_empirical_process_Massart}
\begin{align}
P \left[ \sup_{f \in \mathcal{F}} (P-P_n) f \geq v_n(\epsilon, \delta) - 2 K \right] \leq \frac{\epsilon}{n (n+1)}.
\end{align}
\end{proof}

The following lemma shows that, with high probability, the policy elimination algorithm doesn't eliminate the optimal policy.

\begin{lemma}\label{lemma:PE_doesnt_elim_f_star_whp}
Suppose that $\textbf{A1}$ holds. Suppose $(x_t(\epsilon))$ is as specified in subsection \ref{subsection:PE_regret_analysis}. Then, for all $t \geq 1$,
\begin{align}
P[f^* \in \mathcal{F}] \geq 1 - 3 \epsilon.
\end{align}
\end{lemma}

\begin{proof}
Denote $\hat{f}_\tau \doteq \arg \min_{f \in \mathcal{F}_\tau} \hat{R}_\tau(f)$. 
We have that
\begin{align}
\hat{R}_\tau(f^*) - \hat{R}_\tau(\hat{f}_\tau) \leq & R(f^*) - R(\hat{f}_\tau) \\
& + \hat{R}_\tau(f^*) - R(f^*) \\
& + R(\hat{f}_\tau) - \hat{R}_\tau(\hat{f}_\tau) \\
\leq & \hat{R}_\tau(f^*) - R(f^*) \\
&+ \sup_{f \in \mathcal{F}_\tau} R(f) - \hat{R}_\tau(f).
\end{align}

Define the event
\begin{align}
\mathcal{E}_{1,t} \doteq \left\lbrace \forall \tau \in [t] : \sup_{f \in \mathcal{F}_\tau} V(g_\tau, f) \leq v_\tau(\epsilon, \delta_\tau) \right\rbrace,
\end{align}
where $v_\tau(\epsilon, \delta_\tau)$ is defined in subsection \ref{subsection:PE_regret_analysis}. 
From lemma \ref{lemma:PE_control_variance},
\begin{align}
P[\mathcal{E}_{1,t}] \geq 1 - 2 \epsilon.
\end{align}
For all $\tau \in [t]$, define the event
\begin{align}
\mathcal{E}_{2,t} \doteq \left\lbrace \max_{\tau \in [t]} \sup_{f \in \mathcal{F}_\tau} R(f) - \hat{R}_\tau(f) \leq a_\tau(\epsilon, \delta_\tau, v_\tau(\epsilon, \delta_\tau), p) \right\rbrace,
\end{align}
where $a_\tau$ is defined in subsection \ref{subsection:PE_regret_analysis}.
From theorem \ref{thm:max_ineq_under_param_dependent_IS_bound}, 
\begin{align}
P[\mathcal{E}_{2,t}^c, \mathcal{E}_{1,t}] \leq \epsilon.
\end{align}
We now turn to controlling $\hat{R}_\tau(f^*) - R(f^*)$. So as to be able to obtain a high probability bound scaling as  $\sqrt{v_\tau(\epsilon, \delta_\tau) / \tau}$, we need $f^*$ to be in $\mathcal{F}_\tau$. As we are about to show, if the desired bound holds, that $\mathcal{E}_{1,t} \cap \mathcal{E}_{2,t}$ holds, and that $f^* \in \mathcal{F}_\tau$, them we will have that $f^* \in \mathcal{F}_{\tau + 1}$. This motivates a reasoning by induction.

Let, for all $\tau \in [t]$,
\begin{align}
\mathcal{E}_{3,\tau} \doteq \left\lbrace \hat{R}_\tau(f^*) - R(f^*) \leq b_\tau(\epsilon, \delta_\tau, v(\epsilon, \delta_\tau)) \right\rbrace,
\end{align}
where $b_\tau$ is defined in subsection \ref{subsection:PE_regret_analysis}. We are going to show by induction that for all $\tau \in [t]$,
\begin{align}
P\left[ \mathcal{E}_{3,t}^c, \mathcal{E}_{1,t}, \mathcal{E}_{2,t}\right] \leq \sum_{s=1}^\tau \frac{\epsilon}{s(s+1)}.
\end{align}
By convention, we let $\mathcal{E}_{3,0} \doteq \{f^* \in \mathcal{F} \}$. and $\sum_{s=1}^0 1/(s(s+1)) = 0$. The induction claim thus trivially holds at $\tau = 0$. Consider $\tau \in [t]$. Suppose that
\begin{align}
P[\mathcal{E}_{3, \tau-1}^c, \mathcal{E}_{1,t}, \mathcal{E}_{2,t} ] \leq \sum_{s=1}^{\tau-1} \frac{\epsilon}{s(s+1)}.
\end{align}
Observe that $\mathcal{E}_{\tau -1} \cap \mathcal{E}_{1,t} \cap \mathcal{E}_{2,t}$ implies $f^* \in \mathcal{F}_\tau$ as we then have 
\begin{align}
\hat{R}_\tau(f^*) - \hat{R}_\tau(\hat{f}_\tau) \leq & a_{\tau-1}(\epsilon, \delta_{\tau-1}, v_{\tau-1}(\epsilon, \delta_{\tau-1}), p) + b_{\tau-1}(\epsilon, \delta_{\tau-1}, v_{\tau  -1}(\epsilon, \delta_{\tau-1})) \\
<& x_{\tau-1}(\epsilon, \delta_{\tau-1}, v_{\tau  -1}(\epsilon, \delta_{\tau-1})).
\end{align}
Using this fact, distinguishing the cases $\mathcal{E}_{3,\tau-1}$ and $\mathcal{E}_{3,\tau-1}^c$, and using the induction hypothesis yields
\begin{align}
P[\mathcal{E}_{3, \tau}^c, \mathcal{E}_{1,t}, \mathcal{E}_{2,t}] \leq & P [ \mathcal{E}_{3,\tau}^c, \mathcal{E}_{3,\tau-1}, \mathcal{E}_{1,t}, \mathcal{E}_{2,t} ] + P[\mathcal{E}_{3, \tau-1}^c, \mathcal{E}_{1,t}, \mathcal{E}_{2,t}] \\
\leq & P[ \mathcal{E}_{3,\tau}^c, f^* \in \mathcal{F}, \mathcal{E}_{2,t} ] + \sum_{s=1}^{\tau-1} \frac{\epsilon}{s(s+1)}.
\end{align}
Observe that under $\{f^* \in \mathcal{F}_{\tau} \} \cap \mathcal{E}_{2,t}$, we have that $V(g_\tau, f^*) \leq v_\tau(\epsilon, \delta_{\tau})$ and thus
\begin{align}
E[ ( \ell_{\tau}(f^*)(O_\tau))^2 | F_{\tau-1} ] \leq K v_\tau(\epsilon, \delta_\tau).
\end{align}
Besides, $| \ell_{\tau}(f^*)(O_\tau) - E[ \ell_{\tau}(f^*)(O_\tau)| F_{\tau-1}] \leq \delta_\tau^{-1}$.
Therefore, from Bernstein's inequality for martingales
\begin{align}
P[ \mathcal{E}_{3,\tau}^c, f^* \in \mathcal{F}, \mathcal{E}_{2,t} ] \leq \frac{\epsilon}{\tau (\tau  +1)}.
\end{align}
Therefore,
\begin{align}
P[\mathcal{E}^c_{3,\tau}, \mathcal{E}_{1,t}, \mathcal{E}_{2,t} ] \leq \sum_{s=1}^\tau \frac{\epsilon}{s(s+1)}.
\end{align}
We have thus shown that, for all $\tau \in [t]$,
\begin{align}
P[\mathcal{E}^c_{3,\tau}, \mathcal{E}_{1,t}, \mathcal{E}_{2,t} ] \leq \sum_{s=1}^\tau \frac{\epsilon}{s(s+1)}.
\end{align}
Therefore,
\begin{align}
P[\mathcal{E}_{3,t}, \mathcal{E}_{1,t}, \mathcal{E}_{2,t}] =& P[ \mathcal{E}_{1,t}, \mathcal{E}_{2,t}] - P[\mathcal{E}^c_{3,t}, \mathcal{E}_{1,t}, \mathcal{E}_{2,t}] \\
=& P[\mathcal{E}_{1,t}] - P[\mathcal{E}_{1,t}, \mathcal{E}^c_{2,t}] - P[\mathcal{E}^c_{3,t}, \mathcal{E}_{1,t}, \mathcal{E}_{2,t}] \\
=& 1 - P[\mathcal{E}^c_{1,t}] - P[\mathcal{E}_{1,t}, \mathcal{E}^c_{2,t}] - P[\mathcal{E}_{1,t}, \mathcal{E}_{2,t}, \mathcal{E}_{3,t}^c] \\
\geq & 1 - 4\epsilon.
\end{align}
\end{proof}

The following lemma gives a bound on $\sup_{f \in \mathcal{F}_\tau} R(f) - R(f^*)$ which holds uniformly in time with high probability.

\begin{lemma}\label{lemma:PE_sup_suboptimality_over_F_tau}
Consider algorithm \ref{alg:PE}. Make assumption \textbf{A1}. Then, with probability $1-4\epsilon$, we have that, for all $\tau \in [t]$, 
\begin{align}
\sup_{f \in \mathcal{F}_\tau} R(f) - R(f^*) \leq 2 x_\tau.
\end{align}
\end{lemma}

\begin{proof}
Observe that, for all $f \in \mathcal{F}$, 
\begin{align}
R(f) - R(f^*) =& \hat{R}_\tau(f) - \hat{R}_\tau(f^*)  \\
&+ R(f) - \hat{R}_\tau(f) \\
&-R(f^*) - \hat{R}_\tau(f^*)) \\
\leq & \hat{R}_\tau(f) - \hat{R}_\tau(\hat{f}_\tau) \\
&+ \sup_{f \in \mathcal{F}_\tau} ( R(f) - \hat{R}_t(f) ) \\
& - (R(f^*) - \hat{R}_t(f^*)) \\
\leq & x_\tau \\
&+ \sup_{f \in \mathcal{F}_\tau} ( R(f) - \hat{R}_t(f) ) \\
&- (R(f^*) - \hat{R}_t(f^*)).
\end{align}
Define the events
\begin{align}
\mathcal{E}_{1,t} \doteq & \left\lbrace \forall \tau \in [t], \sup_{f \in \mathcal{F}_\tau } V(g_\tau, f) \leq v_\tau(\epsilon, \delta_\tau) \right\rbrace,\\
\mathcal{E}_{2,t} \doteq & \left\lbrace f^* \in \mathcal{F}_t \right\rbrace.
\end{align}
From lemma \ref{lemma:PE_doesnt_elim_f_star_whp},
\begin{align}
P[\mathcal{E}_{1,t}] \geq 1  - 4 \epsilon.
\end{align}
Under $\mathcal{E}_{1,t}$, we have that, for all $f \in \mathcal{F}_{\tau}$,
\begin{align}
E\left[ ( \ell_{\tau}(f)(O_\tau))^2 | F_{\tau-1} \right] \leq K v_\tau(\epsilon, \delta_\tau).
\end{align}
Therefore, using also that $| \ell_{\tau}(f)(O_\tau)| \leq \delta_{\tau}^{-1}$, theorem \ref{thm:max_ineq_under_param_dependent_IS_bound} gives us that, for all $\tau \in [t]$,
\begin{align}
P \left[ \sup_{f \in \mathcal{F}_\tau} R(f) - \hat{R}_\tau(f) \geq a_\tau(\epsilon, v_\tau(\epsilon, \delta_\tau), \delta_\tau, p), \mathcal{E}_{1,t} \right] \leq \frac{\epsilon}{\tau(\tau + 1)},
\end{align}
which, by a union bound gives us that
\begin{align}
P \left[ \mathcal{E}_{3,t}^c, \mathcal{E}_{1,t} \right] \leq \epsilon,
\end{align}
with 
\begin{align}
\mathcal{E}_{3,t} \doteq \left\lbrace \forall \tau \in [t], \sup_{f \in \mathcal{F}_\tau} R(f) - \hat{R}_\tau(f) \leq a_\tau(\epsilon, v_\tau(\epsilon, \delta_\tau), \delta_\tau, p) \right\rbrace.
\end{align}
We now consider the term $\hat{R}_\tau(f) - R(f^*)$. We have that
\begin{align}
\hat{R}_\tau(f^*) - R(f^*) = \frac{1}{t} \sum_{\tau=1}^t  \ell_{\tau}(f^*)(O_\tau) - E[ \ell_{\tau}(f^*)(O_\tau) | F_{\tau-1}]
\end{align}
Under $\mathcal{E}_{1,t} \cap \mathcal{E}_{2,t}$, each term in the sum satisfies
\begin{align}
E_P\left[ ( \ell_{\tau}(f^*)(O_\tau))^2 | F_{\tau-1} \right] \leq K v_\tau(\epsilon, \delta_\tau)
\end{align}
and 
\begin{align}
| \ell_{\tau}(f^*)(O_\tau) - E_P \left[  \ell_{\tau}(f^*)(O_\tau) | F_{\tau-1} \right] \leq \delta_\tau^{-1}.
\end{align}
Therefore, from Bernstein's inequality and a union bound, letting 
\begin{align}
\mathcal{E}_{4,t} \doteq \left\lbrace \forall \tau \in [t],  \hat{R}_\tau(f^*) - R(f^*) \leq b_\tau(\delta, v_\tau(\epsilon, \delta_\tau), \delta_\tau) \right\rbrace,
\end{align}
we have that
\begin{align}
P[\mathcal{E}_{4,t}^c, \mathcal{E}_{1,t}, \mathcal{E}_{2,t}] \leq \epsilon.
\end{align}
Observe that under $\mathcal{E}_{3,t} \cap \mathcal{E}_{4,t}$ it holds that 
\begin{align}
\forall \tau \in [t] \sup_{f \in \mathcal{F}_{\tau} } R(f) - R(f^*) \leq x_\tau.
\end{align}
Therefore, to conclude the proof, it suffices to bound $P[ \mathcal{E}_{3,t}, \mathcal{E}_{4,t}]$. We have that
\begin{align}
P[ (\mathcal{E}_{3,t} \cap \mathcal{E}_{4,t})^c] \leq & P [ \mathcal{E}_{1,t}, \mathcal{E}_{2,t}, (\mathcal{E}_{3,t} \cap \mathcal{E}_{4,t})^c] + P[\mathcal{E}_{1,t}^c] + P[\mathcal{E}_{2,t}^c] \\
\leq & P[\mathcal{E}_{1,t}, \mathcal{E}_{2,t}, \mathcal{E}_{3,t}^c] \\
&+ P[\mathcal{E}_{1,t}, \mathcal{E}_{2,t}, \mathcal{E}_{4,t}^c] + P[\mathcal{E}_{1,t}^c] + P[\mathcal{E}_{2,t}^c] \\
\leq & 6 \epsilon,
\end{align}
which yields the wished claim.
\end{proof}

We can now prove theorem \ref{thm:regret_PE}.
\begin{proof}[Proof of theorem \ref{thm:regret_PE}]
Observe that 
\begin{align}
\sum_{\tau=1}^t \mathcal{V}(f^*) - Y_\tau =& \sum_{\tau=1}^t (1-Y_\tau) - E_P[(1-Y_\tau)| F_{\tau-1}] \\
&+ \sum_{\tau=1}^t E_P[(1-Y_\tau)| F_{\tau-1}] - R(f^*).
\end{align}
Since $(1-Y_\tau) \in [0,1]$, from Azuma-Hoeffding, we have that, with probability at least $1-\epsilon$,
\begin{align}
\sum_{\tau=1}^t (1-Y_\tau) - E_P[(1-Y_\tau)| F_{\tau-1}] \leq \sqrt{t \log \left( \frac{1}{\epsilon} \right)}.
\end{align}
Observe that
\begin{align}
E_P \left[(1-Y_\tau) | F_{\tau-1}\right] = R(g_\tau) = \delta_\tau R(g_{ref}) + (1-  \delta_\tau) R(\tilde{g}_\tau),
\end{align}
where $\tilde{g}_\tau \in \mathcal{F}_\tau$. Therefore,
\begin{align}
E_P \left[(1-Y_\tau) | F_{\tau-1}\right] - R(f^*) \leq & \delta_\tau (R(g_{ref}) - R(f^*)) + (1-\delta_\tau) (R(\tilde{g}_\tau)  -R(f^*)) \\
\leq & \delta_\tau + (R(\tilde{g}_\tau) - R(f^*))
\end{align}
From lemma \ref{lemma:PE_sup_suboptimality_over_F_tau}, with probability $1 - 6 \epsilon$, for all $\tau \in [t]$, 
\begin{align}
R(\tilde{g}_\tau) - R(f^*) \leq x_\tau.
\end{align}
Therefore, with probability at least $1-7\epsilon$, we have the wished bound.
\end{proof}

\section{Regret analysis of the $\varepsilon$-greedy algorithm}
\label{sec:regret:analysis:greedy}

\subsection{Regret decomposition}

Using in particular the linearity of $\pi \mapsto R(\pi)$ and the definition of $g_t$, we have that

\begin{align}
&Y_t - R(\pi^*) \\
=& Y_t - E[Y_t | F_{t-1}] + E[Y_t | F_{t-1}] - R(\pi^*) \\
=& Y_t - E[Y_t | F_{t-1}] + R(g_t) - R(\pi^*) \\
=& \underbrace{Y_t - E[Y_t | F_{t-1}]}_{\text{reward noise}} + \underbrace{\delta_t (R(g_{ref}) - R(\pi^*))}_{\text{exploration cost}} \\
&+ (1- \delta_t) \underbrace{( R(\hat{\pi}_{t-1}) - R(\pi^*))}_{\text{exploitation cost}}. \label{eq:regret_decomp_eps_greedy}
\end{align}

\subsection{Proof of deviations inequalities}

\begin{proof}[Proof of theorem \ref{thm:phi-risk_dev_bound_eps_greedy}]
Observe that 
\begin{align}
R^\phi(\hat{f}_t) - R^\phi(f^*_\mathcal{F}) =& \frac{1}{t} \sum_{\tau=1}^t E\left[ \ell_{\tau}(f)(O_\tau) -  \ell_{\tau}(f^*_\mathcal{F})(O_\tau) | F_{\tau-1} \right]\big|_{f=\hat{f}_t} \\
=& \frac{1}{t} \sum_{\tau=1}^t \ell_{\tau}^{\phi}(\hat{f}_t)(O_\tau) - \ell^{\phi}_\tau(f^*_\mathcal{F})(O_\tau) \label{eq:phi_risk_lemma_diff_emp_risks} \\
&+ M_t(\hat{f}_t)
\end{align}
with $M_t(f)$ as defined in \eqref{eq:def_M_t_of_f_eps_basic_IS} and where we take $f_0 = f^*_\mathcal{F}$ in the definition of $M_t$. Since $\hat{f}_t$ is the empirical $\phi$-risk minimizer, line \ref{eq:phi_risk_lemma_diff_emp_risks} is non-positive, and thus
\begin{align}
R^\phi(\hat{f}_t) - R^\phi(f^*_\mathcal{F}) \leq M_t(\hat{f}_t). \label{eq:excess_phi_risk_bounded_by_mart_processs}
\end{align}
Observe that, for all $f \in \mathcal{F}$,
\begin{align}
|\ell^{\phi}_\tau(f) - \ell^{\phi}_\tau(f^*_\mathcal{F}) | \leq \frac{B}{\delta}
\end{align}
and 
\begin{align}
E_P\left[ \left( (\ell_{\tau}^{\phi}(f)(O_\tau) - \ell^{\phi}_\tau(f^*_\mathcal{F})(O_\tau) \right)^2 \big| F_{\tau-1} \right] = & E_P \left[ \frac{\left(\phi(f(A_\tau, W_\tau) - \phi(f^*_\mathcal{F}(A_\tau, W_\tau)\right)^2}{g_\tau(A_\tau,W_\tau)^2} (1-Y_\tau)^2 \big| F_{\tau-1} \right] \\
\leq & E \left[\sum_{a\in[K]} \frac{\left(\phi(f(a, W_\tau) - \phi(f^*_\mathcal{F}(a, W_\tau)\right)^2}{g_\tau(a,W_\tau)} \big| F_{\tau-1} \right]\\
\leq & \frac{KB^2}{\delta}.
\end{align}
Therefore, using \eqref{eq:excess_phi_risk_bounded_by_mart_processs} and theorem \ref{thm:max_ineq_IS_weighted_mart_process}, we have that
\begin{align}
P \left[ R^\phi(\hat{f}_t) - R^\phi(f^*_\mathcal{F}) \geq H_t \left( \alpha, \delta, B \sqrt{\frac{K}{\delta}}, B \right) + 160 B \sqrt{\frac{K x}{\delta t}} + 3 \frac{B x}{\delta t} \right] \leq 2 e^{-x},
\end{align}
with
\begin{align}
H_t(\alpha, \delta, v, B) = \alpha + 160 \sqrt{\frac{v}{t}} \int_{\alpha/2}^B \sqrt{\log (1 + N_{[\,]}(\epsilon, \mathcal{F}, L_2(P)))} d\epsilon + 3 \frac{B}{\delta t} \log 2.
\end{align}
\end{proof}

\begin{proof}[Proof of theorem 4]
For any $p \in (0,2) \cup (2, \infty)$,
\begin{align}
\int_{\alpha/2}^B \sqrt{\log (1 + N_{[\,]}(\epsilon, \mathcal{F}, L_2(P)} \leq \frac{\sqrt{c_0}}{1-p/2} \left( B^{1-p/2} - \left(\frac{\alpha}{2}\right)^{1-p/2} \right).
\end{align}
We set
\begin{align}
\alpha	= \begin{cases} 0 &\text{ for } p \in (0,2)\\
B^{2/p} \left( \frac{K}{\delta \tau} \right)^{\frac{1}{p}} & \text{ for } p > 2.
\end{cases}
\end{align}
Then, we have
\begin{align}
H_\tau(\alpha, \delta, v, B) \leq \begin{cases} 
B \sqrt{\frac{K}{\delta \tau}} \frac{\sqrt{c_0}}{1-p/2} B^{1-p/2} + \frac{3 B \log 2}{\delta \tau} &\text{ for } p \in (0,2),\\
B^{2/p} \left( \frac{K}{\delta \tau} \right)^{1/p} \left(1 + \frac{\sqrt{c_0} 2^{1/p} 2^{p/2-1}}{1-p/2} \right) + \frac{3 B}{\delta \tau} \log 2 &\text{ for } p > 2.
\end{cases}
\end{align}
Therefore, for 
\begin{align}
x_\tau(\epsilon, K, \delta, B, p) \doteq & \begin{cases} 
B \sqrt{\frac{K}{\delta \tau}} \left( \frac{\sqrt{c_0}}{1-p/2} B^{1-p/2} + 160 \sqrt{\log (2 / \epsilon)}\right) + \frac{3 B}{\delta \tau} \log (4 / \epsilon) &\text{ if } p \in (0,2)\\
B^{2/p} \left( \frac{K}{\delta \tau} \right)^{1/p} \left( 1 + \frac{\sqrt{c_0} 2^{p/2-1} }{1-p/2} \right) + B \sqrt{\frac{K}{\delta \tau} \log( 2 / \epsilon)} + \frac{3 B}{\delta \tau}\log (4/\epsilon) &\text{ if } p > 2.
\end{cases}
\end{align}
Theorem \ref{thm:phi-risk_dev_bound_eps_greedy} gives that
\begin{align}
P \left[ R^\phi(\hat{f}_t) - R^\phi(f^*_\mathcal{F}) \geq x_\tau(\epsilon, K, \delta, \delta, B, p) \right] \leq \epsilon.
\end{align}
Observe that 
\begin{align}
\sum_{\tau=1}^t \mathcal{V}(\pi^*_\Pi) - Y_\tau =& \sum_{\tau=1}^t \mathcal{V}(\pi^*_\Pi) - E_P \left[ Y_\tau | F_{\tau-1} \right]+ \sum_{\tau=1}^t E_P[Y_\tau | F_{\tau-1} ] \\
\leq & \sum_{\tau=1}^t \delta_\tau (R(g_{ref}) - R(\pi^*_\Pi)) + (1 - \delta_\tau) R(\tilde{\pi}(\hat{f}_{\tau-1}) - R(\pi^*_\Pi) \\
&+ \sum_{\tau=1}^t E_P[Y_\tau|F_{\tau-1}] - Y_\tau \\
\leq & \sum_{\tau=1}^t \delta_\tau \\
&+ \sum_{\tau=1}^t \left(R^\phi(\hat{f}_{\tau-1}) - R^\phi(f^*_\mathcal{F}) \right)\\
&+ \sum_{\tau=1}^t E_P [Y_\tau | F_{\tau-1} ].
\end{align}
By a union bound, with probability at least $1-\epsilon/2$,
\begin{align}
\sum_{\tau=1}^t R^\phi(\hat{f}_{\tau-1}) - R^\phi(f^*_\mathcal{F}) \leq \sum_{\tau=1}^t x_\tau \left( \frac{\epsilon}{\tau (\tau+1)}, K, \delta, B, p \right).
\end{align}
By Azuma-Hoeffding, with probability at least $1-\epsilon/2$,
\begin{align} 
\sum_{\tau=1}^t E_P[Y_\tau | F_{\tau-1} ] - Y_\tau \leq \sqrt{2 \log (2/\epsilon)}.
\end{align}
Therefore, with probability at least $1 - \epsilon$, 
\begin{align}
\sum_{\tau=1}^t \mathcal{V}(\pi^*_\Pi) - Y_\tau \leq & \sum_{\tau=1}^t \delta_\tau + x_\tau \left( \frac{\epsilon}{2 \tau (\tau + 1)}, K, \delta_\tau, B, p \right) \\
\lesssim & t^{\frac{2}{3} \vee \frac{p}{p+1}} \sqrt{\log (t / \epsilon)}.
\end{align}
\end{proof}

\section{Results on efficient algorithm for policy search in GPE}
\label{sec:efficient:algorithm:for:policy:search:in:GPE}

\subsection{Casting exploration policy search as a convex feasibility problem}\label{subsubection:policy_search_as_convex_feasiblity}

For any $M > 0$, denote $\mathcal{P}_t(M)$ the following feasibility problem. 
\begin{align}
\text{Find } \tilde{g}_t \in \mathcal{F}_t \text{ such that } \frac{1}{t-1} \sum_{\substack{a \in [K] \\ \tau \in [t-1]}}  \frac{f(a,W_\tau)}{\delta_t / K + (1 - \delta_t) \tilde{g}_t(a|W_\tau)} \leq M. 
\end{align}
For all $f \in \mathcal{F}_t$, let 
\begin{align}
w_{t, f} \doteq (f(a,W_\tau) : a \in [K], \tau \in [t] ).
\end{align}
For any given $f \in \mathcal{F}$, observe that
\begin{align}
f \in \mathcal{F}_t \iff & \forall \tau \in [t-1], \hat{R}_\tau(f) \leq \min_{f \in \mathcal{F}_{\tau}} \hat{R}_\tau(f) + \epsilon_\tau \doteq b_\tau \\
\iff & \forall \tau \in [t-1], u_{t-1,\tau}^\top w_{t-1,f} \leq b_\tau \label{eq:linear_constraints_defining_F_t},
\end{align}
where 
\begin{align}
u_{t,\tau} \doteq \left(  \Ind\{s \leq \tau \} \frac{\Ind\{A_s = a \} (1 - Y_s) }{g_\tau(a|W_s)} : a \in [K], s \in [t] \right).
\end{align}
Introduce the set 
\begin{align}
\mathcal{C}_t \doteq \{ w_{f,t} : f \in \mathcal{F}_t \},
\end{align}
which, by \eqref{eq:linear_constraints_defining_F_t}
can be rewritten as 
\begin{align}
\mathcal{C}_t \doteq \{ w_{f,t} : f \in \mathcal{F}_t, \forall \tau \in [t], u_{t,f} w_{f,t} \leq b_\tau \}. \label{eq:rephrased_C_t}
\end{align}
Based on \eqref{eq:linear_constraints_defining_F_t} and \eqref{eq:rephrased_C_t}, we can thus rewrite $\mathcal{P}_t(M)$ as the following two-step problem.
\begin{align}
1.& \text{ Find } w \in \mathcal{C}_t \text { such that } \forall z \in \mathcal{C}_t,   \frac{1}{t-1} \sum_{\substack{a \in [K], \tau \in [t-1]}} \frac{z_{a,\tau}}{\delta_t / K + (1 - \delta_t) w_{a,\tau}} \leq M.\\
2.& \text{ Find } f \in \mathcal{F}_t \text{ such that } w_{f,t} = w.
\end{align}
As $\mathcal{F}$ is convex, that functions in $f$ have range in $[0,1]$, and that for all $z \in \mathbb{R}^{K t}$,
\begin{align}
w \mapsto \frac{1}{t-1} \sum_{\substack{a \in [K] \\ \tau \in [t] } } \frac{z_{a,\tau}}{\delta_t / K + (1 - \delta_t) w_{a,\tau}}
\end{align}
is a convex mapping, the set 
\begin{align}
\mathcal{D}_t(M) \doteq \mathcal{C}_t \cap \left\lbrace w \in \mathbb{R}^{Kt} : \forall z \in \mathcal{C}_t, \frac{1}{t-1}\sum_{\substack{a \in [K] \\ \tau \in [t-1] } } \frac{z_{a,\tau}}{\delta_t / K + (1 - \delta_t) w_{a,\tau}} \leq M \right\rbrace
\end{align}
is a convex set. The following lemma ensures it is not empty.
\begin{lemma}\label{lemma:fundamental_PE_property_finite_version}
Let $\mathcal{C}$ be a compact convex subset of $\mathbb{R}^{K (t-1)}$. Set arbitrary $\delta \in (0,1)$ and $w \in \mathcal{C}$. Then
\begin{align}
\max_{z \in \mathcal{C} } \frac{1}{t-1}\sum_{\substack{a \in [K] \\ \tau \in [t-1] } } \frac{z_{a,\tau}}{\delta / K + (1 - \delta) w_{a,\tau}} \leq \frac{4}{3} K.
\end{align}
\end{lemma}

As we will recall precisely in the next subsection, so as to be able to give gaurantees on the number of iterations needed by the ellipsoid algorithm to find a point in a convex set, we need a lower bound on the volume of the set. As we can make the volume of $\mathcal{D}_t$ arbitrarily small in some cases, similarly to \citep{dudik2011}, we will consider a slightly enlarged version of $\mathcal{D}_t$ whose volume we can explicitly lower bound. The following lemma informs how to construct such an enlarged set. Before stating the lemma, we introduce the following notation:
\begin{align}
h_{t,\delta} \doteq \frac{1}{t} \sum_{\substack{a \in [K] \\ \tau \in [t] } } \frac{z_{a,\tau}}{\delta / K + (1 - \delta) w_{a,\tau}}
\end{align}

\begin{lemma}\label{lemma:Lispchitz_property_constraint_PE} Let $w \in (\mathbb{R}_+)^{Kt}$, $\delta \in (0,1)$, $\Delta \in (0, \delta / 2)$. Then, for all $u \in B_{Kt}(0,1)$, $z \in [0,1]^{Kt}$,
\begin{align}
|h_{\delta, t}(w + \Delta u, z) - h_{\delta, t}(w, z)| \leq \xi_{t, \delta}(\Delta),
\end{align}
with $\xi_{t,\delta}(\Delta) \doteq 2 \Delta \delta^{-2} \sqrt{K / t}$.
\end{lemma}
For all $\Delta > 0$, let 
\begin{align}
C_{t,\Delta} = \left\lbrace w \in \mathbb{R}^{Kt} : d(w, \mathcal{C}_t \leq \Delta \right\rbrace.
\end{align}
From the above lemma, if $w \in \mathcal{D}_t(M)$, every point $w' \in B(w, \Delta)$ satisfies 
\begin{align}
\max_{z \in \mathcal{C}_t} h_{t,\delta}(w',z) \leq M + \xi_{t,\delta}(\Delta).
\end{align}
Therefore, provided $\mathcal{D}_t$ contains at least one point, say $w$, the set 
\begin{align}
\mathcal{D}_{t,\Delta} \doteq \left\lbrace w \in \mathcal{C}{t, \Delta} : \forall z \in \mathcal{C}_t, h_{t, \delta}(w,z) \leq M + \xi_{t,\delta}(\Delta) \right\rbrace
\end{align}
contains $B(w,\Delta)$. Finally, suppose that $w \in \mathcal{D}_{t, \Delta}(M)$. Then, by definition of $\mathcal{D}_{t, \Delta}(M)$, there exists a $w' \in \mathcal{C}_t$ such that $d(w', w') \leq \Delta$, and thus by lemma \ref{lemma:Lispchitz_property_constraint_PE},
\begin{align}
\max_{z \in \mathcal{C}_t} h_{t,\delta}(w,z) \leq M + 2 \xi_{t,\delta}(\Delta).
\end{align}
By lemma \ref{lemma:fundamental_PE_property_finite_version}, we can pick $M = 4 K / 3$ while still ensuring that $\mathcal{D}_t(M)$ is non-empty. Them setting $\Delta$ such that $\xi_{t, \delta_t}(\Delta) = K/3$, that is setting it to $\Delta_t \doteq \delta^2 \sqrt{(t-1)/K}$ ensures that $\mathcal{D}_{t, \Delta_t}$ contains a ball od radius $\Delta_t$ and that $M + 2 \xi_{t,\delta_t}(\Delta_t) \leq 2 K$. Therefore, the exploration policy search problem \eqref{eq:exploration_policy_search_step} is equivalent to the two-step process
\begin{align}
1.& \text{ Find } w \in \mathcal{D}_{t, \Delta_t} \\
2.& \text{ Find } f \in \mathcal{F} \text{ such that } \|w_{f,t} - w \|_2 \leq \Delta_t.
\end{align}

\subsection{Finding an element of $\mathcal{U}$ using the ellipsoid algorithm}
\label{subsubsection:solving_feasibility_with_ellipsoid_alg}

Finding an element of a convex set of non-negligilble volume such as $\mathcal{D}_{t,\Delta_t}(4K/3)$ can be performed in polynomial time with the ellipsoid algorithm. The ellipsoid algorithm requires having access to a separation oracle.

\begin{definition}[Separation oracle] Let $\mathcal{C} \subseteq \mathbb{R}^n$, $n \geq 1$ be a convex set. A separation oracle for $\mathcal{C}$ is a routine that, for any $w \in \mathbb{R}^n$ outputs whether $w \in \mathcal{C}$, and if $w \neq \mathcal{C}$, returns an hyperplane separating $w$ and $\mathcal{C}$.
\end{definition}

We will not recall here the ellipsoid algorithm as it is standard, but we restate a know lemma on its runtime.

\begin{lemma}[Runtime of the ellipsoid algorithm]
Let $\mathcal{C}$ be a convex set. Suppose we know an $R > 0$ such that $\mathcal{C} \subseteq B_n(0,R)$, and that there exists a point $w \in \mathcal{C}$ and $\Delta > 0$ such that $B(w,\Delta) \subseteq \mathcal{C}$. Then the ellipsoid algorithm finds a point in $\mathcal{C}$ in no more than 
\begin{align}
O\left(n^2 \log \left(\frac{R}{\Delta} \right) \right)
\end{align}
calls to a separation oracle for $\mathcal{C}$.
\end{lemma}

Therefore, to construct an efficient algorithm that finds the exploration policy at time $t$, we just need to find how to implement a separation oracle for $\mathcal{D}_{t, \Delta_t}$. Observe that we can rewrite $\mathcal{D}_{t,\Delta_t}$ as the intersection of two convex sets:
\begin{align}
\mathcal{D}_{t,\Delta_t} \doteq \mathcal{C}_{t,\Delta_t} \cap \left\lbrace w \in \mathbb{R}^{Kt} : \ \forall z \in \mathcal{C}_t h_{t, \delta_t}(w, z) \leq \frac{5}{3} K \right\rbrace.
\end{align}
A separation oracle for $\mathcal{D}_{t,\Delta_t}$ can thus be built from a separation oracle for $\mathcal{C}_{t,\Delta}$ and a separation oracle for $\{  w \in \mathbb{R}^{Kt} : \ \forall z \in \mathcal{C}_t h_{t, \delta_t}(w, z) \leq 5K/3 \}$. 

The following lemma shows how to implement a separation oracle for $\mathcal{C}_{t,\Delta}$ using one call to LCLSO.

\begin{lemma}[Separation oracle for $\mathcal{C}_t$]
 Let $w \in \mathcal{C}^{Kt}$. Let 
 $$\tilde{w} \doteq \arg \min_{w'\in\mathcal{C}_t} \|w - w'\|.$$
If $\|w - \tilde{w}\| \leq \Delta$, then $w \in \mathcal{C}_{t,\Delta}$. If not, then
\begin{align}
\mathcal{H} \doteq \left\lbrace z \in \mathbb{R}^{Kt} : \langle z - w, w - \tilde{w} \rangle = 0 \right\rbrace
\end{align}
is an hyperplane that separates $w$ from $\mathcal{C}_{t,\Delta}$.

\begin{proof}
It suffices to show that $\forall z \in \mathcal{H}$, $d(z, \mathcal{C}_{t,\Delta}) > 0$, or equivalently that $d(z, \mathcal{C}_t) > \Delta$. Observe that since $w \in \mathcal{C}_{t,\Delta}$, we must have that $d(w, \mathcal{C}_t) > \Delta$. Therefore, it will be enough to show that 
\begin{align}
\forall z \in \mathcal{H},\ d(z,\mathcal{C}_t) \geq d(w,\mathcal{C}_t).
\end{align}

We first show that for all $\tilde{z} \in \mathcal{C}_t$, $\langle \tilde{z} - \tilde{w}, w - \tilde{w} \rangle > 0$. Then, for all $\lambda \in (0,1)$,
\begin{align}
\|w - (\lambda \tilde{z} + (1 - \lambda) \tilde{w}) \|_2^2 =& \|(w - \tilde{w}) - \lambda (\tilde{z} - \tilde{w}) \|_2^2 \\
=& \|w - \tilde{w}\|_2^2 + \lambda^2 \|z - \tilde{w}\|_2^2 - 2 \lambda \langle \tilde{z} - \tilde{w}, w - \tilde{w} \rangle.
\end{align}
Therefore, for $\lambda \in (0,1)$ small enough, 
$$\|w - (\lambda \tilde{z} + (1 - \lambda \tilde{w}) \|_2^2 \leq \|w - \tilde{w} \|_2^2.$$ 
Since, by convexity of $\mathcal{C}_t$, $\lambda \tilde{z} + (1 - \lambda) \tilde{w} \in \mathcal{C}_t$, this contradicts that $\tilde{w}$ is the projection of $w$ on $\mathcal{C}_t$. Therefore, we must have that 
\begin{align}
\langle \tilde{z} - \tilde{w}, w - \tilde{w} \rangle \leq 0 \label{eq:dot_product_property_proj_on_convex_body}
\end{align}
for all $\tilde{z} \in \mathcal{C}_t$.

We can now use this property to show the wished claim. Let $z \in \mathcal{H}$, and let $\tilde{z} \in \mathcal{C}_t$. We have that
\begin{align}
\|z - \tilde{z}\|_2^2 =& \|(z-w) + (w - \tilde{w}) + (\tilde{w} - \tilde{z})\|_2^2 \\
=& \|z-w\|_2^2 + \|w - \tilde{w}\|_2^2 + \|\tilde{w} - \tilde{z}\|_2^2 \\
& + 2 \underbrace{\langle z - w , w - \tilde{w} \rangle}_{=0 \text{ by definition of } \mathcal{H}} \\
&+ 2 \underbrace{\langle w -  \tilde{w}, \tilde{w} - \tilde{z} \rangle}_{\geq 0 \text{ from \eqref{eq:dot_product_property_proj_on_convex_body}}} \\
&+ 2 \underbrace{\langle z - w , \tilde{w} - \tilde{z} \rangle}_{\substack{\geq - \|z-w\| \|\tilde{w} - \tilde{z}\| \\ \text{ by Cauchy-Schwartz}}}\\
\geq & (\|z - w \| - \|w - \tilde{w}\|)^2 + \|\tilde{w} - \tilde{z}\|_2^2 \\
\geq & d(w, \mathcal{C}_t),
\end{align}
which concludes the proof.
\end{proof}
\end{lemma}

The next lemma shows how to implement a separation oracle for \begin{align}
\mathcal{L}_t \doteq \left\lbrace w \in \mathbb{R}^{Kt} : \forall z \in \mathcal{C}_t, h_{t,\delta_t}(w, z) \leq \frac{5}{3} K \right\rbrace.
\end{align}
using one call to LCCSCO.

\begin{lemma}[Separation oracle for $\mathcal{L}_t$]
Let $w \in \mathbb{R}^{Kt}$. Let 
\begin{align}
z^* \doteq \arg \max_{z \in \mathcal{C}_t} h_{t, \delta_t}(w,z).
\end{align}
$z^*$ can be found in one call to LCCSCO. If $h_{t,\delta_t}(w,z^*) \leq 5K/3$, then $w \in \mathcal{L}_t$. If not, then $w \not\in \mathcal{L}_t$ and 
\begin{align}
\mathcal{H} \doteq \left\lbrace w': h_{t,\delta_t}(w,z^*) + (\nabla_w h_t)(w,z^*)^\top (w'-w) = 0 \right\rbrace
\end{align}
separates $w$ and $\mathcal{L}_t$.
\end{lemma}

We restate below for self-containdness lemma 10 from \cite{dudik2011}, which will be useful in the rest of the section.

\begin{lemma}[Lemma 10 in \citep{dudik2011}]\label{lemma:lemma10_Dudik}
For $x \in \mathbb{R}^n$, let $f(x)$ be a convex function
of $x$, and consider the convex set $K$ defined by $K =
\{x : f (x) \leq 0\}$. Suppose we have a point $y$ such that
$f (y) > 0$. Let $\nabla f (y)$ be a subgradient of $f$ at $y$. Then
the hyperplane $f(y) + \nabla f (y)^\top (x- y) = 0$ separates y
from $K$.
\end{lemma}

\begin{proof}
Observe that 
\begin{align}
h_{t, \delta_t}(w,z) \doteq \frac{1}{t-1} \sum_{\substack{a \in [K] \\ \tau \in [t-1]}} u_{a,\tau} z_{a,\tau},
\end{align}
with 
\begin{align}
u_{a,\tau} \doteq \frac{1}{\delta_t / K + (1-\delta_t) w_{a,\tau}} \geq 0.
\end{align}
Therefore, $\arg \max_{z \in \mathcal{C}_t} h_{t, \delta}(w, z) = w_{t,f^*}$,
where
\begin{align}
f^* \doteq \arg \max_{f \in \mathcal{F}} \frac{1}{t-1} \sum_{\substack{a \in [K] \\ \tau \in [t-1]}} u_{a,\tau} f(a,W_\tau) \text{ subject to } \forall \tau \in [t], \hat{R}_\tau(f) \leq   \max_{f \in \mathcal{F}} \hat{R}_\tau(f) + \epsilon_\tau.
\end{align}
As 
\begin{align}
\hat{R}_\tau(f) = \frac{1}{\tau} \sum_{\substack{a \in [K] \\s \in [\tau]}} \frac{\Ind\{A_s = a\} (1-Y_s)}{g(a|W_s)} f(a,W_s),
\end{align}
the constraint $\hat{R}_\tau(f) \leq   \max_{f \in \mathcal{F}} \hat{R}_\tau(f) + \epsilon_\tau.$ is a linear constraint, and therefore, $f^*$ can be obtained with one call to LCCSCO.

From lemma \ref{lemma:lemma10_Dudik}, if $h_{t, \delta_t}(w, z^*) - 5K/3 > 0$,
\begin{align}
\mathcal{H} \doteq \left\lbrace w': h_{t,\delta_t}(w,z^*) + (\nabla_w h_t)(w,z^*)^\top (w'-w) = 0 \right\rbrace
\end{align}
separates $w$ from 
\begin{align}
\left\lbrace w' \in \mathbb{R}^{Kt} : h_{t,\delta_t}(w',z^*) - \frac{5}{3} K \leq 0 \right\rbrace,
\end{align}
and thus from $\mathcal{L}_t$, which concludes the proof.
\end{proof}

\section{Proof of the results on the additive model policy class}

\subsection{Proof of lemma \ref{lemma:entropy_additive_model}}

The following result is the fundamental building block of the proof.

\begin{lemma}[Bracketing entropy of univariate distribution functions]\label{lemma:bracketing_entropy_CDFs}
Let $\mathcal{G}$ the set of cumulative distribution functions on $[0,1]$. There exist $c_0 > 0$, $\epsilon_0 \in (0,1)$ such that, for all $\epsilon \in (0,\epsilon_0)$, 
\begin{align}
\log N_{[\,]}(\epsilon, \mathcal{G}, \|\cdot\|_\infty) \leq c_0 \epsilon^{-1} \log (1 / \epsilon).
\end{align}
\end{lemma}

We first state an intermediate result.
\begin{lemma}[Bracketing entropy of linear combinations]\label{lemma:bracketing_entropy_linear_combination} Let $\mathcal{H}$ be a class of functions and let 
\begin{align}
\mathcal{F} \doteq \left\lbrace \sum_{j=1}^J a_j h_j : a_1,\ldots,a_J \in [-B,B], \ h_1,\ldots,h_J \in \mathcal{H} \right\rbrace.
\end{align}
Suppose that for all $h \in \mathcal{H}$, $\|h\|_\infty \leq M$. Then, for all $\epsilon > 0$,
\begin{align}
\log N_{[\,]}(\epsilon, \mathcal{F}, \|\cdot\|_\infty) \leq J \log N_{[\,]}\left( \frac{\epsilon}{2 J B}, \mathcal{H}, \|\cdot\|_\infty \right) + J \log \left(\frac{4 J B M}{\epsilon} \right).
\end{align}
\end{lemma}

\begin{proof}[Proof of lemma \ref{lemma:bracketing_entropy_linear_combination}] Let 
\begin{align}
\mathcal{B} = \left\lbrace (l_k, u_k) : k \in [N] \right\rbrace
\end{align}
be an $\epsilon$ bracketing in $\|\cdot\|_\infty$ norm of $\mathcal{H}$. For all $m$, let $\alpha_m = m (B / M) \epsilon$. For all $f = \sum_{j=1}^J a_j h_j$, there exist $k_1, \ldots, k_J$ and $m_1, \ldots, m_J$ such that $\forall j \in [J]$, 
\begin{align}
l_{k_j} \leq h_j \leq u_{k_j},
\end{align}
and 
\begin{align}
\alpha_{m_{j-1}} \leq a_j \leq \alpha_{m_j}.
\end{align}
Therefore,
\begin{align}
\Lambda(k_1,\ldots,k_J, m_1,\ldots,m_J) \leq f \leq \Upsilon(k_1,\ldots,k_J, m_1,\ldots,m_J),
\end{align}
with 
\begin{align}
\Lambda(k_1,\ldots,k_J, m_1,\ldots,m_J) \doteq \sum_{j=1}^J \alpha_{m_j-1} (l_{i_j})^+ +  \alpha_{m_j} (l_{i_j})^-,\\
\text{ and } \Upsilon(k_1,\ldots,k_J, m_1,\ldots,m_J) \doteq \sum_{j=1}^J \alpha_{m_j-1} (u_{i_j})^+ +  \alpha_{m_j} (u_{i_j})^-.
\end{align}
Therefore, we have that 
\begin{align}
&| \Upsilon(k_1,\ldots,k_J, m_1,\ldots,m_J) - \Lambda(k_1,\ldots,k_J, m_1,\ldots,m_J)| \\
=& \bigg| \sum_{j=1}^J \alpha_{m_j-1} (u_{i_j} - l_{i_j}) + \sum_{j=1}^J (\alpha_{m_j} - \alpha_{m_j-1}) ((u_{m_j})^+ -(l_{m_j})^-) \bigg| \\
\leq & J B \epsilon + \frac{J B \epsilon}{M} M \\
= & 2 JB \epsilon.
\end{align}
Therefore,
\begin{align}
N_{[\,]}(2 J B \epsilon, \mathcal{F}, \|\cdot\|_\infty) \leq N_{[\,]}(\epsilon, \mathcal{H}, \|\cdot\|_\infty) \times \left( \frac{2 M}{\epsilon} \right)^J,
\end{align}
hence the claim.
\end{proof}

We can now prove lemma \ref{lemma:entropy_additive_model}
\begin{proof}[Proof of lemma \ref{lemma:entropy_additive_model}] Let $\epsilon > 0$. Let 
\begin{align}
\mathcal{B} \doteq \left\lbrace (\ell_{i}, u_i) : i \in [N] \right\rbrace,
\end{align}
be an $\epsilon$-bracketing in $\|\cdot\|_\infty$ the set of distribution functions on $[0,1]$, which we will denote $\mathcal{G}$. Let $h \in \mathcal{H}$. There exist $a \in [-B,B]$, $b \in [0,B]$, $h_1, h_2 \in $ such that $h = a + b(h_1 - h_2)$, and there exists $i_1, i_2 \in [N]$ such that
\begin{align}
l_{i_1} \leq h_1 \leq u_{i_1} \qquad \text {and} \qquad l_{i_2} \leq h_2 \leq u_{i_2},
\end{align}
and $i_3 \in [-1/\epsilon, 1/\epsilon]$ such that $a \in [\alpha_{i_3 -1}, \alpha_{i_3} ]$ with 
\begin{align}
\alpha_{i_3} \doteq i_3 M \epsilon,
\end{align}
and $i_4 \in [0, 1/\epsilon]$ such that $b \in [\beta_{i_4-1}, \beta_{i_4}]$ with 
\begin{align}
\beta_{i_4} \doteq i_4 M \epsilon.
\end{align}
Therefore, we have that
\begin{align}
\Lambda(i_1,i_2,i_3,i_4) \leq h \leq \Upsilon(i_1,i_2,i_3,i_4),
\end{align}
with 
\begin{align}
\Lambda(i_1,i_2,i_3,i_4) \doteq & \alpha_{i_3-1} + \beta_{i_3-1} (l_{i_1} - u_{i_2})^+ + \beta_{i_3}(l_{i_1} - u_{i_2})^- \\
\text{ and } \Upsilon(i_1, i_2, i_3, i_4) \doteq & \alpha_{i_3} + \beta_{i_3} (u_{i_1} - l_{i_2})^+ + \beta_{i_3-1} (u_{i_1} -l_{i_2})^-.
\end{align}
Note that
\begin{align}
0 \leq \Upsilon(i_1, i_2, i_3, i_4) - \Lambda(i_1, i_2, i_3, i_4) = &\alpha_{i_3} - \alpha_{i_3-1} \\
&+ \beta_{i_3 - 1}  (u_{i_1} - l_{i_1} + u_{i_2} - l_{i_2} ) \\
&+ (\beta_{i_3} - \beta_{i_3-1}) ( (u_{i_1} - l_{i_2})^+ - (l_{i_1} - u_{i_2})^-) \\
\leq & M \epsilon + 2 M \epsilon + M_\epsilon (u_{i_1} - l_{i_2})^+ \\
\leq & 4 M \epsilon.
\end{align}
Therefore,
\begin{align}
\mathcal{B}' \doteq \left\lbrace (\Lambda(i_1, i_2, i_3, i_4), \Upsilon(i_1, i_2, i_3, i_4))  : i_1, i_2 \in [N], i_3 \in [-1/\epsilon, 1 / \epsilon], i_4 \in [0, 1 / \epsilon] \right\rbrace
\end{align}
is an $4M \epsilon$-bracket in $\|\cdot\|_\infty$ norm of $\mathcal{H}$. Thus
\begin{align}
N_{[\,]}(4 M \epsilon, \mathcal{H}, \|\cdot\|_\infty) \leq &  \frac{2}{\epsilon} \times \frac{1}{\epsilon} \times N_{[\,]}(\epsilon, \mathcal{G}, \|\cdot\|_\infty)^2
\end{align}
That is
\begin{align}
\log N_{[\,]}(\epsilon, \mathcal{H}, \|\cdot\|_\infty) \leq & 2 \log \left( \frac{8 M^2}{\epsilon^2} \right) + 2 \log N_{[\,]}\left(\frac{\epsilon}{4M}, \mathcal{G}, \|\cdot\|_\infty\right) \\
=& 2 \log \left(\frac{\sqrt{8} M}{\epsilon} \right) + 2 \log N_{[\,]}(\epsilon, \mathcal{G}, \|\cdot\|_\infty).
\end{align}
Therefore, from lemma \ref{lemma:bracketing_entropy_linear_combination} and lemma \ref{lemma:bracketing_entropy_CDFs}, 
\begin{align}
\log N_{[\,]}(\epsilon, \mathcal{F}, \|\cdot\|_\infty) \leq & J \log \left( \frac{4 J B M}{\epsilon} \right) + 2 J \log \left( \frac{\sqrt{8} M J B M}{\epsilon} \right) + 2 J \log N_{[\,]}\left( \frac{\epsilon}{2 J B}, \mathcal{G}, \|\cdot\|_\infty \right) \\
\leq & (2J c_0 + 1) \epsilon^{-1} \log \left( \frac{4 \sqrt{8} J (M \vee M^2)(B \vee 1)}{\epsilon} \right),
\end{align}
for all $\epsilon \in (0, 2 J B\epsilon_0)$.
\end{proof}

\subsection{Proof of lemma \ref{lemma:representation_thm_additive_model}}

\begin{proof}[Proof of lemma \ref{lemma:representation_thm_additive_model}]
We decompose the proof in three steps. We will denote $\text{feas}(\mathcal{P}_1)$ and $\text{feas}(\mathcal{P}_2)$ the feasible sets of $\mathcal{P}_1$ and $\mathcal{P}_2$.

\paragraph{Step 1: The feasible set of $\mathcal{P}_2$ is contained in the feasible set of $\mathcal{P}_1$.} First, observe that for any $h: x \mapsto \sum_{\tau=1}^t \beta_\tau \Ind\{ x \geq x_\tau\}$, $\|h\|_v = \sum_{\tau =0}^t |\beta_\tau|$.
Therefore, for every $l$, $\tilde{\mathcal{H}}_{l,t} \subseteq \mathcal{H}$ and thus $\tilde{F}_t \subseteq \mathcal{F}$.

Second, observe that for any $w \in [0,1]^d$, there exists $\tilde{w} \in \mathcal{G}(w_0,\ldots,w_n)$ such that $\tilde{f}(w) = \tilde{f}(\tilde{w})$. Therefore, if $\tilde{f}$ satisfies \eqref{eq:P2_pos_constrt} and \eqref{eq:P2_sums_to1_constr} at every $(a, w) \in [K] \times \mathcal{G}(w_0,\ldots,w_t)$, it satisfies them everywhere. Therefore, this proves that the feasible set of $\mathcal{P}_2$ is contained in the feasible set of $\mathcal{P}_1$.

\paragraph{Step 2: For any $f$ in the feasible set of $\mathcal{P}_1$, there is an $\tilde{f}$ in the feasible set of $\mathcal{P}_2$ that achieves the same value of the objective function.} Let $f: (a,w) \mapsto \sum_{l=1}^d \alpha_{a,l} h_{a,l}(w_l)$ be an element of the feasible set of $\mathcal{P}_1$. Observe that for all $a$, $l$, there exists $\tilde{h}_{a,l}$ of the form $\tilde{h}_{a,l} : x \mapsto \sum_{\tau = 1}^t \beta_{a,l,\tau} \Ind\{ x \geq w_{\tau, l}\}$ such that for all $\tau \in \{0,\ldots,t\}$, $\tilde{h}_{a,l}(w_{\tau, l}) = h_{a,l}(w_{\tau, l})$. As $f$ and $\tilde{f}$ coincide at every $(a, w) \in [K] \times \mathcal{G}(w_0,\ldots,w_n)$, constraints $\eqref{eq:P1_pos_constrt}$ and $\eqref{eq:P1_sums_to1_constr}$ are satisfied at every $(a, w) \in [K] \times \mathcal{G}(w_0,\ldots,w_n)$, and $f$ and $\tilde{f}$ achieve the same value of the objective function. To prove that $\tilde{f}$ is in the feasible set of $\mathcal{P}_2$, it remains to show that the functions $(\tilde{h}_{a,l})_{a \in [K], l \in [d]}$ are in $\mathcal{H}_{l,t}$, that is that for all $a, l$, $\sum_{\tau = 0}^t | \beta_{a,l,\tau}| \leq M$. We have that 
\begin{align}
\sum_{\tau=0}^t | \beta_{a, l, \tau} | =& | \tilde{h}_{a,l}(0)| + \sum_{\tau=1}^t | \tilde{h}_{a,l}(w_{\tau,l}) - \tilde{h}_{a,l}(w_{\tau,l})| \\
=& | h_{a,l}(0)| + \sum_{\tau=1}^t | h_{a,l}(w_{\tau,l}) - h_{a,l}(w_{\tau-1,l})| \\
\leq & |h_{a,l}(0)| + \sum_{\substack{ m \in \mathbb{N} \\ 0 \leq x_1 \leq \ldots \leq x_m \leq 1}} |h_{a,l}(x_{m+1}) - h_{a,l}(x_m)| \\
\leq & \|h_{a,l}\|_v \\
\leq & M.
\end{align}

\paragraph{Step 3: End of the proof.} Let $f^*$ be a solution to $\mathcal{P}_1$. Let $\tilde{f}^*$ be a function in the feasible set of $\mathcal{P}_2$ such that $f^* = \tilde{f}^*$ on $[K] \times \mathcal{G}(w_0,\ldots,w_t)$. From step 2, such a function exists. The objective function evaluated at $\tilde{f}^*$ is equal to the objective function evaluated at $f^*$.  Since, from step 1, $\text{feas}(\mathcal{P}_1) \subseteq \text{feas}(\mathcal{P}_2)$, and $f^*$ is a maximizer over $\text{feas}(\mathcal{P}_1)$, $\tilde{f}^*$ must be a maximizer over both $\mathcal{P}_1$ and $\mathcal{P}_2$.
\end{proof}

\section{Representation results for the ERM over cadlag functions with bounded
  sectional variation norm}
\label{sec:cadlag:svn}

\subsection{Empirical risk minimization in the hinge case}

The following result shows that empirical risk minimization over $\mathcal{F}^{hinge}$, with $\mathcal{F}_0$ the class of	cadlag functions with bounded sectional variation norm.

\begin{lemma}[Representation of the ERM in the hinge case]\label{lemma:HAL_ERM_hinge} Consider a class of policies of the form $\mathcal{F}^{hinge}$, as defined in \eqref{eq:F_hinge_def}, derived from $\mathcal{F}_0$, as defined in \eqref{eq:def_HAL_class}. Let $\phi = \phi^{hinge}$.
Suppose we have observed $(W_1,A_1,Y_t), \ldots, W_t,A_t,Y_t)$ and let $\tilde{W}_1,\ldots, \tilde{W}_m$ be the elements of $G(W_1,\ldots,W_t)$.

Let $(\beta^a_j)_{a\in[K], j\in[m]}$ be a solution to
\begin{align}
\min_{\beta \in \mathbb{R}^{Km}} & \sum_{\tau=1}^t \sum_{a\in[K]} \bigg\{ \frac{\Ind\{A_\tau = a\}}{g_\tau(A_\tau,W_\tau)}(1-Y_\tau) \\
 &\times \max\left(0,1+\sum_{j=1}^m \beta^a_j \Ind\{W_\tau \geq \tilde{W}_j \} \right) \bigg\}\\
\text{ s.t. } & \forall l \in [m],\  \sum_{a\in[K]} \sum_{j=1}^m \beta^a_j \Ind\{\tilde{W}_l \geq \tilde{W}_j\} = 0,\\
& \forall a \in [K],\  \sum_{j=1}^m | \beta^a_j| \leq M.
\end{align}
Then $f : (a,w) \mapsto \sum_{j=1}^m \beta^a_j \Ind\{w \geq \tilde{W}_j\}$ is a solution to $\min_{f \in \mathcal{F}^{Id}} \sum_{\tau=1}^t \ell^{\phi}_\tau(f)(O_\tau).$
\end{lemma}

\subsection{Formal definition of the Vitali variation and the sectional variation norm}
We now present in full generality the definitions of the notions Vitali variation, Hardy-Krause variation and sectional variation norm. This requires introducing some prelimiary definitions. This section is heavily inspired from the excellent presentation of \cite{Fang-Guntuboyina-Sen-2019}, and we write it instead of directly referring to their work mostly for self-containdness, and so as to ensure matching notation.

\begin{definition}[Rectangular split, rectangular partition and rectangular grid] For any $d$ subvidisions 
\begin{align}
0 = w_{k,1} \leq w_{k,2} \leq \ldots \leq w_{k,q_k} = 1, \ k =1,\ldots,d,
\end{align}
of $[0,1]$, let 
\begin{itemize}
\item $\mathcal{P}$ be the collection of all closed rectangles of the form $[w_{1, i_1}, w_{1, i_1 + 1}] \times \ldots \times [w_{d, i_d}, w_{d, i_d + 1}]$,
\item $\mathcal{P}^*$ be the collection of all open rectangles of the form $[w_{1, i_1}, w_{1, i_1 + 1}) \times \ldots \times [w_{d, i_d}, w_{d, i_d + 1})$.
\item $\mathcal{G}$ the collection of all points  of the form $(w_{i_1},\ldots,w_{i_d})$.
\end{itemize}
Any collection of the form $\mathcal{P}$ is called a rectangular split of $[0,1]^d$, any collection of the form $\mathcal{P}^*$ is called a rectangular partition of $[0,1]^d$ and any set of points of the form $\mathcal{G}$ is called a rectangular grid on $[0,1]^d$.
\end{definition}

\begin{definition}[Minimum rectangular split, partition and grid] Let $w_1,\ldots,w_n$ be $n$ points of $[0,1]^d$.
We call minimum  rectangular split induced by $w_1,\ldots,w_n$,  and we denote
$\mathcal{P}(w_1,\ldots,w_n)$,  the rectangular  split of  minimum cardinality
such   that    $w_1,\ldots,w_n$   are    all   corners   of    rectangles   in
$\mathcal{P}(w_1,\ldots,w_n)$.  We define  similarly  the minimum  rectangular
parition induced by $w_1,\ldots w_n$. We denote it $\mathcal{P}^*(w_1,\ldots,w_n)$. We define the minimum rectangular grid induced by $w_1,\ldots,w_n$, which we denote $\mathcal{G}(w_1,\ldots,w_n)$, as the smallest cardinality rectangular grid that contains $w_1,\ldots,w_n$.
\end{definition}

\begin{definition}[Section of a function]
Let $s \in [d]$, $s \neq \emptyset$, and consider $f \in \mathbb{D}([0,1]^d)$. We call the $s$-section of $f$, and denote $f_s$, the restriction of $f$ to the set
\begin{align}
\{ (w_1,\ldots,w_d) \in [0,1]^d : \forall j \in s, w_j = 0 \}.
\end{align}
Observe that the above set is a face of the cube $[0,1]^d$ and that $f_s$ is a cadlag function with domain $[0,1]^{|s|}$.
\end{definition}

\begin{definition}[Vitali variation]
For any $d \geq 1$ and any rectangle $R$ of the form $[w_{1,1}, w_{2,1}] \times \ldots \times [w_{1,d}, w_{2,d}]$ or  $[w_{1,1}, w_{2,1}) \times \ldots \times [w_{1,d}, w_{2,d})$, such that for all $k = 1,\ldots,d$, $w_{k,1} \leq w_{k,2}$, let
\begin{align}
\Delta^{(d)}(f, R) = \sum_{j_1=0}^{J_1} \ldots \sum_{j_d=0}^{J_d} (-1)^{j_1+\ldots+j_d} f( w_{2,1} + j_1 (w_{1,1} - w_{2,1}),\ldots,w_{2,d} + j_d (w_{1,d} - w_{2,d}) ),
\end{align}
where, for all $k=1,\ldots,d$, $J_k = I(w_{2,d} \neq w_{1,d})$. The quantity $\Delta^{(d)}(f, R)$ is called the \textit{quasi volume} ascribed to $R$ by $f$.
The Vitali variation of $f$ on $[0,1]^d$ is defined as
\begin{align}
V^{(d)}(f, [0,1]^d) = \sup_{\mathcal{P} } \sum_{R \in \mathcal{P} } |\Delta^{(d)}(f, R)|,
\end{align}
where the $\sup$ is over all the rectangular partitions of $[0,1]^d$.
\end{definition}

\begin{definition}[Hardy-Krause variation and sectional variation norm]
The Hardy-Krause variation anchored at the origin of a function $f \in \mathbb{D}([0,1]^d)$ is defined as the sum of the Vitali variation of its sections, that is it is defined as the quantity
\begin{align}
V_{HK,\bm{0}}(f) = \sum_{\emptyset \neq s \subseteq [d]} V^{(d)}(f_s, [0,1]^{|s|} ).
\end{align}
The sectional variation norm of $f$ is defined as follows:
\begin{align}
\|f\|_v = |f(0)| + V_{HK,\bm{0}}(f).
\end{align}
\end{definition}

\subsection{Proof of lemmas \ref{lemma:HAL_ERM_direct_policy_optim} and \ref{lemma:HAL_ERM_hinge}}

The proof of lemmas \ref{lemma:HAL_ERM_direct_policy_optim} and \ref{lemma:HAL_ERM_hinge} will easily follow from the following two results.

\begin{lemma}\label{lemma:exists_piecewise_cst_coinciding_on_grid_lesser_svn} Let $f \in \mathcal{F}_0$. Let $x_1,\ldots,x_n \in [0,1]^d$. Denote $\tilde{x}_1,\ldots,\tilde{x}_{m}$ the elements of $G(x_1,\ldots,x_n)$. Let 
\begin{align}
\tilde{\mathcal{F}}_0(x_1,\ldots,x_n) \doteq \left\lbrace x \mapsto \sum_{j=1}^m \beta_j \Ind\{ x \geq \tilde{x}_j \} : \sum_{j=1}^m | \beta_j | \leq M \right\rbrace. \label{eq:def_tildeF0}
\end{align}
Then
\begin{itemize}
\item $\tilde{\mathcal{F}}_0(x_1,\ldots,x_n) \subseteq \mathcal{F}$,
\item there exists $\tilde{f} \in \tilde{\mathcal{F}}_0(x_1,\ldots,x_n)$ such that $\tilde{f}$ and $f$ coincide on $G(x_1,\ldots,x_m)$ and $\|\tilde{f}\|_v \leq \|f\|_v$.
\end{itemize}
\end{lemma}

\begin{lemma}\label{lemma:constraints_satisfied_everywhere_piecewise_const}
Let $\tilde{f}_1,\ldots,\tilde{f}_q \in \tilde{\mathcal{F}}_0(x_1,\ldots,x_n)$. Let $\alpha_1,\ldots,\alpha_q, \beta \in \mathbb{R}$. Consider the inequality constraint 
\begin{align}
\sum_{l=1}^q \alpha_l \tilde{f}_l \leq \beta.
\end{align}
The following are equivalent.
\begin{enumerate}
\item $\tilde{f}_1,\ldots,\tilde{f}_q$ satisfy the inequality constraint everywhere on $[0,1]^d$.
\item $\tilde{f}_1,\ldots,\tilde{f}_q$ satisfy the inequality constraint everywhere at every point of $G(x_1,\ldots,x_n)$.
\end{enumerate}
\end{lemma}
We relegate the proofs of the two above lemmas further down in this section.
We can now state the proof of lemmas \ref{lemma:HAL_ERM_direct_policy_optim} and \ref{lemma:HAL_ERM_hinge}.

\begin{proof}[Proof of lemmas \ref{lemma:HAL_ERM_direct_policy_optim} and \ref{lemma:HAL_ERM_hinge}.]
The following arguments apply similarly to lemma \ref{lemma:HAL_ERM_direct_policy_optim} and lemma \ref{lemma:HAL_ERM_hinge}. We present the proof in the direct policy optimization case.
We proceed in two steps.
\paragraph{Step 1: the feasible set of \eqref{eq:EMR_HAL_dpo} contains a solution the ERM problem over $\mathcal{F}^{Id}$}
Let $f$ be a solution to
\begin{align}
\min_{f \in \mathcal{F}^{Id}} \sum_{\tau=1}^t  \ell_{\tau}^{Id}(f)(O_\tau). \label{eq:optim_pb_direct_pol_optim}
\end{align}
There exists $f_1,\ldots,f_K \in \mathcal{F}_0$ such that $\forall a \in [K]$, $f(a,\cdot) = f_a(\cdot)$. From lemma \ref{lemma:exists_piecewise_cst_coinciding_on_grid_lesser_svn}, there exists $\tilde{f}_1,\ldots,\tilde{f}_K$ that coincide with $f_1,\ldots,f_K$ on $G(x_1,\ldots,x_n)$. Then the function $\tilde{f}:(x,a) \mapsto \tilde{f}_a(x)$ achieves the same value of the objective in \eqref{eq:optim_pb_direct_pol_optim} as $f$.

Since $\tilde{f}_1,\ldots, \tilde{f}_K$ coincide with $f_1,\ldots,f_K$ on $G(x_1,\ldots,x_n)$, they satisfy the same inequality constraints as $f_1,\ldots,f_K$ (that is non-negativity, and summing up to 1) on $G(x_1,\ldots,x_n)$. From lemma \ref{lemma:constraints_satisfied_everywhere_piecewise_const}, $\tilde{f}_1,\ldots,\tilde{f}_K$ must satisfy these constraints everywhere. 

That $\tilde{f}_1, \ldots,\tilde{f}_K$ are in $\mathcal{F}_0$, satisfy the positivity constraint, and sum to 1 everywhere, imply that that $\tilde{f}$ defined above is in $\mathcal{F}^{Id}$.

\paragraph{Step 2: The feasible set of \eqref{eq:EMR_HAL_dpo} is included in $\mathcal{F}^{Id}$.}

This follows directly from lemmas \ref{lemma:exists_piecewise_cst_coinciding_on_grid_lesser_svn} and \ref{lemma:constraints_satisfied_everywhere_piecewise_const}.
\end{proof}

\begin{proof}[Proof of lemma \ref{lemma:exists_piecewise_cst_coinciding_on_grid_lesser_svn}]
Let $\tilde{f}$ be of the form $x \mapsto \sum_{j=1}^m \beta_j \Ind\{ x \geq \tilde{x}_j \}$ such that for every $j \in [m]$, $\tilde{f}(\tilde{x}_j) = f(\tilde{x}_j)$. Let us show that $\|\tilde{f}\|_v \leq \|f\|_v \leq M$. We have that
\begin{align}
V(f, [0,1]^d) = & \sup_{\mathcal{P}} \sum_{R \in \mathcal{P} } | \Delta(f, R) | \\
=& \sup_{\substack{\mathcal{P}' = \mathcal{P} \cap \mathcal{P}(x_1,\ldots,x_n) \\ \mathcal{P} \text{ rect. split} }} \sum_{R \in \mathcal{P}} | \Delta(f,R) | \\
\geq & \sup_{R \in \mathcal{P}(x_1,\ldots,x_n)} | \Delta(f,R)| \\
=& \sup_{R \in \mathcal{P}(x_1,\ldots,x_n)} | \Delta(\tilde{f}, R) | \\
=& \sum_{R \in \mathcal{P}(x_1, \ldots, x_n)} | \Delta(\tilde{f},R) | \\
=& V(\tilde{f}, [0,1]^d).
\end{align}
The second line in the above display follows from corollary \ref{corollary:Vitali_as_sup_over_refined_splits}. The third line follows from lemma \ref{lemma:absolute_pseudo_vol_increases_as_split_gets_finer}. The fourth line follows from the fact that, as $|\Delta(f,R)|$ only depends on $f$ through its values at the corners of $R$, which, for $R$ in $\mathcal{P}(x_1,\ldots,x_n)$, are points of $G(x_1,\ldots,x_n)$, at which $f$ and $\tilde{f}$ coincide. The last line follows from corollary \ref{corollary:Vitali_rect_piecewise_const}.

The above implies that $M \geq \|f\|_v \geq \|\tilde{f} \|_v = \sum_{j=1}^m | \beta_j|$, where the last equality follows from lemma \ref{lemma:HK_var_piecewiese_constant}.

We have thus shown that for every $f \in \mathcal{F}_0$, we can find an $\tilde{f} \in \tilde{F}_0(x_1, \ldots, x_n)$ that coincides with $G(x_1,\ldots,x_n)$. 

It remains to show that $\tilde{\mathcal{F}}_0(x_1, \ldots, x_n) \subseteq \mathcal{F}_0$. It is clear that the elements of $\tilde{F}_0(x_1, \ldots, x_n)$ are cadlag. From lemma \ref{lemma:HK_var_piecewiese_constant}, the definition of $\tilde{\mathcal{F}}_0(x_1, \ldots, x_n)$ implies that its elements have sectional variation norm smaller than $M$. Therefore, $\tilde{\mathcal{F}}_0(x_1, \ldots, x_n) \subseteq \mathcal{F}_0$.
\end{proof}
\subsection{Technical lemmas on splits and Vitali variation}

\subsubsection{Effect on Vitali variation and absolute pseudo-volume of taking finer splits}
The following lemma says that the sum over a split of the absolute pseudo-volume ascribed by $f$ increases as one refines the split.

\begin{lemma}\label{lemma:absolute_pseudo_vol_increases_as_split_gets_finer}
Let $f:[0,1]^d \rightarrow \mathbb{R}$. Let $\mathcal{P}_1$ and $\mathcal{P}_2$ be two rectangular splits of $[0,1]^d$. Define 
\begin{align}
\mathcal{P}_1 \cap \mathcal{P}_2 \doteq \left\lbrace R_1 \cap R_2 : R_1 \in \mathcal{P}_1,\ R_2 \in \mathcal{P}_2 \right\rbrace.
\end{align}
It holds that
\begin{align}
\sum_{R \in \mathcal{P}_1} | \Delta^{(d)}(f,R) | \leq \sum_{R' \in \mathcal{P}_1 \cap \mathcal{P}_2} | \Delta(f, R')|.
\end{align}
\end{lemma}

We relegate the proof at the end of this section. The following lemma has the following corollary.

\begin{corollary}\label{corollary:Vitali_as_sup_over_refined_splits}
For any function $f:[0,1]^d \rightarrow \mathbb{R}$ and any rectangular split $\mathcal{P}_0$ of $[0,1]^d$, the Vitali variation of $f$, which we recall is defined as $V^{(d)}(f) \doteq \sup_{\mathcal{P} \text{ rect. split}} \sum_{R \in \mathcal{P}} | \Delta(f, R)|$ can actually be written as
\begin{align}
V^{(d)} = \sup_{\substack{\mathcal{P}' = \mathcal{P} \cap \mathcal{P}_0 \\ \mathcal{P} \text{ rect. split}}} \sum_{R \in \mathcal{P}'} | \Delta(f,R')|.
\end{align}
\end{corollary}

\begin{proof}[Proof of corollary \ref{corollary:Vitali_as_sup_over_refined_splits}]
Observe that the set of rectangular splits $\{\mathcal{P} \cap \mathcal{P}_0: \mathcal{P} \text{ rect. split} \}$ is included in the set of all rectangular splits. Therefore,
\begin{align}
\sup_{\substack{\mathcal{P}'=\mathcal{P} \cap \mathcal{P}_0 \\ \mathcal{P}' \text{ rect. split} }} \sum_{R \in \mathcal{P}'} | \Delta(f, R)| \leq \sum_{\mathcal{P} \text{split} } | \Delta(f,R)|.
\end{align}
Lemma \ref{lemma:absolute_pseudo_vol_increases_as_split_gets_finer} implies the converse inequality:
\begin{align}
\sup_{\substack{\mathcal{P}'=\mathcal{P} \cap \mathcal{P}_0 \\ \mathcal{P}' \text{ rect. split} }} \sum_{R \in \mathcal{P}'} | \Delta(f, R)| \geq \sum_{\mathcal{P} \text{split} } | \Delta(f,R)|.
\end{align}
We therefore have the wished equality.
\end{proof}

\subsubsection{Vitali variation of piecewise constant functions}

The following lemma characterizes the sum over a rectangular split of the absolute pseudo-volumes of a function that is piecewise constant on the rectangles of that split.

\begin{lemma}\label{lemma:HK_var_piecewiese_constant}
Let $x_1,\ldots,x_n \in [0,1]^d$ and let $\tilde{x}_1,\ldots,\tilde{x}_m$ be the elements of $G(x_1,\ldots,x_n)$. Consider a function $f$ of the form
\begin{align}
f:x \mapsto \sum_{j=1}^m \beta_j \Ind\{x \geq x_j\},
\end{align}
It holds that
\begin{align}
V_{HK,\bm{0}}(f) = \sum_{j=1}^m |\beta_j|.
\end{align}
\end{lemma}

\begin{corollary}[Vitali variation of rectangular piecewise constant function]\label{corollary:Vitali_rect_piecewise_const}
Let $x_1,\ldots,x_n \in [0,1]^d$, let $\tilde{x}_1,\ldots,\tilde{x}_m$ be the elements of $G(x_1,\ldots,x_m)$, and consider a function $f$ of the form
\begin{align}
f: x \mapsto \sum_{j=1}^J \beta_j \Ind\{ x \geq \tilde{x}_j \}.
\end{align}
Then 
\begin{align}
\sup_{\mathcal{P} \text{ rect. split} } | \Delta(f,R)| = \sum_{R \in \mathcal{P}(x_1,\ldots,x_n)} |\Delta(f,R)|,
\end{align}
where $\mathcal{P}(x_1,\ldots,x_n)$ is a minimal rectangular split induced by $x_1,\ldots,x_m$.
\end{corollary}

\begin{proof}[Proof of lemma \ref{lemma:absolute_pseudo_vol_increases_as_split_gets_finer}]
Consider a rectangle $R \in \mathcal{P}(x_1,\ldots,x_n)$. There exist $k,l \in [m]$ such that $R = [\tilde{x}_k, \tilde{x}_l]$. (Since $\mathcal{P}(x_1,\ldots,x_n)$ is a minimal split, we must have $\tilde{x}_k < \tilde{x}_l$ as otherwise the corresponding minimum grid would have duplicate points and would therefore not be minimal). Observe that
\begin{align}
\Delta(f, [\tilde{x}_k, \tilde{x}_l]) = & \Delta \left( \sum_{j=1}^m \beta_j \Ind\{\cdot \geq \tilde{x}_j\}, [\tilde{x}_k, \tilde{x}_l]\right) \\
= & \beta_j \sum_{j=1}^m \Delta\left( \Ind\{\cdot \geq \tilde{x}_j\}, [\tilde{x}_k, \tilde{x}_l] \right),
\end{align}
as the operator $f' \mapsto \Delta(f', [\tilde{x}_k, \tilde{x}_l])$ is linear. Let us calculate $\Delta(1 \{\cdot \geq \tilde{x}_j\}, [\tilde{x}_k, \tilde{x}_l])$ for every $j \in [m]$. We have that 
\begin{align}
&\Delta(\Ind\{\cdot \geq \tilde{x}_j\}, [\tilde{x}_k, \tilde{x}_l]) \\
=& \sum_{j_1,  \ldots ,j_d \in \{0,1\} } (-1)^{j_1+\ldots+j_d} 1\left\lbrace \tilde{x}_{l,1} + j_1 (\tilde{x}_{k,1} - \tilde{x}_{l,1}) \geq \tilde{x}_{j,1}, \ldots,\tilde{x}_{l,d} + j_d (\tilde{x}_{k,d} - \tilde{x}_{l,d}) \geq \tilde{x}_{j,d} \right\rbrace. \label{eq:quasi_volume_corner_function}
\end{align}
From there, we distinguish three cases.
\paragraph{Case 1:} There exists $i \in [d]$ such that $\tilde{x}_{j,i} > \tilde{x}_{l,i}$. Then, all terms in \eqref{eq:quasi_volume_corner_function} are zero and thus $\Delta(\Ind\{\cdot \geq \tilde{x}_j\}, [\tilde{x}_k, \tilde{x}_l]) = 0$.

\paragraph{Case 2:} $\tilde{x}_j = \tilde{x}_l$. Then, only the term in \eqref{eq:quasi_volume_corner_function} corresponding to $j_1=\ldots=j_d=0$ is non-zero and thus $\Delta(\Ind\{\cdot \geq \tilde{x}_j\}, [\tilde{x}_k, \tilde{x}_l]) = 1$.

\paragraph{Case 3:} $\tilde{x}_j \leq \tilde{x}_l$ and $\tilde{x}_j \neq \tilde{x}_l$. Then denote 
\begin{align}
I = & \{ i \in [d] : \tilde{x}_{j,i} = \tilde{x}_{l,i} \} \\
\text{and } I^c = & [d] \backslash I.
\end{align}
As $\tilde{x}_j \neq \tilde{x}_l$, $I^c \neq \emptyset$. Denote $i_1,\ldots,i_q$ the elements of $I^c$, where $q=|I^c|$. Then
\begin{align}
&\Delta(\Ind\{\cdot \geq \tilde{w}_j\}, [\tilde{x}_k, \tilde{x}_l]) \\
=& \sum_{j_1,\ldots,j_d \in \{0,1\} } (-1)^{j_1+\ldots+j_d} 1\left\lbrace \tilde{x}_{l 1} + j_1 (\tilde{x}_{k,1} - \tilde{x}_{l,1}) \geq \tilde{x}_{j,1}, \ldots, \tilde{x}_{l,d} + j_d (\tilde{x}_{k,d} - \tilde{x}_{l,d}) \geq \tilde{x}_{j,d} \right\rbrace \\
=& \sum_{j_1,\ldots j_d \in \{0,1\} } (-1)^{j_1+\ldots+j_d} \Ind\{ \forall i \in I, j_i = 0\} \\
=& \sum_{j_{i_1},\ldots,j_{i_q} \in \{0,1\}} (-1)^{j_{i_1}+\ldots+j_{i_q}} \\
=& \sum_{j_{i_1}=0}^1 (-1)^{j_{i_1}} \sum_{j_{i_2}=0}^1 (-1)^{j_{i_2}} \ldots \sum_{j_{i_q} = 0}^1 (-1)^{j_{i_q}} \\
= & 0.
\end{align}
Therefore, we have shown that, for all $j=1,\ldots,m$,
\begin{align}
\Delta(\Ind\{\cdot \geq \tilde{x}_j\}, [\tilde{x}_k, \tilde{x}_l]) = \begin{cases}
1 &\text{ if } \tilde{x}_j = \tilde{x}_k,\\
0 &\text{ otherwise.}
\end{cases}
\end{align}
This implies that
\begin{align}
| \Delta(f, [\tilde{w}_k, \tilde{w}_l]) = |\beta_k|.
\end{align}
which concludes the proof.
\end{proof}

\begin{proof}[Proof of corollary \ref{corollary:Vitali_rect_piecewise_const}]
From lemma \ref{lemma:absolute_pseudo_vol_increases_as_split_gets_finer},
\begin{align}
\sup_{\mathcal{P} \text{ split} } \sum_{R \in \mathcal{P}} | \Delta(f,R)| = \sup_{\substack{\mathcal{P}'= \mathcal{P} \cap \mathcal{P}(x_1,\ldots,x_n) \\ \mathcal{P}' \text{rect. split} }} \sum_{R \in \mathcal{P'}} \sum_{R \in \mathcal{P'}} | \Delta(f,R)|.
\end{align}
Consider a split $\mathcal{P}'$ of the form $\mathcal{P} \cap \mathcal{P}(x_1,\ldots,x_n)$. We can write the corresponding rectangular grid as $\tilde{x}_1,\ldots,\tilde{x}_m,\tilde{x}_{m+1},\ldots,x_{m'}$ where $\tilde{x}_1,\ldots,\tilde{x}_m$ are the points of $G(x_1,\ldots,x_n)$. We can rewrite $f$ as 
\begin{align}
f:x \mapsto \sum_{j=1}^m \beta_j \Ind\{x \geq \tilde{x}_j\},
\end{align}
with $\beta_{m+1} = \ldots = \beta_{m'} = 0$.
From lemma \ref{lemma:HK_var_piecewiese_constant},
we have that
\begin{align}
\sum_{R \in \mathcal{P}'} | \Delta(f,R)| = \sum_{j=1}^{m'} | \beta_j| = \sum_{j=1}^m | \beta_j| = \sum_{R \in \mathcal{P}(x_1,\ldots,x_n)} | \Delta(f,R)|.
\end{align}
Therefore
\begin{align}
\sup_{\substack{\mathcal{P}'=\mathcal{P} \cap \mathcal{P}(x_1,\ldots,x_n) \\ \mathcal{P} \text{ rect. split} }}  \sum_{R \in \mathcal{P}'} | \Delta(f, R)| = \sum_{R \in \mathcal{P}(x_1,\ldots,x_n)} | \Delta(f,R)|.
\end{align}
\end{proof}


\end{document}